\documentclass{article} 
\usepackage{iclr2023_conference,times}

\usepackage{amsmath,amsfonts,bm}

\def\eqref#1{equation~\ref{#1}}

\def\1{\bm{1}}

\DeclareMathAlphabet{\mathsfit}{\encodingdefault}{\sfdefault}{m}{sl}
\SetMathAlphabet{\mathsfit}{bold}{\encodingdefault}{\sfdefault}{bx}{n}

\newcommand{\reg}{\lambda}

\newcommand{\ind}{\mathbbm{1}}
\newcommand{\probP}{\text{I\kern-0.15em P}}

\DeclareMathOperator*{\argmax}{arg\,max}
\DeclareMathOperator*{\argmin}{arg\,min}

\DeclareMathOperator{\sign}{sign}

\DeclareMathOperator{\proj}{proj}
\def\delequal{\mathrel{\ensurestackMath{\stackon[1pt]{=}{\scriptstyle\Delta}}}}

\usepackage[colorlinks=true,linkcolor=blue,citecolor=green,urlcolor=blue]{hyperref}
\usepackage{url}
\usepackage{hyperref}
\usepackage{url}
\usepackage{natbib}
\usepackage{graphicx}
\usepackage{amsmath}               
\usepackage{amsthm}
\usepackage{mathtools}
\usepackage{bbm}
\usepackage{xcolor}
\usepackage{algorithm}
\usepackage{algorithmic}
\usepackage{scalerel}
\usepackage{stackengine}

\usepackage[outdir=./]{epstopdf}
\usepackage{graphicx,float,pgfplots,wrapfig,sidecap,lipsum}

\usepackage{colortbl}
\usepackage{tablefootnote}
\usepackage[font=small,labelfont=bf]{caption}
\usepackage{subcaption}
\usepackage{tikz}
\usetikzlibrary{fit}
\usetikzlibrary{calc,shapes}
\usetikzlibrary{decorations.pathmorphing} 
\usetikzlibrary{fit}					
\usetikzlibrary{backgrounds}	
\usetikzlibrary{pgfplots.groupplots}

\usepackage[utf8]{inputenc}
\usepackage{pgfplots}
\DeclareUnicodeCharacter{2212}{−}
\usepgfplotslibrary{groupplots,dateplot}
\usetikzlibrary{patterns,shapes.arrows}
\pgfplotsset{compat=newest}
\usepackage{booktabs} 
\usepackage{xcolor}
\usepackage{colortbl}
\usepackage{multirow}
\usepackage{array}
\usepackage{makecell}
\usepackage{caption}
\usepackage{soul}

\newtheorem{theorem}{Theorem}[section]

\newtheorem{corollary}[theorem]{Corollary}
\newtheorem{definition}[theorem]{Definition}
\newtheorem{assumption}[theorem]{Assumption}
\theoremstyle{remark}

\AtBeginEnvironment{proof}{}{}{}

\newcommand{\lossreg}{\mathcal{L}_{\mathrm{reg}}}

\definecolor{sourcecolor}{rgb}{0.5,1,0.5}
\definecolor{ourcolor}{rgb}{1,0,0}
\definecolor{singlecolor}{rgb}{0,0,1}
\definecolor{auxcolor}{rgb}{0.54,0.17,0.89}
\definecolor{linearcolor}{rgb}{0.172549019607843,0.627450980392157,0.172549019607843}
\definecolor{randomcolor}{rgb}{1,0.498039215686275,0.0549019607843137}
\definecolor{tunecolor}{rgb}{0.9568627450980393, 0.8156862745098039,0}
\definecolor{aligncolor}{rgb}{0,0.5,0}

\newcommand{\mytitle}{
Is Model Ensemble Necessary?
Model-based RL via a Single Model with Lipschitz Regularized Value Function}

\title{\mytitle}

\iclrfinalcopy

\author{
Ruijie Zheng\textsuperscript{\rm 1, \textsection} 
\qquad
Xiyao Wang\textsuperscript{\rm 1, \textsection} 
\qquad
Huazhe Xu\textsuperscript{\rm 2, \rm 3}
\qquad
Furong Huang\textsuperscript{\rm 1}  \\
\textsuperscript{\rm 1} \text{University of Maryland, College Park} \quad \texttt{\{rzheng12, xywang, furongh\}@umd.edu} \\
\textsuperscript{\rm 2} Tsinghua University \quad  \texttt{huazhe\_xu@mail.tsinghua.edu.cn}   \\
\textsuperscript{\rm 3} Shanghai Qi Zhi Institute \\
}

\begin{document}

\maketitle
\begingroup\renewcommand\thefootnote{\textsection}
\footnotetext{Equal contribution}
\begin{abstract}
Probabilistic dynamics model ensemble is widely used in existing model-based reinforcement learning methods as it outperforms a single dynamics model in both asymptotic performance and sample efficiency.
In this paper, we provide both practical and theoretical insights on the empirical success of the probabilistic dynamics model ensemble through the lens of Lipschitz continuity.
We find that, for a value function, the stronger the Lipschitz condition is, the smaller the gap between the true dynamics- and learned dynamics-induced Bellman operators is, thus enabling the converged value function to be closer to the optimal value function.
Hence, we hypothesize that the key functionality of the probabilistic dynamics model ensemble is to regularize the Lipschitz condition of the value function using generated samples.
To test this hypothesis, we devise two practical robust training mechanisms through computing the adversarial noise and regularizing the value network's spectral norm to directly regularize the Lipschitz condition of the value functions. Empirical results show that combined with our mechanisms, model-based RL algorithms with a single dynamics model outperform those with an ensemble of probabilistic dynamics models. These findings not only support the theoretical insight, but also provide a practical solution for developing computationally efficient model-based RL algorithms.

\end{abstract}

\section{Introduction}\label{s1_introduction}
Model-based reinforcement learning (\emph{MBRL}) improves the sample efficiency of an agent by learning a model of the underlying dynamics in a real environment. 
One of the most fundamental questions in this area is how to learn a model to generate good samples so that it maximally boosts the sample efficiency of policy learning. To address this question, various model architectures are proposed such as Bayesian nonparametric models~\citep{{kocijan2004gaussian}, {nguyen2008local}, {kamthe2018data}}, inverse dynamics model~\citep{{pathak2017curiosity}, {liu2022plan}}, multi-step model~\citep{{asadi2019combating}}, and hypernetwork~\citep{huang2021continual}.

Among these approaches, the most popular and common approach is to use an ensemble of probabilistic dynamics models~\citep{{buckman2018}, {janner2019trust}, {29}, {clavera2020modelaugmented}, {froehlich2022onpolicy}, {li2022gradient}}.
It is first proposed by~\citet{chua2018deep} to capture both the aleatoric uncertainty of the environment and the epistemic uncertainty of the data.
In practice, MBRL methods with an ensemble of probabilistic dynamics models can often achieve higher sample efficiency and asymptotic performance than using only a single dynamics model.

However, while the uncertainty-aware perspective seems reasonable, previous works only directly apply probabilistic dynamics model ensemble in their methods without an in-depth exploration of why this structure works. 
There still lacks enough theoretical evidence to explain the superiority of probabilistic neural network ensemble. In addition, extra computation time and resources are needed to train an ensemble of neural networks.

In this paper, we provide a new perspective on why training a probabilistic dynamics model ensemble can significantly improve the performance of the model-based algorithm. 
We find that the ensemble model-generated state transition samples that are used for training the policies and the critics
are much more ``diverse'' than samples generated from a single deterministic model. 
We hypothesize that it implicitly regularizes the Lipschitz condition of the critic network especially over a local region where the model is uncertain (i.e., \emph{a model-uncertain local region}). 
Therefore, the Bellman operator induced by the learned transition dynamics will yield an update to the agent's value function close to the true underlying Bellman operator's update.
We provide systematic experimental results and theoretical analysis to support our hypothesis.

Based on this insight, we propose two simple yet effective robust training mechanisms to regularize the Lipschitz condition of the value network. 
The first one is \emph{spectral normalization}~\citep{miyato2018} which provides a global Lipschitz constraint for the value network. It directly follows our theoretical insights to explicitly control the Lipschitz constant of the value function. However, based on both of our theoretical and empirical observations, only the local Lipschitz constant around a model-uncertain local region is required for a good performance. So we also propose the second mechanism, \emph{robust regularization}, which stabilizes the local Lipschitz condition of the value network by computing the adversarial noise with fast gradient sign method~(FGSM)~\citep{ian2015}. 
To compare the effectiveness of controlling global versus local Lipschitz, systematic experiments are implemented.
Experimental results on five MuJoCo environments verify that the proposed Lipschitz regularization mechanisms with a single deterministic dynamics model improves SOTA performance of the probabilistic ensemble MBRL algorithms with less computational time and GPU resources. 

Our contributions are summarized as follows.
\textbf{(1)} We propose a new insight into why an ensemble of probabilistic dynamics models can significantly improve the performance of the MBRL algorithms, supported by both theoretical analysis and experimental evidence.
\textbf{(2)} We introduce two robust training mechanisms for MBRL algorithms, which directly regularize the Lipschitz condition of the agent's value function.
\textbf{(3)} Experimental results on five MuJoCo tasks demonstrate the effectiveness of our proposed mechanisms and validate our insights, improving SOTA asymptotic performance using less computational time and resources.

\section{Preliminaries and Background}\label{s2_preliminaries}

\textbf{Reinforcement learning.} \quad
We consider a Markov Decision Process (MDP) defined by the tuple $(\mathcal{S}, \mathcal{A}, \mathcal{P},\mathcal{P}_0, r, \gamma)$, 
where $\mathcal{S}$ and $\mathcal{A}$ are the state space and action space respectively, $\mathcal{P}(s^{\prime}|s,a)$ is the transition dynamics, $\mathcal{P}_0$ is the initial state distribution, $r(s,a)$ is the reward function and $\gamma$ is the discount factor.
In this paper, we focus on value-based reinforcement learning algorithms. Define the \emph{optimal Bellman operator} $\mathcal{T^*}$ such that $\mathcal{T^*} Q(s,a)=r(s,a)+\gamma \int \mathcal{P}(s'|s,a)\argmax_{a'}Q(s', a')ds'$. The goal of the value-based algorithm is to learn a $Q_\phi$ to approximate the optimal state-action value function $Q^*$, where $Q^*=\mathcal{T^*}Q^*$ is the fixed point of the optimal Bellman operator.

For model-based RL~(MBRL), we denote the approximate model for state transition and reward function as $\hat{\mathcal{P}}_\theta$ and $\hat{r}_\theta$ respectively. Similarly, we define the \emph{model-induced Bellman Operator} $\widehat{\mathcal{T}}^*$ such that $\widehat{\mathcal{T}}^* Q(s,a)=\hat{r}(s,a)+\gamma\int \hat{\mathcal{P}}_\theta(s'|s,a)\argmax_{a'}Q(s', a')ds'$.



\textbf{Probabilistic dynamics model ensemble.} \quad
Probabilistic dynamics model ensemble consists of $K$ neural networks $\hat{\mathcal{P}}_\theta = \{\hat{\mathcal{P}}_{\theta_1}, \hat{\mathcal{P}}_{\theta_2}, ..., \hat{\mathcal{P}}_{\theta_K}\}$ with the same architecture but randomly initialized with different weights.
Given a state-action pair $(s,a)$, the prediction of each neural network is a Gaussian distribution with diagonal covariances of the next state: $\hat{\mathcal{P}}_{\theta_k}(s^\prime | s,a) = \mathcal{N}(\mu_{\theta_k} (s,a), \Sigma_{\theta_k} (s,a))$. The model is trained using negative log likelihood loss~\citep{janner2019trust}: $\mathcal{L}(\theta_{k}) =  \sum_{t=1}^N [\mu_{\theta_k}(s_t,a_t)-s_{t+1}]^\top \Sigma_{\theta_k}^{-1}(s_t,a_t)[\mu_{\theta_k}(s_t,a_t)-s_{t+1}] + \log\det\Sigma_{\theta_k}(s_t,a_t)$
During model rollouts, the probabilistic dynamics model ensemble first randomly selects a network from the ensemble and then samples the next state from the predicted Gaussian distribution.


\textbf{Local Lipschitz constant.} \quad
We now give the definition of local Lipschitz constant, a key concept in the following discussion.
\begin{definition}
Define the \textbf{local $(\mathcal{X}, \epsilon)$-Lipschitz constant} of a scalar valued function $f: \mathbb{R}^N\rightarrow \mathbb{R}$ over a set $\mathcal{X}$ as:
\begin{equation}
\displaystyle L^{(p)}(f, \mathcal{X}, \epsilon)= \sup_{x\in \mathcal{X}}\quad\sup_{y_1, y_2\in B^{(p)}(x,\epsilon)}\frac{|f(y_1)-f(y_2)|}{\|y_1-y_2\|_p} \quad (y_1 \neq y_2).
\label{local_lipschitz}
\end{equation}
Throughout the paper, we consider $p=2$ unless explicitly stated. If $L(f, \mathcal{X}, \epsilon)=C$ is finite, we say that $f$ is $(\mathcal{X},\epsilon)$-locally Lipschitz with constant $C$. 
\end{definition}
In particular, we can view the \textbf{global Lipschitz constant} as $L(f, \mathbb{R}^N, \epsilon)$, which upper bounds the local Lipschitz constant $L(f, \mathcal{X}, \epsilon)$ with $\mathcal{X}\subset \mathbb{R}^N$.

\section{Local Lipschitz Condition of Value Functions in MBRL}\label{s3_lip_mbrl}

\subsection{Key Insight: Implicit Regularization on Value Network by Probabilistic Ensemble Model}
\label{s3s1_insight}
In this section, we provide insight into why a value-based reinforcement learning algorithm trained from simulated samples generated by a probabilistic ensemble model performs well in practice. We first conduct an experiment on the MuJoCo Humanoid environment with MBPO \citep{janner2019trust}, the SOTA model-based algorithm using Soft Actor-Critic (SAC) \citep{haarnoja2018soft} as the backbone algorithm for policy and value optimization.
We run MBPO with four different types of trained environment models: (1) a single deterministic dynamics model, (2) a single probabilistic dynamics model, (3) an ensemble of seven deterministic dynamics models, and (4) an ensemble of seven probabilistic dynamics models. As we can see from Figure \ref{sfig:mbpo_a}, in the Humanoid environment, the agent trained from an ensemble of the probabilistic dynamics model indeed achieves much better performance.
A similar observation is found across all the other MuJoCo environments.

\begin{figure}[!htbp]
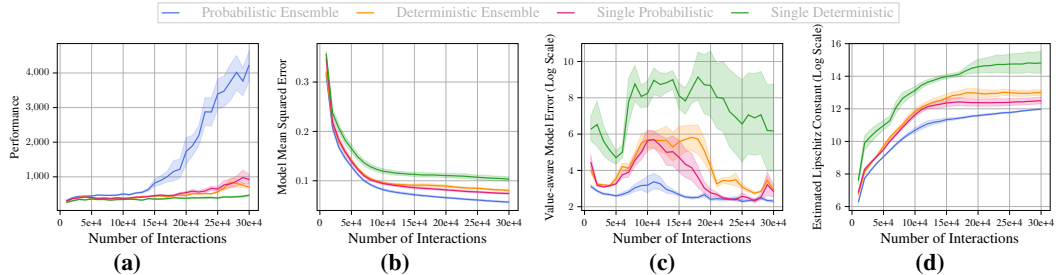

\vspace{-0.6em}
\centering
 \begin{subfigure}[t]{0.23\textwidth}
  \centering
  \input{humanoid_performance.tex}
  \vspace{-1.5em}
  \centering
    \caption{}
    \label{sfig:mbpo_a}
 \end{subfigure}
 \hfill
 \begin{subfigure}[t]{0.23\textwidth}
  \centering
  \input{humanoid_mse} 
    \vspace{-1.5em}
  \centering
    \caption{}
    \label{sfig:mbpo_b}
 \end{subfigure}
 \hfill
 \begin{subfigure}[t]{0.23\textwidth}
  \centering
  \input{motivation_humanoid_vaml} 
  \vspace{-1.5em}
  \centering
    \caption{}
    \label{sfig:mbpo_c}
 \end{subfigure}
 \hfill
 \begin{subfigure}[t]{0.23\textwidth}
  \centering
  \input{humanoid_lip} 
  \vspace{-1.5em}
  \centering
    \caption{}
    \label{sfig:mbpo_d}
 \end{subfigure}
 \vspace{-0.5em}
 \caption{
\small{\textbf{(a)} Performance of MBPO algorithm trained with single deterministic, single probabilistic, deterministic ensemble, and probabilistic ensemble model respectively on Humanoid.     \textbf{(b)} Model mean squared error.  \textbf{(c)} Value-aware model error in log scale.  \textbf{(d)} Upper bound of Lipschitz constant of value network. Results are averaged over 8 random seeds.}
 }
\label{fig:mbpo_exp}
\vspace{-1em}
\end{figure}

\textbf{Value-aware model error.} Why does MBPO with an ensemble of probabilistic dynamics models perform much better? 
We try to answer this question by considering the effects of the learned transition models on the agent's value functions. 
In MBPO algorithm, it fits an ensemble of $M$ Gaussian environment models $\hat{\mathcal{P}}_\theta^{(i)}(\cdot|s,a)\equiv \mathcal{N}\big(\mu_\theta^{(i)}(s,a), \Sigma_\theta^{(i)}(s,a)\big), i=1,..., M$. 
With a deep neural network as a powerful function approximator for the mean and variance of the Gaussian distribution, the mean squared error  $\mathbb{E}_{(s,a)\sim \rho, s'\sim\mathcal{P}(\cdot|s,a), \hat{s}'\sim \hat{\mathcal{P}}_\theta(\cdot|s,a)}[\|s'-\hat{s}'\|^2]$ is often small. 
However, as argued in many previous MBRL works \citep{grimm2020value, grimm2021}, the mean squared error is often not a good measurement for the quality of the learned transition model. 
Instead, the \emph{value-aware model error} \citep{farahmand17itervaml} as defined below plays a crucial role here, which connects directly to the suboptimality of the MBRL algorithm.
\vspace{-0.5em}
\begin{align*}
    \setlength\abovedisplayskip{2pt}
    \setlength\belowdisplayskip{0pt}
\mathcal{L}_{\textit{vame}}(\hat{\mathcal{P}}_\theta;Q,\rho)&=\mathbb{E}_{(s,a)\sim \rho}[\big(\mathcal{T}^*Q(s,a)-\widehat{\mathcal{T}}^*Q(s,a)\big)^2]\\
&=\gamma^2 \mathbb{E}_{(s,a)\sim \rho}\big[\Big(\int \mathcal{P}(s'|s,a)V(s')ds'-\int \hat{\mathcal{P}}_\theta(\hat{s}'|s,a)V(\hat{s}')d\hat{s}'\Big)^2\big]\\
&\leq \gamma^2 \mathbb{E}_{(s,a)\sim \rho, s'\sim \mathcal{P}(\cdot|s, a), \hat{s}'\sim \hat{\mathcal{P}}_\theta(\cdot|s, a)}[\big(V(\hat{s}')-V(s')\big)^2],\: V(s)=\max_{a}Q(s,a)
\end{align*}
\vspace{-2em}

This value-aware model error measures the difference between the simulated and true Bellman operator acting on a given $Q$ function. In other words, even though the learned transition model $\hat{\mathcal{P}}_\theta$ is approximately accurate, the model prediction error can still be amplified by the value function so that the two Bellman operators yield completely different updates on the value function.  

\textbf{Value function Lipschitz regulates value-aware model error.} When the value function's local-Lipschitz constant is under control, the value-aware model error can be made small. As we see from Figure~\ref{sfig:mbpo_b}, \ref{sfig:mbpo_c}, and \ref{sfig:mbpo_d}, although the probabilistic ensemble model and single deterministic model achieve similar mean squared errors, value functions trained from an ensemble of probabilistic dynamics model has a significantly smaller Lipschitz constant and value-aware model error. (Note that the value-aware model error plotted in Figure~\ref{sfig:mbpo_d} is in log scale.) 

To see why this is the case, we view the learned Gaussian transition models $\hat{\mathcal{P}}_\theta^{(i)}$ as $f^{(i)}_\theta + g^{(i)}_\theta$, where $f_\theta^{(i)}$ is a deterministic model and $g^{(i)}_\theta(\cdot|s,a)\equiv \mathcal{N}(0,\Sigma^{(i)}_\theta(s,a))$ is a noise distribution with zero mean.
When training the value network of the MBRL agent, the target value is computed as $\displaystyle r(s,a)+\gamma \mathbb{E}_{i\sim \text{Cat}(M, \frac{1}{M}), \epsilon_i\sim g_\theta^{(i)}(\cdot|s, a)}[V\big(f_\theta^{(i)}(s,a)+\epsilon_i\big)]$. 
Here we view data generated from the probabilistic ensemble model as an implicit augmentation. 
The augmentation comes from two sources: (i) variation of prediction across different models in the ensemble and (ii) the noise added by each noise distribution $g_\theta^{(i)}$. 
By augmenting the transition with a diverse set of noise and then training the value functions with such augmented samples, it implicitly regularizes the local Lipschitz condition of the value network over the local region around the state where the model prediction is uncertain.

In the next section, we provide some theoretical insights into how the local Lipschitz constant can play a role in the suboptimality of the MBRL algorithms. Later in the algorithm and experiment section, we will provide two practical mechanisms to regularize the Lipschitz condition of the value network and demonstrates the effectiveness of such mechanisms, which further validates our claim.

\subsection{Error Bound of Model-based Value Iteration with locally Lipschitz value functions}
\label{s3_theory}
In this section, we formally analyze how the Lipschitz constant of the value function affects the learning dynamics of the model-based (approximate) value iteration algorithm, the prototype for most of the value-based MBRL algorithms. 
For simplicity, following previous works \citep{farahmand2017vaml, grimm2020value}, we assume that we have the true reward function $r(s,a)$, but extending the results to reward function approximation should be straightforward. 

We consider the following value iteration algorithm. At $k\:$th-iteration of the algorithm, we obtain a dataset $\mathcal{D}_k\delequal\{(s_i, a_i, s'_i)\}_{i=1}^{N}$, where $(s_i, a_i)$ is sampled i.i.d. from $\rho\in \Delta(\mathcal{S}\times\mathcal{A})$, the empirical state-action distribution, and $s'_i\sim \mathcal{P}(\cdot|s_i, a_i)$ is the next state under the environment transition. Based on this dataset, we first approximate the transition kernel $\hat{\mathcal{P}}$ to minimize the mean $L_2$ difference between the predicted and true next states.
\vspace{-0.5em}
\begin{equation}
\label{eq:model}
\hat{\mathcal{P}}_k\leftarrow \argmin_{\hat{\mathcal{P}}\in \mathcal{M}}
\frac{1}{N}\sum_{i=1}^{N}\int \hat{\mathcal{P}}(\hat{s}'|s_i,a_i)\big\|\hat{s}'-s'_i\big\| d\hat{s}'
\vspace{-0.3em}
\end{equation}
Then at $k\:$th-iteration, we update the value function by solving the following regression problem: 
\begin{equation}
\label{eq:reg_error}
\hat{Q}_k \leftarrow \argmin_{\hat{Q}\in \mathcal{F}}\mathcal{L}_{\textit{reg}}(\hat{Q}; \hat{Q}_{k-1}, \hat{\mathcal{P}}_k) \coloneqq \argmin_{\hat{Q}\in \mathcal{F}}
\mathbb{E}_{(s,a)\sim \rho}\Big[\big(\hat{Q}(s,a)-\hat{\mathcal{T}}^*\hat{Q}_{k-1}(s,a)\big)^2\Big]
\vspace{-0.5em}
\end{equation}
Empirically, we update the value function such that
\vspace{-0.5em}
\begin{equation}\label{eq:regression}
    \hat{Q}_{k} \leftarrow \argmin_{\hat{Q}\in \mathcal{F}} \frac{1}{N}\sum_{i=1}^{N}\Big|\hat{Q}(s_i, a_i)-\Big(r_i+\gamma \int \hat{\mathcal{P}}_k(\hat{s}'|s_i,a_i)\hat{V}_{k-1}(\hat{s}')d\hat{s}' \Big)\Big|^2
\vspace{-0.5em}
\end{equation}
where $\hat{V}_{k-1}(\hat{s}')=\max_{a'}\hat{Q}_{k-1}(\hat{s}', a')$.

To simplify the analysis, we assume that we are given a fixed state-action distribution $\rho$ such that for every iteration, we can sample i.i.d. from this data distribution. However, in practice, we may use a different data collecting policy at different iterations. As argued in \citep{farahmand17itervaml}, a similar result can be shown in this case by considering the mixing behavior of the stochastic process.

Now we list the assumptions we make. Some assumptions are made only to simplify the finite sample analysis, while others characterize the crucial aspects of model and value learning.

First, we make the deterministic assumption of the environment. This is only for the purpose of the finite sample analysis. When the environment transition is stochastic, we will not have the finite sample guarantee, but our insights still hold. We will provide more discussion on this later when we present our finite sample guarantee. 
\begin{assumption}
\label{a0_deterministic}
The environment transition is deterministic.
\end{assumption}
\vspace{-0.5em}
Next, to apply the concentration inequality in our analysis, we have to make the following technical assumption of the state space and reward function. 
\begin{assumption}
\label{a1_boundedness}
\textbf{(Boundedness of State Space and Reward Function)}
There exists constants $D$, $R_{\max}$ such that for all $s\in \mathcal{S}, a\in \mathcal{A}$, $\|s\|_2 \leq D$ and $r(s,a)<R_{\max}$. 
\end{assumption}
\vspace{-0.5em}
In addition, we make the following mild approximate realizability assumption on the model class of approximated transition kernels so that in the model space, at least one transition model close to the true underlying transition kernel should exist.
\begin{assumption}
\label{a2_realizability}
\big(\textbf{($\epsilon, \rho$)-Approximate Realizability}\big)
\begin{equation}
\label{eq:mse}
\inf_{\hat{\mathcal{P}}\in\mathcal{M}}\mathcal{L}_2(\hat{\mathcal{P}})\coloneqq\inf_{\hat{\mathcal{P}}\in\mathcal{M}} \int \hat{\mathcal{P}}(d\hat{s}'|s,a)\big\|\mathcal{P}(s,a)-\hat{s}'\big\|_2 d\rho(s,a)\leq \epsilon
\end{equation}
\end{assumption}
\vspace{-1.0em}
We also make a critical assumption of the local Lipschitz condition of the value function class. In particular, for the state-action value function $Q:\mathcal{S}\times\mathcal{A}\rightarrow \mathbb{R}$, we define that $Q$ is $(\epsilon, p)$-locally Lipschitz with constant $L$ if for every $a\in \mathcal{A}$, the function $Q_a: s\xmapsto[]{} Q(s,a)$ is $(\epsilon, p)$-locally Lipschitz with constant less than or equal to L. 
\begin{assumption} 
\label{a3_lip}
\textbf{(Local Lipschitz condition of value functions)}
There exists a finite $L$, such that for every $Q\in \mathcal{F}$, it is $\big(\mathcal{X}, 2\epsilon\big)$-locally Lipschitz with a constant less than or equal to $L$, where $\mathcal{X}$ is the support of the distribution $\rho$.
\vspace{-0.8em}
\end{assumption}

Note here we only need to assume that value functions are all $(\mathcal{X}, (1+\beta)\epsilon)$-locally Lipschitz with $\beta>0$. We set $\beta=1$ for simplicity.
Finally, same as the previous work~\citep{farahmand17itervaml}, to provide a finite sample guarantee of model learning, we make the following assumption on the complexity of the model space of the approximated transition kernels 

\begin{assumption} 
\label{a4_complexity}
\textbf{(Complexity of Model Space)}
Let $R>0$, $J:\mathcal{M}_0\rightarrow [0,\infty)$ be a pseudo-norm, where $\mathcal{M}_0$ is a space of transition kernels. Let $\mathcal{M}=\mathcal{M}_R=\{\mathcal{P}:J(\mathcal{P})\leq R\}$. There exists constants $c>0$ and $0<\alpha<1$ such that for any $\epsilon, R>0$ and all sequence of state-action pairs $z_{1:n} \delequal z_1,..., z_n \in \mathcal{S}\times \mathcal{A}$, the following metric entropy condition is satisfied:
\vspace{-0.3em}
\begin{equation}
\log \mathcal{N}(\epsilon, \mathcal{M}, L_2(z_{1:n}))\leq c \Big(\frac{R}{\epsilon}\Big)^{2\alpha}
\end{equation}
\vspace{-0.35em}
where $ \mathcal{N}(\epsilon, \mathcal{M}, L_2(z_{1:n}))$ is the covering number of $\mathcal{M}$ with respect to the empirical norm $L_2(z_{1:n})$ such that $\displaystyle \|\mathcal{P}\|_{2,z_{1:n}}^2=\frac{1}{N}\sum_{i=1}^{N}\|\mathcal{P}(\cdot|z_i)\|_2^2$
\vspace{-1.0em}
\end{assumption}

Under these assumptions, we can now present our main theorem, which relates the local Lipschitz constant to the suboptimality of the approximate model-based value iteration algorithm. In particular, we provide a finite sample analysis of model learning in Theorem~\ref{t1_finite_sample} and value-aware model error in Theorem~\ref{t2_vaml}. Then we apply the error accumulation results from \citep{farahmand17itervaml}, connecting the local Lipschitz constant with the suboptimality of the algorithm.
\begin{theorem}
\label{thm:main}
 Suppose $\hat{Q}_0$ is initialized such that  $\displaystyle \hat{Q}_0(s,a)\leq \frac{R_{\max}}{1-\gamma}$ for $\forall\: (s,a)$. Under the assumptions of \ref{a0_deterministic}, \ref{a1_boundedness}, \ref{a2_realizability}, \ref{a3_lip}, and \ref{a4_complexity}, after $K$ iterations of the model-based approximate value iteration algorithm, there exists a constant $\kappa(\alpha)$ which depends solely on $\alpha\in (0,1)$ such that, 
\begin{align}
\label{eq:suboptimal}
\mathbb{E}_{(s,a)\sim \mathcal{P}_0}\Big[\Big|Q^*(s,a)-\hat{Q}_K(s,a)\Big|\Big] \leq \frac{2\gamma}{(1-\gamma)^2} \Big[&C(\rho, \mathcal{P}_0)\Big(\max_{0\leq k\leq K}\delta_k+\gamma^2\big(4\epsilon^2 L^2 \xi \nonumber \\
&+\frac{(1-\xi)R_{\max}^2}{(1-\gamma)^2}\big)\Big)+2\gamma^K R_{\max}\Big]
\end{align}
where  $\delta_k^2=\mathcal{L}_{\textit{reg}}(\hat{Q}_k; \hat{Q}_{k-1}, \hat{\mathcal{P}}_k)$ is the regression error defined in Equation~(\ref{eq:reg_error}), 
$\xi=\displaystyle 1-\exp(-\frac{\epsilon N^{\frac{1}{1+\alpha}}}{\kappa(\alpha)D^2R^{\frac{2\alpha}{1+\alpha}}})$, and $C(\rho, \mathcal{P}_0)$ is the concentrability constant defined in Definition~\ref{df:concentrability}.
\end{theorem}

\textbf{Remarks.}\\
\textbf{(1)} As the number of samples $N\to \infty$, $\xi\to 1$. Consequently, the value-aware model error $\displaystyle \mathbb{E}\big[|\mathcal{T}^*Q_k(s,a)-\widehat{\mathcal{T}}^*Q_k(s,a)|^2\big]$ will be governed by the term $4\epsilon^2 L^2 \xi$ which is controlled by the local Lipschitz constant $L$. 
In addition, if the number of iterations $k\rightarrow \infty$, the left hand side of Equation~(\ref{eq:suboptimal}) will be bounded by $\displaystyle \frac{2\gamma C(\rho, \mathcal{P}_0)}{(1-\gamma)^2}\Big(\max_{0\leq k\leq K}\delta_k+4\epsilon^2 L^2 \xi\Big)$. 
From here, we can see the role that the local Lipschitz constant has played in controlling the suboptimality of the algorithm.\\
\textbf{(2)} Although we assume that the environment transition is deterministic, our insights should still hold when it is stochastic. 
In particular, if the learned transition model has a low prediction error $\mathcal{L}_2(\hat{\mathcal{P}})$ (Eq. \ref{eq:mse}) (which is often the case with a neural network as the function approximator), and the local Lipschitz constant of the value function is bounded, then we can still have a small value-aware model error and get a similar error propagation result as Theorem \ref{thm:main}.

\textbf{Tradeoff between regression error and value-aware model error.} \quad 
This theorem also reveals an essential trade-off between the regression error and value-aware model error through \emph{the Lipschitz condition of the value function class}, i.e., the constant $L$. 
 As $L$ gets smaller, the model-induced Bellman operator gets closer to the actual underlying Bellman operator. 
However, with a smaller $L$, we also impose a stronger condition on the value function class. Therefore, the value function space will shrink, and the regression error $\delta_k$ is expected to get larger. 

 To further visualize this trade-off, we conduct an experiment on the Inverted-Pendulum environment, where we run the analyzed model-based approximate value iteration algorithms for 1000 iterations.
 We plot (Figure~\ref{sfig:test_exp_a}) the best evaluation performance for 20 episodes at every iteration, (Figure~\ref{sfig:test_exp_b}) the maximum regression error across every iteration, as well as (Figure~\ref{sfig:test_exp_c}) the maximum value-aware model error. 
 To upper bound the (local) Lipschitz constant, we constrain the spectral norm of the weight matrix for each layer of the value network.
\begin{figure}[!htbp]
\vspace{-0.5em}
\centering
 \begin{subfigure}[t]{0.3\textwidth}
  \centering
\begin{tikzpicture}[scale=0.4]

\definecolor{darkgray176}{RGB}{176,176,176}
\definecolor{tomato2556060}{RGB}{255,60,60}

\begin{axis}[
tick align=outside,
tick pos=left,
x grid style={darkgray176},
xlabel style={font=\Large},
title style={font=\Large},
ylabel style={font=\large},
xlabel={Upper Bound of Network Lipschitz Constant},
xmajorgrids,
xmin=-268.75, xmax=8393.75,
xtick style={color=black},
y grid style={darkgray176},
ylabel={Performance},
ymajorgrids,
ymin=-8.2625, ymax=1048.0125,
ytick style={color=black}
]
\addplot [thick, black, dashed]
table {%
1000 -8.26250000000002
1000 1048.0125
};
\addplot [thick, black, dashed]
table {%
4000 -8.26250000000002
4000 1048.0125
};
\addplot [very thick, tomato2556060, mark=square*, mark size=3, mark options={solid}]
table {%
125 39.75
300 525.25
500 759.5
1000 1000
2000 1000
3000 1000
4000 1000
5000 762.5
6000 762.75
7000 762
8000 521
};
\end{axis}

\end{tikzpicture}
  \centering
   \vspace{-0.5em}
    \caption{best evaluation performance}
    \label{sfig:test_exp_a}
 \end{subfigure}
 \hfill
 \begin{subfigure}[t]{0.3\textwidth}
  \centering
\begin{tikzpicture}[scale=0.4]

\definecolor{darkgray176}{RGB}{176,176,176}
\definecolor{green012713}{RGB}{0,127,13}

\begin{axis}[
tick align=outside,
tick pos=left,
x grid style={darkgray176},
xlabel style={font=\Large},
title style={font=\Large},
ylabel style={font=\large},
xlabel={Upper Bound of Network Lipschitz Constant},
xmajorgrids,
xmin=-268.75, xmax=8393.75,
xtick style={color=black},
y grid style={darkgray176},
ylabel={Regression Error (Log Scale)},
ymajorgrids,
ymin=3.78297510575, ymax=15.63255095925,
ytick style={color=black}
]
\addplot [thick, black, dashed]
table {%
1000 3.78297510575
1000 15.63255095925
};
\addplot [thick, black, dashed]
table {%
4000 3.78297510575
4000 15.63255095925
};
\addplot [very thick, green012713, mark=diamond*, mark size=3, mark options={solid}]
table {%
125 15.093933875
300 10.61746532
500 10.742280635
1000 4.804889995
2000 4.732167875
3000 5.32001526
4000 4.9326272
5000 4.32159219
6000 5.25291406
7000 4.78572444
8000 5.027342495
};
\end{axis}

\end{tikzpicture} 
  \centering
  \vspace{-0.5em}
    \caption{max regression error}
    \label{sfig:test_exp_b}
 \end{subfigure}
 \hfill
 \begin{subfigure}[t]{0.3\textwidth}
  \centering
\begin{tikzpicture}[scale=0.4]

\definecolor{darkgray176}{RGB}{176,176,176}
\definecolor{steelblue76114176}{RGB}{76,114,176}

\begin{axis}[
tick align=outside,
tick pos=left,
x grid style={darkgray176},
xlabel style={font=\Large},
title style={font=\Large},
ylabel style={font=\large},
xlabel={Upper Bound of Network Lipschitz Constant},
xmajorgrids,
xmin=-268.75, xmax=8393.75,
xtick style={color=black},
y grid style={darkgray176},
ylabel={Value-aware Model Error (Log Scale)},
ymajorgrids,
ymin=3.16848845075, ymax=8.96248873425,
ytick style={color=black}
]
\addplot [thick, black, dashed]
table {%
1000 3.16848845075
1000 8.96248873425
};
\addplot [thick, black, dashed]
table {%
4000 3.16848845075
4000 8.96248873425
};
\addplot [very thick, steelblue76114176, mark=*, mark size=3, mark options={solid}]
table {%
125 3.4318521
300 4.38023631
500 4.56800905
1000 6.001915275
2000 6.167813145
3000 6.283317305
4000 6.64332927
5000 6.77647159
6000 7.281957665
7000 7.359245405
8000 8.699125085
};
\end{axis}

\end{tikzpicture} 
  \centering
  \vspace{-0.5em}
    \caption{max value-aware model error}
    \label{sfig:test_exp_c}
 \end{subfigure}
 \vspace{-0.5em}
\caption{Model-based value iteration on Inverted-Pendulum across value networks with different Lipschitz constraints. All the results are the median over 5 random seeds. 
}
\label{fig:demo_exp}
\vspace{-1.5em}
\end{figure}
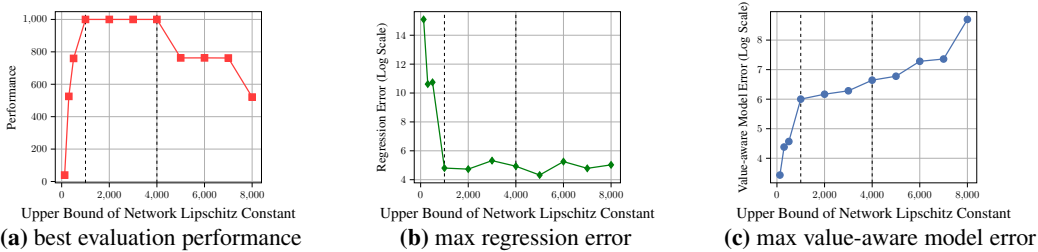


As we see from the results in Figure~\ref{sfig:test_exp_b}, indeed, the regression error dramatically drops as the Lipschitz constant grows from 100 to 1000 and then levels off, which indicates that perhaps the Lipschitz constant of 1000 is rich enough for the value function class to be Bellman-complete for this environment. Meanwhile, the value-aware model error also increases with a bigger Lipschitz constant. However, we can achieve a good balance between these two errors when the Lipschitz constants are between 1000 (left dashed line) and 4000 (right dashed line), where the algorithm is observed to perform the best.
\vspace{-0.5em}
\section{Methods}\label{s4_method}
\vspace{-0.5em}
In this section, we present two different approaches to regularize the (local) Lipschitz constant of the value function.
\vspace{-0.5em}

\subsection{Spectral Normalization}
\label{s4s1_spectral}
First, to explicitly control the upper bound of the Lipschitz constant, we adapt the technique of Spectral Normalization, which is originally proposed to stabilize the training of Generative Adversarial Network (GAN) \citep{miyato2018}. By controlling the spectral norm of the weight matrix at every layer of the value network, its Lipschitz constant is upper bounded. In particular, during each forward pass, we approximate the spectral norm of the weight matrix $\|W\|_2$ with one step of power iteration. Then, we perform a projection step so that its spectral norm will be clipped to $\beta$ if bigger than $\beta$, and unchanged otherwise. See Appendix~\ref{app:implementation} for more computational details.

\subsection{Robust Regularization}
\label{s4s2_robust_reg}
Spectral Normalization directly bounds the global Lipschitz constant of the value function. However, as we argued in Section \ref{s3_theory}, we only require the local Lipschitz conditions around the model-uncertain local region. Such a strong regularization is not necessary and can even negatively impact the expressive power of value function. Now we present an alternative approach to quickly regularize the local Lipschitz constant of the value function based on a robust training mechanism. 


To regularize the local Lipschitz condition of the value network, we minimize the following loss:
\vspace{-0.6em}
\begin{equation}
    \lossreg^{(\alpha)}(Q_\phi; \epsilon, \pi_\psi, \{s_i, a_i, s'_i\}_{i=1}^{D})=\sum_{i=1}^{D}\max_{\|\tilde{s}_i-s_i\|_\alpha\leq \epsilon}\Big(Q_\phi\big(\tilde{s}_i, \pi_\psi(s'_i)\big)-Q_\phi(s_i,a_i)\Big)^2
\vspace{-0.5em}
\end{equation}
This robust loss is to guarantee that the variation of the value function locally is small.
Then we combine it with the original loss of the value function $\mathcal{L}_{\textit{org}}(Q_\phi; \pi_\psi, \{s_i, a_i, s'_i\}_{i=1}^{D})$ by minimizing 
$
    \mathcal{L}(Q_\phi; \pi_\psi, \{s_i, a_i, s'_i\}_{i=1}^{D})=\mathcal{L}_{\textit{org}}(Q_\phi; \pi_\psi, \{s_i, a_i, s'_i\}_{i=1}^{D})+\reg\lossreg^{(\alpha)}(Q_\phi; \epsilon, \pi_\psi, \{s_i, a_i, s'_i\}_{i=1}^{D}).
$

Here, $\lambda$ and $\epsilon$ are two hyperparameters. A larger $\epsilon$ is required for the dynamics model which has larger prediction errors. In terms of $\lambda$, a bigger $\lambda$ makes the value network varies smoother over the local region and gives better convergence guarantees, but it also hurts the expressive power of the value network. In practice, we find that using $\epsilon=0.1$ is often enough for good performance. So we fix $\epsilon$ and search for the best $\lambda$. For a detailed discussion on the choice of $\lambda\:$s, see Appendix~\ref{a3:robust_coeff}.

To solve the perturbation $\tilde{s}$ in the constrained optimization within the robust loss, we use the fast gradient sign method (FGSM) \citep{ian2015}. Here, let $(s, a)$ be the state action pair. Then we compute $\tilde{s}=\proj\Big(s_0+\epsilon \sign\Big(\nabla_{\tilde{s}}\big|_{\tilde{s}=s_0}\big(Q_\phi\big(\tilde{s}, \pi_\psi(\tilde{s})\big)-Q_\phi(s,a)\big)^2\Big), B^{(\alpha)}(s_0, \epsilon)\Big)$, where $s_0\sim \mathcal{N}(s, \epsilon^2I)$. Here the purpose of random initialization of $s_0$ is because the gradient $ \nabla_{s'}\big(Q_\phi\big(\tilde{s}, \pi_\psi(\tilde{s})\big)-Q_\phi(s,a)\big)^2\big)$ vanishes at $s$. In our experiments, we focus on the $l_\infty$ norm, but our proposed robust loss should be applicable to any $l_p$ norm.

In addition, we can apply more advanced constrained optimization solvers such as Projected Gradient Decent (PGD). But here, using FGSM together with a single deterministic environment model is mainly for efficiency purposes. In practice, we find that using PGD does not give much performance gain over FGSM method, and we refer the readers to Appendix \ref{a3:pgd_vs_fgsm} for the experimental results.

\section{Experiment}
\label{s5_experiment}

\subsection{Empirical Evaluations of Proposed Mechanisms}
We evaluate our two proposed training mechanisms on five MuJoCo tasks, including Walker, Humanoid, Ant, Hopper, and HalfCheetah. 
We compare the two training mechanisms with MBPO \citep{janner2019trust} using an ensemble of probabilistic transition models. 
We do not compare our methods with MBPO using an ensemble of deterministic transition models mainly because it is ensemble-based and across all five tasks, it is outperformed by probabilistic ensemble models.
In addition, since our regularization mechanisms are only trained on top of a single deterministic transition model, we compare it with both MBPO using a single deterministic transition model and a single probabilistic model. 
We implement our methods and the baseline methods based on a PyTorch implementation of MBPO~\citep{xingyu2022}.
More implementation details are provided in Appendix~\ref{app:implementation}.

\textbf{Improved asymptotic performance.} \quad
Figure~\ref{fig:main_experimental_results} presents the learning curves for all methods. These results show that using just a single deterministic model, MBPO with our two Lipschitz regularization mechanisms achieves a comparable and even better performance across all five tasks than MBPO with a probabilistic ensemble model. In particular, the proposed robust regularization technique shows a larger advantage on three more sophisticated tasks: Humanoid, Ant, and Walker. For example, on Humanoid, it achieves the same final performance as that of a probabilistic ensemble with only about 60\% of the environment interaction steps. In contrast, spectral normalization shows little improvement over a single deterministic model on these three tasks, showing the limitation of constraining the global Lipschitz constant. We provide more discussion on this in Section~\ref{s5s3:rr_sn}.

\textbf{Time Efficiency.} \quad
Compared with MBPO using an ensemble of probabilistic models, our proposed training mechanisms, especially robust regularization, is more time efficient. In Table~\ref{tab:time_compare} we record the wall clock time of the algorithms trained for 200,000 environment interaction steps on Walker and Ant, averaged over 8 random seeds. We see that robust regularization is much faster than spectral normalization. While spectral normalization takes slightly less the amount of time than probabilistic ensemble, robust regularization only takes about 70\% of the computational time. 

\begin{wraptable}{r}{0.45\textwidth}
\vspace{0em}
\caption{Comparison of computational time.}
\vspace{-10pt}
\resizebox{0.45\textwidth}{!}{
    \begin{tabular}{|c|c|c|  }
    \toprule
    \multirow{2}{*}{Algorithm} & Walker & Ant \\
    \cline{2-3}
    & Time (h) & Time(h) \\
    \hline
    Single Deterministic & 29.6  & 34.0        \\
    Probabilistic Ensemble & 60.0    & 62.3      \\
    \textbf{Robust Regularization}  & 45.4 & 48.3 \\
    \textbf{Spectral Normalization} & 54.2 & 57.1               \\    
    \bottomrule 
    \end{tabular}}
    \vspace{-10pt}
    \label{tab:time_compare}
\end{wraptable}

The significance of the experimental results is twofold. First, it further validates our insights and shows the importance of the local Lipschitz condition of value functions in MBRL. Second, it demonstrates that having an ensemble of transition models is not necessary. We can save the computational time and cost of training an ensemble of transition models with a simple Lipschitz regularization mechanism. In practice, we can even use a single probabilistic model combined with our proposed mechanisms to get the best performance. See Appendix~\ref{a3:robust_probabilistic} for the extra experimental results.

\begin{figure}[!t]
\vspace{0em}
\centering
\begin{subfigure}[t]{0.3\textwidth}
\centering
\begin{tikzpicture}[scale=0.45]

\definecolor{color0}{rgb}{0.297,0.445,0.9}
\definecolor{color1}{rgb}{0.8901960784313725, 0.10980392156862745, 0.4745098039215686}
\definecolor{color2}{rgb}{0.172549019607843,0.627450980392157,0.172549019607843}
\definecolor{color3}{rgb}{0.549019607843137,0.337254901960784,0.294117647058824}
\definecolor{color4}{rgb}{0.83921568627451,0.152941176470588,0.156862745098039}

\begin{axis}[
legend columns=6,
legend cell align={left},
legend style={
  fill opacity=0.3,
  draw opacity=1,
  text opacity=1,
  at={(-0.2,1.3)},
  anchor=north west,
  draw=white!80!black,
  font=\Large,
  column sep=0.15cm
},
tick align=outside,
tick pos=left,
title={Walker2d-v2},
x grid style={white!69.0196078431373!black},
xlabel={Number of Interactions},
xlabel style={font=\Large},
title style={font=\Large},
ylabel style={font=\Large},
xmajorgrids,
xmin=-1, xmax=19.95,
xtick style={color=black},
xtick={-1,4,9,14,19},
xticklabels={0,5e+4,10e+4,15e+4,20e+4},
y grid style={white!69.0196078431373!black},
ylabel={Performance},
ymajorgrids,
ymin=-106.005891666667, ymax=4860.67139166667,
ytick style={color=black}
]
\path [draw=color0, fill=color0, opacity=0.2]
(axis cs:0,330.731666666667)
--(axis cs:0,308.397166666667)
--(axis cs:1,486.303666666667)
--(axis cs:2,559.365333333333)
--(axis cs:3,869.950833333334)
--(axis cs:4,1417.00816666667)
--(axis cs:5,2074.6435)
--(axis cs:6,2540.46966666667)
--(axis cs:7,2813.93883333333)
--(axis cs:8,3216.2075)
--(axis cs:9,3380.97516666667)
--(axis cs:10,3465.05683333333)
--(axis cs:11,3594.284)
--(axis cs:12,3474.73683333333)
--(axis cs:13,3891.03033333333)
--(axis cs:14,3779.43166666667)
--(axis cs:15,4095.85266666667)
--(axis cs:16,3719.06916666667)
--(axis cs:17,4056.51116666667)
--(axis cs:18,4121.1405)
--(axis cs:19,4191.56883333333)
--(axis cs:19,4576.98266666667)
--(axis cs:19,4576.98266666667)
--(axis cs:18,4419.6395)
--(axis cs:17,4371.54066666667)
--(axis cs:16,4105.49383333333)
--(axis cs:15,4355.056)
--(axis cs:14,4030.0095)
--(axis cs:13,4198.49833333333)
--(axis cs:12,3951.56383333333)
--(axis cs:11,4009.7415)
--(axis cs:10,3903.72766666667)
--(axis cs:9,3804.73483333333)
--(axis cs:8,3690.71133333333)
--(axis cs:7,3316.92766666667)
--(axis cs:6,3060.23316666667)
--(axis cs:5,2680.73733333333)
--(axis cs:4,2051.1085)
--(axis cs:3,1382.28516666667)
--(axis cs:2,746.745833333333)
--(axis cs:1,569.185833333333)
--(axis cs:0,330.731666666667)
--cycle;

\path [draw=color1, fill=color1, opacity=0.2]
(axis cs:0,241.412333333333)
--(axis cs:0,189.57)
--(axis cs:1,354.478)
--(axis cs:2,483.7285)
--(axis cs:3,651.0335)
--(axis cs:4,1166.4245)
--(axis cs:5,1607.40316666667)
--(axis cs:6,1908.1245)
--(axis cs:7,2311.77116666667)
--(axis cs:8,2528.18016666667)
--(axis cs:9,2694.73216666667)
--(axis cs:10,3101.60883333333)
--(axis cs:11,2707.627)
--(axis cs:12,3094.2335)
--(axis cs:13,3460.712)
--(axis cs:14,3307.45116666667)
--(axis cs:15,3041.49816666667)
--(axis cs:16,3288.781)
--(axis cs:17,3214.40466666667)
--(axis cs:18,2928.31233333333)
--(axis cs:19,3642.15283333333)
--(axis cs:19,3927.93316666667)
--(axis cs:19,3927.93316666667)
--(axis cs:18,3512.3565)
--(axis cs:17,3538.029)
--(axis cs:16,3612.94816666667)
--(axis cs:15,3346.25866666667)
--(axis cs:14,3459.35366666667)
--(axis cs:13,3677.53433333333)
--(axis cs:12,3384.3795)
--(axis cs:11,3117.98166666667)
--(axis cs:10,3319.25566666667)
--(axis cs:9,3098.59283333333)
--(axis cs:8,3076.81933333333)
--(axis cs:7,2613.01416666667)
--(axis cs:6,2281.68466666667)
--(axis cs:5,1958.40483333333)
--(axis cs:4,1494.4815)
--(axis cs:3,787.8885)
--(axis cs:2,523.250833333333)
--(axis cs:1,404.457166666667)
--(axis cs:0,241.412333333333)
--cycle;

\path [draw=color2, fill=color2, opacity=0.2]
(axis cs:0,187.925333333333)
--(axis cs:0,119.752166666667)
--(axis cs:1,287.28)
--(axis cs:2,334.262333333333)
--(axis cs:3,486.314666666667)
--(axis cs:4,581.077)
--(axis cs:5,788.916166666667)
--(axis cs:6,988.836833333333)
--(axis cs:7,1039.48633333333)
--(axis cs:8,1274.37633333333)
--(axis cs:9,1492.05483333333)
--(axis cs:10,1929.03233333333)
--(axis cs:11,1827.306)
--(axis cs:12,2131.71133333333)
--(axis cs:13,2220.894)
--(axis cs:14,2066.04733333333)
--(axis cs:15,2693.695)
--(axis cs:16,2744.9825)
--(axis cs:17,2469.379)
--(axis cs:18,2758.548)
--(axis cs:19,2836.58766666667)
--(axis cs:19,3269.14066666667)
--(axis cs:19,3269.14066666667)
--(axis cs:18,3294.58566666667)
--(axis cs:17,2956.41033333333)
--(axis cs:16,3215.5095)
--(axis cs:15,3256.498)
--(axis cs:14,2656.88516666667)
--(axis cs:13,2968.30783333333)
--(axis cs:12,2871.79383333333)
--(axis cs:11,2501.10366666667)
--(axis cs:10,2578.0955)
--(axis cs:9,2100.968)
--(axis cs:8,1934.03533333333)
--(axis cs:7,1724.25416666667)
--(axis cs:6,1542.376)
--(axis cs:5,1208.52866666667)
--(axis cs:4,820.970666666667)
--(axis cs:3,715.749833333333)
--(axis cs:2,494.763)
--(axis cs:1,341.403833333333)
--(axis cs:0,187.925333333333)
--cycle;

\path [draw=color3, fill=color3, opacity=0.2]
(axis cs:0,271.265666666667)
--(axis cs:0,225.4315)
--(axis cs:1,334.479833333333)
--(axis cs:2,475.13)
--(axis cs:3,696.401666666667)
--(axis cs:4,1333.50683333333)
--(axis cs:5,1847.6495)
--(axis cs:6,2289.1945)
--(axis cs:7,2747.941)
--(axis cs:8,2985.584)
--(axis cs:9,3282.51433333333)
--(axis cs:10,3166.25466666667)
--(axis cs:11,3757.23883333333)
--(axis cs:12,3728.80433333333)
--(axis cs:13,3952.60566666667)
--(axis cs:14,3662.75916666667)
--(axis cs:15,3497.5005)
--(axis cs:16,3702.91066666667)
--(axis cs:17,3889.30316666667)
--(axis cs:18,4193.2635)
--(axis cs:19,3685.36666666667)
--(axis cs:19,4163.034)
--(axis cs:19,4163.034)
--(axis cs:18,4373.47883333333)
--(axis cs:17,4148.79166666667)
--(axis cs:16,4141.826)
--(axis cs:15,3697.5315)
--(axis cs:14,3929.757)
--(axis cs:13,4148.63566666667)
--(axis cs:12,4106.4195)
--(axis cs:11,4101.66866666667)
--(axis cs:10,3620.82983333333)
--(axis cs:9,3642.9905)
--(axis cs:8,3375.83583333333)
--(axis cs:7,3086.337)
--(axis cs:6,2808.633)
--(axis cs:5,2344.02166666667)
--(axis cs:4,1665.98483333333)
--(axis cs:3,888.095833333333)
--(axis cs:2,581.018)
--(axis cs:1,357.460666666667)
--(axis cs:0,271.265666666667)
--cycle;

\path [draw=color4, fill=color4, opacity=0.2]
(axis cs:0,401.433333333333)
--(axis cs:0,362.566666666667)
--(axis cs:1,370.5)
--(axis cs:2,528.6)
--(axis cs:3,582.02)
--(axis cs:4,1049.949)
--(axis cs:5,2027.39866666667)
--(axis cs:6,2649.84)
--(axis cs:7,3000.123)
--(axis cs:8,3812.80566666667)
--(axis cs:9,3935.91333333333)
--(axis cs:10,4042.559)
--(axis cs:11,3828.69016666667)
--(axis cs:12,3871.53333333333)
--(axis cs:13,4232.325)
--(axis cs:14,4317.63716666667)
--(axis cs:15,4485.06483333333)
--(axis cs:16,4304.98)
--(axis cs:17,4481.0695)
--(axis cs:18,4583.18)
--(axis cs:19,4324.94666666667)
--(axis cs:19,4566.64333333333)
--(axis cs:19,4566.64333333333)
--(axis cs:18,4634.91333333333)
--(axis cs:17,4634.72666666667)
--(axis cs:16,4510.88666666667)
--(axis cs:15,4592.56666666667)
--(axis cs:14,4498.62333333333)
--(axis cs:13,4320.16)
--(axis cs:12,4184.51333333333)
--(axis cs:11,4172.03666666667)
--(axis cs:10,4093.58)
--(axis cs:9,4158.05)
--(axis cs:8,3950.49633333333)
--(axis cs:7,3249.7955)
--(axis cs:6,2813.19516666667)
--(axis cs:5,2409.81333333333)
--(axis cs:4,1326.063)
--(axis cs:3,663.036833333333)
--(axis cs:2,633.446666666666)
--(axis cs:1,405.25)
--(axis cs:0,401.433333333333)
--cycle;

\addplot [very thick, color0]
table {%
0 319.106723333333
1 525.371996666667
2 648.493706666667
3 1124.08495666667
4 1730.5093
5 2378.88952333333
6 2800.61892666667
7 3068.19365666667
8 3457.12203
9 3601.32567
10 3692.91555333333
11 3805.98439333333
12 3710.32027666667
13 4045.97511333333
14 3904.60232333333
15 4234.63694333333
16 3921.47660333333
17 4215.47429333333
18 4268.73579333333
19 4385.62110333333
};
\addlegendentry{Probabilistic Ensemble}
\addplot [very thick, color1]
table {%
0 216.607216666667
1 379.058283333333
2 502.89193
3 716.954396666667
4 1327.66286
5 1787.98860333333
6 2096.64118333333
7 2466.8798
8 2818.69895333333
9 2889.76509333333
10 3211.36817666667
11 2917.86926666667
12 3237.52561333333
13 3570.93092666667
14 3378.66685
15 3191.64272
16 3447.89665
17 3376.67198333333
18 3235.99636
19 3782.31749333333
};
\addlegendentry{Single Probabilistic}
\addplot [very thick, color2]
table {%
0 152.506256666667
1 313.327946666667
2 410.496106666667
3 593.007603333333
4 698.0559
5 992.773033333333
6 1263.30976
7 1374.29993666667
8 1600.36568333333
9 1780.68887333333
10 2252.57982
11 2165.22296
12 2515.08885
13 2589.71303666667
14 2357.09003
15 2977.65136333333
16 2987.60988666667
17 2711.48258
18 3035.76863666667
19 3054.13960666667
};
\addlegendentry{Single Deterministic}
\addplot [very thick, color3]
table {%
0 246.843206666667
1 345.370946666667
2 527.59894
3 793.118496666667
4 1493.94405
5 2091.98061666667
6 2548.28407666667
7 2921.46928666667
8 3184.1218
9 3469.44535
10 3398.95080666667
11 3923.89686
12 3912.76365666667
13 4043.03951666667
14 3794.58625
15 3602.37595333333
16 3924.02480333333
17 4020.98356333333
18 4280.34473333333
19 3935.26397666667
};
\addlegendentry{Spectral Normalization (Ours)}
\addplot [very thick, color4]
table {%
0 381.738566666667
1 388.28634
2 580.889883333333
3 622.892016666667
4 1184.60893333333
5 2241.85064
6 2733.49242333333
7 3124.19287333333
8 3881.17836
9 4046.73972333333
10 4068.97074
11 4001.53115666667
12 4037.46456666667
13 4276.32876333333
14 4410.5474
15 4538.18766
16 4414.42896
17 4557.37663
18 4609.99136
19 4449.11221666667
};
\addlegendentry{Robust Regularization (Ours)}
\end{axis}

\end{tikzpicture}
 \end{subfigure}
 \hfill
 \begin{subfigure}[t]{0.3\textwidth}
\begin{tikzpicture}[scale=0.45]

\definecolor{color0}{rgb}{0.297,0.445,0.9}
\definecolor{color1}{rgb}{0.8901960784313725, 0.10980392156862745, 0.4745098039215686}
\definecolor{color2}{rgb}
{0.172549019607843,0.627450980392157,0.172549019607843}
\definecolor{color3}{rgb}{0.83921568627451,0.152941176470588,0.156862745098039}
\definecolor{color4}{rgb}{0.549019607843137,0.337254901960784,0.294117647058824}

\begin{axis}[
tick align=outside,
tick pos=left,
title={Ant-v2},
x grid style={white!69.0196078431373!black},
xlabel={Number of Interactions},
xlabel style={font=\Large},
title style={font=\Large},
ylabel style={font=\Large},
xmajorgrids,
xmin=-1, xmax=19.95,
xtick style={color=black},
xtick={-1,4,9,14,19},
xticklabels={0,5e+4,10e+4,15e+4,20e+4},
y grid style={white!69.0196078431373!black},
ylabel={Performance},
ymajorgrids,
ymin=383.549816666667, ymax=5068.66651666667,
ytick style={color=black}
]
\path [draw=color0, fill=color0, opacity=0.2]
(axis cs:0,721.180666666667)
--(axis cs:0,697.316166666667)
--(axis cs:1,852.138833333333)
--(axis cs:2,854.710833333333)
--(axis cs:3,1043.8965)
--(axis cs:4,1227.195)
--(axis cs:5,1508.34983333333)
--(axis cs:6,1798.42266666667)
--(axis cs:7,2115.8805)
--(axis cs:8,2358.8335)
--(axis cs:9,2552.0325)
--(axis cs:10,2839.64566666667)
--(axis cs:11,3200.2675)
--(axis cs:12,3327.494)
--(axis cs:13,3449.6575)
--(axis cs:14,3484.59883333333)
--(axis cs:15,3642.41283333333)
--(axis cs:16,3699.05883333333)
--(axis cs:17,3973.25516666667)
--(axis cs:18,4081.3175)
--(axis cs:19,4025.1165)
--(axis cs:19,4556.52666666667)
--(axis cs:19,4556.52666666667)
--(axis cs:18,4585.82133333333)
--(axis cs:17,4481.2185)
--(axis cs:16,4245.287)
--(axis cs:15,4085.33866666667)
--(axis cs:14,3906.40516666667)
--(axis cs:13,3855.93583333333)
--(axis cs:12,3735.25916666667)
--(axis cs:11,3560.84116666667)
--(axis cs:10,3180.41133333333)
--(axis cs:9,2907.40516666667)
--(axis cs:8,2693.8725)
--(axis cs:7,2340.83783333333)
--(axis cs:6,1956.94766666667)
--(axis cs:5,1673.55266666667)
--(axis cs:4,1382.14883333333)
--(axis cs:3,1148.58966666667)
--(axis cs:2,872.889166666667)
--(axis cs:1,866.2635)
--(axis cs:0,721.180666666667)
--cycle;

\path [draw=color1, fill=color1, opacity=0.2]
(axis cs:0,654.556833333333)
--(axis cs:0,631.1115)
--(axis cs:1,769.31)
--(axis cs:2,816.995833333333)
--(axis cs:3,820.642833333333)
--(axis cs:4,961.676666666667)
--(axis cs:5,1243.629)
--(axis cs:6,1349.071)
--(axis cs:7,1523.25)
--(axis cs:8,1810.91283333333)
--(axis cs:9,1992.748)
--(axis cs:10,2187.247)
--(axis cs:11,2194.76233333333)
--(axis cs:12,2429.7595)
--(axis cs:13,2522.22466666667)
--(axis cs:14,2667.8655)
--(axis cs:15,2801.75466666667)
--(axis cs:16,2928.106)
--(axis cs:17,2976.98666666667)
--(axis cs:18,2959.07516666667)
--(axis cs:19,3022.46483333333)
--(axis cs:19,3220.32083333333)
--(axis cs:19,3220.32083333333)
--(axis cs:18,3183.32583333333)
--(axis cs:17,3131.18983333333)
--(axis cs:16,3151.58583333333)
--(axis cs:15,3070.24333333333)
--(axis cs:14,2935.74333333333)
--(axis cs:13,2757.23683333333)
--(axis cs:12,2576.23716666667)
--(axis cs:11,2451.83583333333)
--(axis cs:10,2388.0895)
--(axis cs:9,2179.5395)
--(axis cs:8,1984.81333333333)
--(axis cs:7,1722.2625)
--(axis cs:6,1549.498)
--(axis cs:5,1352.00333333333)
--(axis cs:4,1008.44)
--(axis cs:3,838.1)
--(axis cs:2,838.114166666667)
--(axis cs:1,784.507166666667)
--(axis cs:0,654.556833333333)
--cycle;

\path [draw=color2, fill=color2, opacity=0.2]
(axis cs:0,711.310666666667)
--(axis cs:0,683.776166666667)
--(axis cs:1,732.202166666667)
--(axis cs:2,705.79)
--(axis cs:3,725.386666666667)
--(axis cs:4,763.410833333333)
--(axis cs:5,789.4715)
--(axis cs:6,864.3465)
--(axis cs:7,863.286833333334)
--(axis cs:8,884.483)
--(axis cs:9,983.9595)
--(axis cs:10,931.999)
--(axis cs:11,886.339)
--(axis cs:12,1165.09233333333)
--(axis cs:13,1287.25916666667)
--(axis cs:14,1354.54883333333)
--(axis cs:15,1165.4785)
--(axis cs:16,1236.311)
--(axis cs:17,1446.55716666667)
--(axis cs:18,1474.002)
--(axis cs:19,1495.56833333333)
--(axis cs:19,1877.309)
--(axis cs:19,1877.309)
--(axis cs:18,1900.86166666667)
--(axis cs:17,1679.2385)
--(axis cs:16,1479.91166666667)
--(axis cs:15,1426.16583333333)
--(axis cs:14,1747.94333333333)
--(axis cs:13,1666.40783333333)
--(axis cs:12,1387.23866666667)
--(axis cs:11,1102.98716666667)
--(axis cs:10,1232.83216666667)
--(axis cs:9,1294.54916666667)
--(axis cs:8,1181.4475)
--(axis cs:7,1060.53566666667)
--(axis cs:6,977.730166666667)
--(axis cs:5,890.869)
--(axis cs:4,841.358)
--(axis cs:3,795.6435)
--(axis cs:2,757.346166666667)
--(axis cs:1,805.970833333333)
--(axis cs:0,711.310666666667)
--cycle;

\path [draw=color3, fill=color3, opacity=0.2]
(axis cs:0,746.653333333333)
--(axis cs:0,703.862333333333)
--(axis cs:1,746.194666666667)
--(axis cs:2,908.9825)
--(axis cs:3,1004.17433333333)
--(axis cs:4,1165.01616666667)
--(axis cs:5,1564.29433333333)
--(axis cs:6,1573.227)
--(axis cs:7,1961.59816666667)
--(axis cs:8,2467.4045)
--(axis cs:9,2842.451)
--(axis cs:10,3099.1265)
--(axis cs:11,3204.62533333333)
--(axis cs:12,3549.19083333333)
--(axis cs:13,3663.578)
--(axis cs:14,3716.75)
--(axis cs:15,4089.2925)
--(axis cs:16,3987.81583333333)
--(axis cs:17,4165.70016666667)
--(axis cs:18,4287.68333333333)
--(axis cs:19,4411.58966666667)
--(axis cs:19,4770.87083333333)
--(axis cs:19,4770.87083333333)
--(axis cs:18,4855.70666666667)
--(axis cs:17,4637.07183333333)
--(axis cs:16,4456.00366666667)
--(axis cs:15,4631.79)
--(axis cs:14,4268.28833333333)
--(axis cs:13,4338.663)
--(axis cs:12,3952.8525)
--(axis cs:11,3766.03666666667)
--(axis cs:10,3691.82483333333)
--(axis cs:9,3378.2)
--(axis cs:8,2882.80916666667)
--(axis cs:7,2487.13266666667)
--(axis cs:6,1897.4585)
--(axis cs:5,1788.95716666667)
--(axis cs:4,1297.66116666667)
--(axis cs:3,1095.93666666667)
--(axis cs:2,976.22)
--(axis cs:1,804.36)
--(axis cs:0,746.653333333333)
--cycle;

\path [draw=color4, fill=color4, opacity=0.2]
(axis cs:0,622.130166666667)
--(axis cs:0,596.509666666667)
--(axis cs:1,730.51)
--(axis cs:2,675.946833333333)
--(axis cs:3,858.486666666667)
--(axis cs:4,665.5265)
--(axis cs:5,744.515)
--(axis cs:6,701.79)
--(axis cs:7,789.703666666667)
--(axis cs:8,902.766166666667)
--(axis cs:9,1103.64783333333)
--(axis cs:10,942.215666666667)
--(axis cs:11,1044.49466666667)
--(axis cs:12,1209.40433333333)
--(axis cs:13,1421.60116666667)
--(axis cs:14,1427.32916666667)
--(axis cs:15,1569.01866666667)
--(axis cs:16,1552.56683333333)
--(axis cs:17,1661.3035)
--(axis cs:18,1832.63566666667)
--(axis cs:19,1911.2905)
--(axis cs:19,2187.6415)
--(axis cs:19,2187.6415)
--(axis cs:18,2299.84666666667)
--(axis cs:17,2035.19683333333)
--(axis cs:16,2003.90133333333)
--(axis cs:15,1802.58616666667)
--(axis cs:14,1562.7555)
--(axis cs:13,1483.38216666667)
--(axis cs:12,1387.26683333333)
--(axis cs:11,1270.511)
--(axis cs:10,1144.29766666667)
--(axis cs:9,1267.1305)
--(axis cs:8,997.073)
--(axis cs:7,884.495)
--(axis cs:6,796.915)
--(axis cs:5,817.352333333333)
--(axis cs:4,745.4695)
--(axis cs:3,903.872166666667)
--(axis cs:2,736.617833333333)
--(axis cs:1,791.183333333333)
--(axis cs:0,622.130166666667)
--cycle;

\addplot [very thick, color0]
table {%
0 709.108443333333
1 859.22842
2 864.319716666667
3 1096.30969
4 1305.00027666667
5 1588.47200666667
6 1879.03232666667
7 2223.62673666667
8 2526.14884666667
9 2725.64920666667
10 3013.07480666667
11 3380.37587333333
12 3529.963
13 3657.05184333333
14 3696.09795333333
15 3869.94844
16 3974.03432666667
17 4220.01981333333
18 4332.86899666667
19 4290.73553666667
};

\addplot [very thick, color1]
table {%
0 643.025336666667
1 776.664716666667
2 827.81383
3 829.717116666667
4 983.438033333333
5 1297.55545
6 1448.85550666667
7 1622.56993
8 1900.04054333333
9 2083.30494
10 2285.7486
11 2327.27956
12 2504.76244333333
13 2643.77668666667
14 2804.90628333333
15 2940.14061
16 3042.42854666667
17 3056.64343666667
18 3069.20528333333
19 3121.97628333333
};

\addplot [very thick, color2]
table {%
0 697.986246666667
1 770.888986666667
2 732.39411
3 760.26516
4 802.658316666667
5 839.966643333333
6 921.28002
7 954.675503333333
8 1022.36403
9 1128.64292333333
10 1074.89426
11 995.151396666667
12 1276.80778666667
13 1471.80970333333
14 1546.86696333333
15 1291.02033666667
16 1354.26341
17 1560.79248666667
18 1687.24866666667
19 1684.56139
};

\addplot [very thick, color3]
table {%
0 724.908613333333
1 774.771593333333
2 941.266093333333
3 1050.34499666667
4 1231.46013333333
5 1676.91772
6 1737.56132333333
7 2226.83158
8 2668.13829
9 3098.73078
10 3403.42881333333
11 3486.41505
12 3745.44998
13 3992.33724
14 3991.21147666667
15 4345.44935333333
16 4228.77318
17 4388.80469666667
18 4570.04735333333
19 4592.65983666667
};

\addplot [very thick, color4]
table {%
0 609.0872
1 760.54247
2 706.524103333333
3 882.598753333334
4 705.545776666667
5 782.692926666667
6 749.449183333333
7 835.919356666667
8 948.050016666667
9 1185.01391666667
10 1040.66494
11 1162.36566
12 1296.96582666667
13 1454.11262333333
14 1494.70259333333
15 1680.40125
16 1784.75271333333
17 1841.93343
18 2062.49300333333
19 2047.56016
};

\end{axis}

\end{tikzpicture}
 \end{subfigure}
 \hfill
  \begin{subfigure}[t]{0.3\textwidth}
  \input{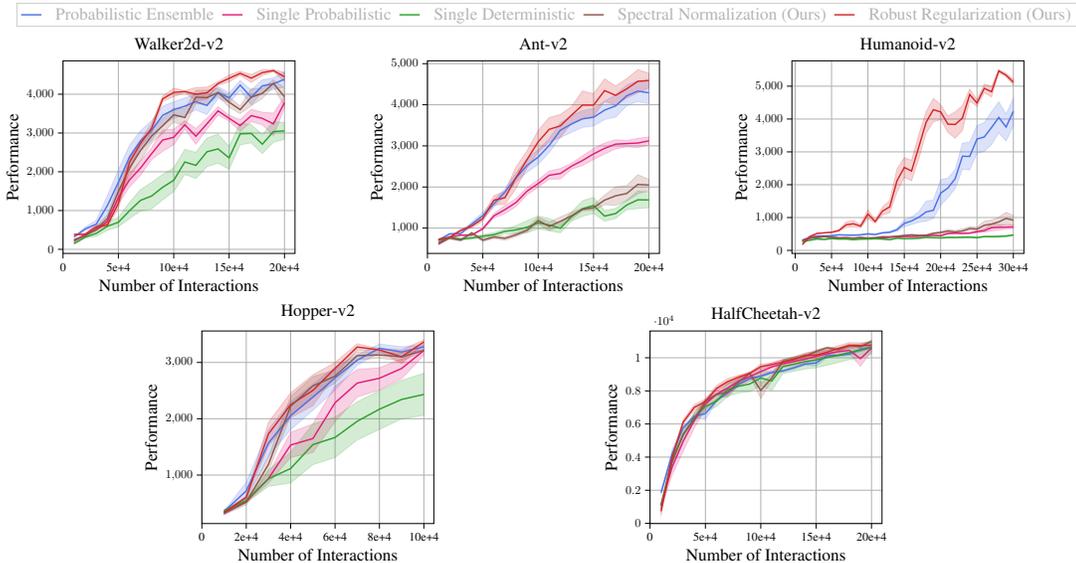} 
 \end{subfigure}
 \vspace{-1em}

 \hspace{5em}
 \begin{subfigure}[t]{0.43\textwidth}
\begin{tikzpicture}[scale=0.45]

\definecolor{color0}{rgb}{0.297,0.445,0.9}
\definecolor{color1}{rgb}{0.8901960784313725, 0.10980392156862745, 0.4745098039215686}
\definecolor{color2}{rgb}
{0.172549019607843,0.627450980392157,0.172549019607843}
\definecolor{color3}{rgb}{0.83921568627451,0.152941176470588,0.156862745098039}
\definecolor{color4}{rgb}{0.549019607843137,0.337254901960784,0.294117647058824}

\begin{axis}[
tick align=outside,
tick pos=left,
title={Hopper-v2},
x grid style={white!69.0196078431373!black},
xlabel={Number of Interactions},
xmajorgrids,
xmin=-1, xmax=9.45,
xtick style={color=black},
xtick={-1,1,3,5,7,9},
xlabel style={font=\Large},
title style={font=\Large},
ylabel style={font=\Large},
xticklabels={0,2e+4,4e+4,6e+4,8e+4,10e+4},
y grid style={white!69.0196078431373!black},
ylabel={Performance},
ymajorgrids,
ymin=138.890991666667, ymax=3558.758175,
ytick style={color=black}
]
\path [draw=color0, fill=color0, opacity=0.2]
(axis cs:0,361.2245)
--(axis cs:0,341.345666666667)
--(axis cs:1,591.8985)
--(axis cs:2,1333.666)
--(axis cs:3,1804.89066666667)
--(axis cs:4,2150.93366666667)
--(axis cs:5,2505.43333333333)
--(axis cs:6,2950.789)
--(axis cs:7,3173.35066666667)
--(axis cs:8,3092.18333333333)
--(axis cs:9,3235.731)
--(axis cs:9,3316.542)
--(axis cs:9,3316.542)
--(axis cs:8,3278.1485)
--(axis cs:7,3327.32333333333)
--(axis cs:6,3131.83333333333)
--(axis cs:5,2919.62666666667)
--(axis cs:4,2610.16433333333)
--(axis cs:3,2301.035)
--(axis cs:2,1807.7495)
--(axis cs:1,872.288166666667)
--(axis cs:0,361.2245)
--cycle;

\path [draw=color1, fill=color1, opacity=0.2]
(axis cs:0,351.961666666667)
--(axis cs:0,299.215)
--(axis cs:1,504.210333333333)
--(axis cs:2,853.7025)
--(axis cs:3,1312.87766666667)
--(axis cs:4,1398.55366666667)
--(axis cs:5,1948.07316666667)
--(axis cs:6,2386.02066666667)
--(axis cs:7,2511.37033333333)
--(axis cs:8,2711.684)
--(axis cs:9,3126.67566666667)
--(axis cs:9,3299.2105)
--(axis cs:9,3299.2105)
--(axis cs:8,3064.65716666667)
--(axis cs:7,2905.61483333333)
--(axis cs:6,2877.1955)
--(axis cs:5,2577.724)
--(axis cs:4,1908.74)
--(axis cs:3,1758.88333333333)
--(axis cs:2,1022.512)
--(axis cs:1,558.3105)
--(axis cs:0,351.961666666667)
--cycle;

\path [draw=color2, fill=color2, opacity=0.2]
(axis cs:0,392.677666666667)
--(axis cs:0,348.082333333333)
--(axis cs:1,508.998833333333)
--(axis cs:2,798.4085)
--(axis cs:3,860.3365)
--(axis cs:4,1184.4575)
--(axis cs:5,1311.2815)
--(axis cs:6,1623.95533333333)
--(axis cs:7,1811.37616666667)
--(axis cs:8,1997.025)
--(axis cs:9,2061.29716666667)
--(axis cs:9,2809.01183333333)
--(axis cs:9,2809.01183333333)
--(axis cs:8,2681.15833333333)
--(axis cs:7,2498.80716666667)
--(axis cs:6,2300.95433333333)
--(axis cs:5,2030.57533333333)
--(axis cs:4,1906.25016666667)
--(axis cs:3,1408.2165)
--(axis cs:2,1090.75683333333)
--(axis cs:1,587.314666666667)
--(axis cs:0,392.677666666667)
--cycle;

\path [draw=color3, fill=color3, opacity=0.2]
(axis cs:0,351.8625)
--(axis cs:0,319.6365)
--(axis cs:1,584.3245)
--(axis cs:2,1512.36)
--(axis cs:3,2019.02833333333)
--(axis cs:4,2232.271)
--(axis cs:5,2777.05666666667)
--(axis cs:6,3220.327)
--(axis cs:7,3199.69333333333)
--(axis cs:8,2978.45133333333)
--(axis cs:9,3326.56)
--(axis cs:9,3403.30966666667)
--(axis cs:9,3403.30966666667)
--(axis cs:8,3222.082)
--(axis cs:7,3245.702)
--(axis cs:6,3331.92)
--(axis cs:5,2998.75766666667)
--(axis cs:4,2780.14766666667)
--(axis cs:3,2461.07683333333)
--(axis cs:2,1937.76666666667)
--(axis cs:1,629.582166666667)
--(axis cs:0,351.8625)
--cycle;

\path [draw=color4, fill=color4, opacity=0.2]
(axis cs:0,345.420833333333)
--(axis cs:0,294.3395)
--(axis cs:1,472.195666666667)
--(axis cs:2,1065.4915)
--(axis cs:3,2050.08833333333)
--(axis cs:4,2398.35733333333)
--(axis cs:5,2633.76466666667)
--(axis cs:6,3046.59133333333)
--(axis cs:7,3076.30083333333)
--(axis cs:8,3026.6025)
--(axis cs:9,3201.12633333333)
--(axis cs:9,3225.688)
--(axis cs:9,3225.688)
--(axis cs:8,3168.97666666667)
--(axis cs:7,3187.719)
--(axis cs:6,3191.70116666667)
--(axis cs:5,2865.58216666667)
--(axis cs:4,2748.309)
--(axis cs:3,2398.179)
--(axis cs:2,1325.00866666667)
--(axis cs:1,574.843833333333)
--(axis cs:0,345.420833333333)
--cycle;

\addplot [very thick, color0]
table {%
0 351.82133
1 719.85446
2 1569.02585
3 2055.89706666667
4 2382.00984
5 2725.03385
6 3039.96683
7 3252.02233666667
8 3184.87284666667
9 3278.74332
};
\addplot [very thick, color1]
table {%
0 325.241773333333
1 531.4198
2 937.503456666667
3 1533.29424
4 1648.42443
5 2284.09290666667
6 2633.80503
7 2719.68228666667
8 2892.96552333333
9 3211.82945333333
};
\addplot [very thick, color2]
table {%
0 369.128926666667
1 545.786243333333
2 933.92177
3 1116.43118
4 1540.99721333333
5 1667.69106666667
6 1955.39243333333
7 2169.96652666667
8 2343.3937
9 2432.06039333333
};
\addplot [very thick, color3]
table {%
0 335.61804
1 606.77568
2 1731.35589666667
3 2250.00703333333
4 2502.42581333333
5 2893.55772
6 3273.97019
7 3222.23983333333
8 3100.41070666667
9 3364.52017666667
};
\addplot [very thick, color4]
table {%
0 319.065476666667
1 522.38357
2 1189.87637
3 2223.56247333333
4 2583.50964333333
5 2755.75577
6 3120.50168666667
7 3134.96012
8 3098.78845333333
9 3213.57925666667
};
\end{axis}

\end{tikzpicture} 
 \end{subfigure}
 \hfill
 \begin{subfigure}[t]{0.43\textwidth}
\begin{tikzpicture}[scale=0.45]

\definecolor{color0}{rgb}{0.297,0.445,0.9}
\definecolor{color1}{rgb}{0.8901960784313725, 0.10980392156862745, 0.4745098039215686}
\definecolor{color2}{rgb}
{0.172549019607843,0.627450980392157,0.172549019607843}
\definecolor{color3}{rgb}{0.549019607843137,0.337254901960784,0.294117647058824}
\definecolor{color4}{rgb}{0.83921568627451,0.152941176470588,0.156862745098039}

\begin{axis}[
legend cell align={left},
legend style={
  fill opacity=0.3,
  draw opacity=1,
  text opacity=1,
  at={(0.97,0.03)},
  anchor=south east,
  draw=white!80!black
},
tick align=outside,
tick pos=left,
title={HalfCheetah-v2},
xlabel style={font=\Large},
title style={font=\Large},
ylabel style={font=\Large},
x grid style={white!69.0196078431373!black},
xlabel={Number of Interactions},
xmajorgrids,
xmin=-1, xmax=19.95,
xtick style={color=black},
xtick={-1,4,9,14,19},
xticklabels={0,5e+4,10e+4,15e+4,20e+4},
y grid style={white!69.0196078431373!black},
ylabel={Performance},
ymajorgrids,
ymin=-13.5507333333335, ymax=11617.9094,
ytick style={color=black}
]
\path [draw=color0, fill=color0, opacity=0.2]
(axis cs:0,1919.98269999588)
--(axis cs:0,1798.42675986687)
--(axis cs:1,3900.17538870853)
--(axis cs:2,5652.19731727149)
--(axis cs:3,6390.47039055218)
--(axis cs:4,6264.10234010428)
--(axis cs:5,7335.39735297416)
--(axis cs:6,7839.33434388348)
--(axis cs:7,8323.28436147108)
--(axis cs:8,8694.27085164444)
--(axis cs:9,8810.1042396506)
--(axis cs:10,9031.66658733355)
--(axis cs:11,9140.70976878258)
--(axis cs:12,9309.08714680512)
--(axis cs:13,9548.00436915718)
--(axis cs:14,9500.64397547634)
--(axis cs:15,10031.3363609921)
--(axis cs:16,10072.6559840643)
--(axis cs:17,10140.956195336)
--(axis cs:18,10382.0073075551)
--(axis cs:19,10523.1902278343)
--(axis cs:19,10780.7424724275)
--(axis cs:19,10780.7424724275)
--(axis cs:18,10644.0326612621)
--(axis cs:17,10284.7431851175)
--(axis cs:16,10204.9408390716)
--(axis cs:15,10226.5474313821)
--(axis cs:14,9857.42796208621)
--(axis cs:13,9693.17492439232)
--(axis cs:12,9487.4327029082)
--(axis cs:11,9268.82951154724)
--(axis cs:10,9185.53572589939)
--(axis cs:9,8947.6424362269)
--(axis cs:8,8794.73909096735)
--(axis cs:7,8463.62738675102)
--(axis cs:6,7925.49906170016)
--(axis cs:5,7492.1837809855)
--(axis cs:4,6943.53492029539)
--(axis cs:3,6546.0706446718)
--(axis cs:2,5890.52871549394)
--(axis cs:1,4462.35852788242)
--(axis cs:0,1919.98269999588)
--cycle;

\path [draw=color1, fill=color1, opacity=0.2]
(axis cs:0,1330.97766666667)
--(axis cs:0,875.634666666667)
--(axis cs:1,2910.4105)
--(axis cs:2,4362.31683333333)
--(axis cs:3,5760.18916666667)
--(axis cs:4,6877.63033333333)
--(axis cs:5,7465.67833333333)
--(axis cs:6,7753.33416666667)
--(axis cs:7,8277.51233333333)
--(axis cs:8,8583.48133333333)
--(axis cs:9,8893.03683333334)
--(axis cs:10,9182.23883333333)
--(axis cs:11,9387.19466666667)
--(axis cs:12,9541.0185)
--(axis cs:13,9686.25366666666)
--(axis cs:14,9889.284)
--(axis cs:15,10024.9653333333)
--(axis cs:16,10152.924)
--(axis cs:17,10244.4043333333)
--(axis cs:18,9491.03383333333)
--(axis cs:19,10366.8406666667)
--(axis cs:19,10785.8581666667)
--(axis cs:19,10785.8581666667)
--(axis cs:18,10385.9573333333)
--(axis cs:17,10653.658)
--(axis cs:16,10568.5155)
--(axis cs:15,10471.421)
--(axis cs:14,10289.0318333333)
--(axis cs:13,10127.4986666667)
--(axis cs:12,10004.4346666667)
--(axis cs:11,9887.83083333333)
--(axis cs:10,9700.13566666667)
--(axis cs:9,9431.92683333333)
--(axis cs:8,9133.67933333333)
--(axis cs:7,8905.4385)
--(axis cs:6,8546.2085)
--(axis cs:5,8139.05533333333)
--(axis cs:4,7515.2145)
--(axis cs:3,6704.23283333333)
--(axis cs:2,5474.24666666667)
--(axis cs:1,3881.31166666667)
--(axis cs:0,1330.97766666667)
--cycle;

\path [draw=color2, fill=color2, opacity=0.2]
(axis cs:0,1250.37716666667)
--(axis cs:0,855.862166666667)
--(axis cs:1,3582.53683333333)
--(axis cs:2,5028.8395)
--(axis cs:3,6118.5155)
--(axis cs:4,6760.39766666667)
--(axis cs:5,7048.7675)
--(axis cs:6,7667.84866666667)
--(axis cs:7,7875.6855)
--(axis cs:8,7943.5)
--(axis cs:9,8315.17983333334)
--(axis cs:10,8021.40283333333)
--(axis cs:11,9045.02833333333)
--(axis cs:12,9183.51583333333)
--(axis cs:13,9317.5695)
--(axis cs:14,9393.98316666667)
--(axis cs:15,9548.32983333334)
--(axis cs:16,9719.1445)
--(axis cs:17,9889.75683333333)
--(axis cs:18,10005.2598333333)
--(axis cs:19,10197.8023333333)
--(axis cs:19,10998.9496666667)
--(axis cs:19,10998.9496666667)
--(axis cs:18,10849.0318333333)
--(axis cs:17,10751.3891666667)
--(axis cs:16,10546.4716666667)
--(axis cs:15,10462.0951666667)
--(axis cs:14,10334.0571666667)
--(axis cs:13,10211.2078333333)
--(axis cs:12,10065.5198333333)
--(axis cs:11,9862.7765)
--(axis cs:10,9122.98166666667)
--(axis cs:9,9257.78716666667)
--(axis cs:8,8852.96)
--(axis cs:7,8688.21266666667)
--(axis cs:6,8466.0165)
--(axis cs:5,7684.8495)
--(axis cs:4,7365.60483333333)
--(axis cs:3,6669.20066666667)
--(axis cs:2,5628.008)
--(axis cs:1,4167.1515)
--(axis cs:0,1250.37716666667)
--cycle;

\path [draw=color3, fill=color3, opacity=0.2]
(axis cs:0,1074.2595)
--(axis cs:0,1025.55)
--(axis cs:1,3675.88833333333)
--(axis cs:2,5363.4045)
--(axis cs:3,6133.31666666667)
--(axis cs:4,7282.93333333333)
--(axis cs:5,7732.8225)
--(axis cs:6,7597.36333333333)
--(axis cs:7,8367.65316666667)
--(axis cs:8,8960.62)
--(axis cs:9,7536.5)
--(axis cs:10,8613.48)
--(axis cs:11,9678.22)
--(axis cs:12,9816.76)
--(axis cs:13,9955.72333333333)
--(axis cs:14,10300.52)
--(axis cs:15,10550.12)
--(axis cs:16,10405.55)
--(axis cs:17,10633.006)
--(axis cs:18,10594.0566666667)
--(axis cs:19,10877.4733333333)
--(axis cs:19,11089.2066666667)
--(axis cs:19,11089.2066666667)
--(axis cs:18,10753.8333333333)
--(axis cs:17,10837.2333333333)
--(axis cs:16,10588.44)
--(axis cs:15,10664.28)
--(axis cs:14,10479.9266666667)
--(axis cs:13,10227.6766666667)
--(axis cs:12,10066.06)
--(axis cs:11,9883.68)
--(axis cs:10,8960.32)
--(axis cs:9,8488.12)
--(axis cs:8,9165.13883333333)
--(axis cs:7,8697.39)
--(axis cs:6,8213.6)
--(axis cs:5,7848.75)
--(axis cs:4,7470.8)
--(axis cs:3,6442.81333333333)
--(axis cs:2,5425.40666666667)
--(axis cs:1,3744.79)
--(axis cs:0,1074.2595)
--cycle;

\path [draw=color4, fill=color4, opacity=0.2]
(axis cs:0,936.358)
--(axis cs:0,515.152)
--(axis cs:1,3755.13666666667)
--(axis cs:2,5987.81683333333)
--(axis cs:3,6830.8205)
--(axis cs:4,7209.558)
--(axis cs:5,7976.52666666667)
--(axis cs:6,8381.52716666667)
--(axis cs:7,8697.12383333333)
--(axis cs:8,8989.2155)
--(axis cs:9,9311.83666666666)
--(axis cs:10,9475.75766666667)
--(axis cs:11,9588.5705)
--(axis cs:12,9754.69966666667)
--(axis cs:13,10004.9756666667)
--(axis cs:14,9926.89)
--(axis cs:15,10172.74)
--(axis cs:16,10386.326)
--(axis cs:17,10593.757)
--(axis cs:18,10572.094)
--(axis cs:19,10584.2968333333)
--(axis cs:19,10940.6866666667)
--(axis cs:19,10940.6866666667)
--(axis cs:18,10908.1166666667)
--(axis cs:17,10918.41)
--(axis cs:16,10728.4133333333)
--(axis cs:15,10522.09)
--(axis cs:14,10378.4248333333)
--(axis cs:13,10293.98)
--(axis cs:12,10026.4866666667)
--(axis cs:11,9899.42316666667)
--(axis cs:10,9690.459)
--(axis cs:9,9621.82116666666)
--(axis cs:8,9113.91566666667)
--(axis cs:7,8949.27)
--(axis cs:6,8739.568)
--(axis cs:5,8323.60333333333)
--(axis cs:4,7504.56866666667)
--(axis cs:3,7174.59666666667)
--(axis cs:2,6248.8395)
--(axis cs:1,4218.61333333333)
--(axis cs:0,936.358)
--cycle;

\addplot [very thick, color0]
table {%
0 1862.04167915857
1 4194.44281368422
2 5779.9277072087
3 6465.3924259729
4 6625.91390587804
5 7411.11774352419
6 7884.77590002357
7 8394.7150793534
8 8748.10200715901
9 8881.73064381354
10 9110.12243784273
11 9205.1642039255
12 9398.46983254388
13 9620.15020857221
14 9689.59018402517
15 10131.0859628861
16 10141.336005873
17 10215.6979153396
18 10514.9551287601
19 10648.9799747951
};
\addplot [very thick, color1]
table {%
0 1110.94322333333
1 3402.08271
2 4957.35197
3 6233.09985333333
4 7186.17637666667
5 7805.25347666667
6 8165.62652666667
7 8582.43207666667
8 8863.18753333333
9 9166.28245333333
10 9447.63307
11 9633.99456333333
12 9773.18093
13 9909.54794666667
14 10089.8394766667
15 10255.41584
16 10362.75679
17 10451.8734066667
18 9946.12245666667
19 10579.299
};
\addplot [very thick, color2]
table {%
0 1063.69790333333
1 3866.12242
2 5329.59514666667
3 6391.5668
4 7063.69025666667
5 7374.57062333333
6 8071.62813
7 8277.93757333333
8 8413.75229
9 8776.66609666667
10 8598.17258666667
11 9471.97196
12 9614.49986333333
13 9769.79539
14 9864.62841
15 10015.0522266667
16 10147.8400833333
17 10327.0765633333
18 10455.6062233333
19 10611.8754133333
};
\addplot [very thick, color3]
table {%
0 1051.30391
1 3710.86604333333
2 5394.43547
3 6290.99234666667
4 7378.56349
5 7791.1472
6 7912.85286333333
7 8532.93004
8 9055.11401333334
9 8028.74261
10 8785.44754
11 9774.49565666667
12 9930.27951333333
13 10089.9007033333
14 10390.36528
15 10607.16818
16 10493.84295
17 10734.11328
18 10673.0679666667
19 10982.1103466667
};
\addplot [very thick, color4]
table {%
0 730.014163333333
1 3991.90040666667
2 6122.69236
3 7005.55848333333
4 7355.35607
5 8148.37062
6 8560.24336
7 8818.95273333333
8 9054.24293333333
9 9466.87802
10 9587.20858666667
11 9741.73666333333
12 9894.45505
13 10144.3429666667
14 10157.3937133333
15 10343.9585266667
16 10560.8314866667
17 10749.63176
18 10735.3689766667
19 10763.87363
};
\end{axis}

\end{tikzpicture} 
 \end{subfigure}
 \vspace{-1.5em}
 \caption{
  Performance of the proposed value function training mechanisms against baselines. Results are averaged over 8 random seeds and shaded regions correspond to the 95\% confidence interval among seeds.
 }
\label{fig:main_experimental_results}
\vspace{-1.5em}
\end{figure}



\subsection{Visualizing the Value-aware Model Error}
\label{s5s2:vaml}
Given the excellent performance of robust regularization, we now verify its effectiveness in controlling the value-aware model error. On Walker, we compare it with a variant: computing the perturbation with uniform random noise instead of adversarially choosing the perturbation. 


As we can see from the first figure in Figure~\ref{fig:vaml}, robust regularization has a much smaller value-aware model error than all other methods. Adding uniformly random noise can somewhat reduce the value-aware model error compared with a single deterministic model (without any noise added). However, it is still much less effective than robust regularization, which computes the error adversarially.
In the second figure of Figure~\ref{fig:vaml}, we see that with uniform random noise, it achieves a slightly better performance than a single deterministic model but is still far worse than robust regularization. The results further verify that by adversarially choosing the noise, robust regularization is extremely effective at controlling the value-aware model error, resulting in great empirical performance. In Appendix~\ref{a3:add_vaml}, we provide additional experimental results of value-aware model error in the rest of the four environments.

\subsection{Robust Regularization vs. Spectral Normalization}
\label{s5s3:rr_sn}
In Table~\ref{tab:time_compare}, we already see that robust regularization achieves a much better time efficiency than spectral normalization. Now we analyze their performance difference through the lens of value-aware model error. To reduce the value-aware model error, we only need to control the local Lipschitz constant of the value network over the model-uncertain region, and thus controlling the global Lipschitz constant of the value function with spectral normalization is not necessary. To see this, in the last two figures of Figure~\ref{fig:vaml}
, we visualize the value-aware model error and performance of the algorithm on Walker with varying spectral radius $\beta$, defined in Section~\ref{s4s1_spectral}. 

\begin{figure}[!t]
\vspace{0em}
\centering
 \begin{subfigure}[t]{0.2\textwidth}
  \centering
\begin{tikzpicture}[scale=0.4]

\definecolor{color0}{rgb}{0.297,0.445,0.9}
\definecolor{color1}{rgb}{0.172549019607843,0.627450980392157,0.172549019607843}
\definecolor{color2}{rgb}{0.83921568627451,0.152941176470588,0.156862745098039}
\definecolor{color3}{rgb}{1,0.55,0}

\begin{axis}[
legend cell align={left},
legend style={fill opacity=0.3, draw opacity=1, text opacity=1, draw=white!80!black},
tick align=outside,
tick pos=left,
x grid style={white!69.0196078431373!black},
xlabel={Number of Interactions},
xlabel style={font=\Large},
title style={font=\Large},
ylabel style={font=\large},
xmajorgrids,
xmin=-1, xmax=19.95,
xtick style={color=black},
xtick={-1,4,9,14,19},
xticklabels={0,5e+4,10e+4,15e+4,20e+4},
y grid style={white!69.0196078431373!black},
ylabel={Value-aware Model Error (Log Scale)},
ymajorgrids,
ymin=1.59734381533071, ymax=7.54167245609963,
ytick style={color=black}
]
\path [draw=color0, fill=color0, opacity=0.2]
(axis cs:0,2.81748673645189)
--(axis cs:0,2.58981808152078)
--(axis cs:1,2.77088853146468)
--(axis cs:2,2.34789008023446)
--(axis cs:3,2.24418521910821)
--(axis cs:4,2.10581077450115)
--(axis cs:5,2.13030071402482)
--(axis cs:6,2.25703354868507)
--(axis cs:7,2.35203667152322)
--(axis cs:8,2.31863169711325)
--(axis cs:9,2.29166799506554)
--(axis cs:10,2.2959267513767)
--(axis cs:11,2.31037023692807)
--(axis cs:12,2.30330285414503)
--(axis cs:13,2.31784970542001)
--(axis cs:14,2.31832140452386)
--(axis cs:15,2.31961850056719)
--(axis cs:16,2.25968878286545)
--(axis cs:17,2.22943049753514)
--(axis cs:18,2.25900549158108)
--(axis cs:19,2.2597043401196)
--(axis cs:19,2.3474993144912)
--(axis cs:19,2.3474993144912)
--(axis cs:18,2.36510112173151)
--(axis cs:17,2.31314546432508)
--(axis cs:16,2.34106927205517)
--(axis cs:15,2.39774953928551)
--(axis cs:14,2.40059113529525)
--(axis cs:13,2.40844476414588)
--(axis cs:12,2.39977289545573)
--(axis cs:11,2.4186511293242)
--(axis cs:10,2.41186390903306)
--(axis cs:9,2.39741062405083)
--(axis cs:8,2.41857320348624)
--(axis cs:7,2.42941391495885)
--(axis cs:6,2.36192064013225)
--(axis cs:5,2.29883414732637)
--(axis cs:4,2.25871438967362)
--(axis cs:3,2.3577590374381)
--(axis cs:2,2.40548503457925)
--(axis cs:1,2.87267847131674)
--(axis cs:0,2.81748673645189)
--cycle;

\path [draw=color1, fill=color1, opacity=0.2]
(axis cs:0,6.62740706917695)
--(axis cs:0,4.67009459377407)
--(axis cs:1,4.82070434974252)
--(axis cs:2,3.93653225602816)
--(axis cs:3,3.14864795213264)
--(axis cs:4,3.06248521524954)
--(axis cs:5,3.06282350132321)
--(axis cs:6,3.15010435080495)
--(axis cs:7,3.04644429383389)
--(axis cs:8,2.9836974092087)
--(axis cs:9,2.80670904742246)
--(axis cs:10,2.70335512069474)
--(axis cs:11,2.76499193263796)
--(axis cs:12,2.65173285282749)
--(axis cs:13,2.69438455382327)
--(axis cs:14,2.75449231766022)
--(axis cs:15,2.62936637911357)
--(axis cs:16,2.69708101737461)
--(axis cs:17,2.89044553088297)
--(axis cs:18,2.79547002698656)
--(axis cs:19,2.84613328454671)
--(axis cs:19,2.98717300886176)
--(axis cs:19,2.98717300886176)
--(axis cs:18,3.0049014966016)
--(axis cs:17,3.52050982217255)
--(axis cs:16,2.87774400981042)
--(axis cs:15,2.83240602097136)
--(axis cs:14,2.98816450216799)
--(axis cs:13,2.90185999327417)
--(axis cs:12,2.80918414809664)
--(axis cs:11,2.99317208723361)
--(axis cs:10,2.96558863413369)
--(axis cs:9,3.17161911544955)
--(axis cs:8,3.4314528995204)
--(axis cs:7,3.8211248584114)
--(axis cs:6,4.11810268986344)
--(axis cs:5,3.99439315644537)
--(axis cs:4,4.13718411321813)
--(axis cs:3,4.45941972498511)
--(axis cs:2,5.83263805770241)
--(axis cs:1,7.27147569970105)
--(axis cs:0,6.62740706917695)
--cycle;

\path [draw=color2, fill=color2, opacity=0.2]
(axis cs:0,2.63245497941473)
--(axis cs:0,2.2769104661099)
--(axis cs:1,2.17123956237894)
--(axis cs:2,1.897437005002)
--(axis cs:3,1.8675405717293)
--(axis cs:4,1.94887623188559)
--(axis cs:5,1.97865354480255)
--(axis cs:6,2.030474232787)
--(axis cs:7,2.03452797279905)
--(axis cs:8,2.05099877686553)
--(axis cs:9,2.0427985076241)
--(axis cs:10,2.06093000140447)
--(axis cs:11,2.08196632010185)
--(axis cs:12,2.15730917614581)
--(axis cs:13,2.10785712793548)
--(axis cs:14,2.11809528695482)
--(axis cs:15,2.08581143158639)
--(axis cs:16,2.17044177065505)
--(axis cs:17,2.11578261097413)
--(axis cs:18,2.11291862098602)
--(axis cs:19,2.12666487721785)
--(axis cs:19,2.25472379769854)
--(axis cs:19,2.25472379769854)
--(axis cs:18,2.24853386419106)
--(axis cs:17,2.2503547779356)
--(axis cs:16,2.36090107900631)
--(axis cs:15,2.21613538432232)
--(axis cs:14,2.24778294386766)
--(axis cs:13,2.22480533021696)
--(axis cs:12,2.40471115486486)
--(axis cs:11,2.19429496461129)
--(axis cs:10,2.17213985978511)
--(axis cs:9,2.17354547357388)
--(axis cs:8,2.17914435356784)
--(axis cs:7,2.16162191509619)
--(axis cs:6,2.16462439816407)
--(axis cs:5,2.14633820392795)
--(axis cs:4,2.13943756153359)
--(axis cs:3,2.09343936650287)
--(axis cs:2,2.12248350110831)
--(axis cs:1,2.42450820349714)
--(axis cs:0,2.63245497941473)
--cycle;

\path [draw=color3, fill=color3, opacity=0.2]
(axis cs:0,3.86039833617922)
--(axis cs:0,3.41131782282755)
--(axis cs:1,3.126153336523)
--(axis cs:2,2.70385210918574)
--(axis cs:3,2.3367298529138)
--(axis cs:4,2.46671666698347)
--(axis cs:5,2.44614383308253)
--(axis cs:6,2.38801685656035)
--(axis cs:7,2.48756448172892)
--(axis cs:8,2.63637735756818)
--(axis cs:9,2.53567967907599)
--(axis cs:10,2.62691658983168)
--(axis cs:11,2.65533872318829)
--(axis cs:12,2.68664013375988)
--(axis cs:13,2.69613870614507)
--(axis cs:14,2.79987263929762)
--(axis cs:15,2.78621722759013)
--(axis cs:16,2.75340960957976)
--(axis cs:17,2.75189806937345)
--(axis cs:18,2.76595757333708)
--(axis cs:19,2.78849595778054)
--(axis cs:19,2.95864096972912)
--(axis cs:19,2.95864096972912)
--(axis cs:18,2.95104233515067)
--(axis cs:17,2.91787231300943)
--(axis cs:16,2.91112327195706)
--(axis cs:15,2.96338011269626)
--(axis cs:14,3.1658866644223)
--(axis cs:13,2.86867391337958)
--(axis cs:12,2.87265784051179)
--(axis cs:11,2.85722881456646)
--(axis cs:10,2.95213643173826)
--(axis cs:9,2.73746209698127)
--(axis cs:8,2.9181274076699)
--(axis cs:7,2.66145828553599)
--(axis cs:6,2.5423799539001)
--(axis cs:5,2.81907463094768)
--(axis cs:4,3.0138981856918)
--(axis cs:3,2.50494425255066)
--(axis cs:2,2.91182263094816)
--(axis cs:1,3.44209616802677)
--(axis cs:0,3.86039833617922)
--cycle;

\addplot [very thick, color0]
table {%
0 2.69937296829589
1 2.81989949983648
2 2.37676332877011
3 2.30068755180654
4 2.18137444693788
5 2.21486811083225
6 2.30823031137546
7 2.39087521145738
8 2.37068102063292
9 2.34392330626982
10 2.35270353904833
11 2.36362024950274
12 2.35224236952605
13 2.36268006962767
14 2.36157318999626
15 2.35976781839652
16 2.29970023692606
17 2.27110261970543
18 2.30855899823872
19 2.30334713576136
};
\addlegendentry{Probabilistic Ensemble}
\addplot [very thick, color1]
table {%
0 5.64738344260263
1 5.99348261534713
2 4.84287802695371
3 3.77645855322148
4 3.58216487843923
5 3.50367790673845
6 3.60299893370114
7 3.40299752572342
8 3.18910837560332
9 2.98271859760287
10 2.83323673519043
11 2.87977595025641
12 2.73404710942605
13 2.79828444355696
14 2.87025725201173
15 2.73056094093531
16 2.78754540424996
17 3.17789198307262
18 2.89815132353579
19 2.91580274245655
};
\addlegendentry{Single Deterministic}
\addplot [very thick, color2]
table {%
0 2.45228823026828
1 2.29746853845857
2 2.0095871042937
3 1.98193236792124
4 2.04634174540013
5 2.06043681961027
6 2.09498602490532
7 2.09848154135988
8 2.11402257410674
9 2.1059516962265
10 2.11687446912136
11 2.13595751774374
12 2.27574270500665
13 2.16635255603869
14 2.17973321043937
15 2.15272167971889
16 2.26424157210768
17 2.18461517570948
18 2.18054925044898
19 2.1916252014697
};
\addlegendentry{Robust Regularization}
\addplot [very thick, color3]
table {%
0 3.63933513531246
1 3.27640734786716
2 2.80835300698857
3 2.42172906091492
4 2.73169420331845
5 2.61976296751214
6 2.4641927189362
7 2.57253812967987
8 2.77616927529497
9 2.63830683158037
10 2.78914351236511
11 2.75578486789123
12 2.77677205137963
13 2.78190705572197
14 2.96933051544347
15 2.87328992057615
16 2.83200201839509
17 2.8339237896532
18 2.85762845223168
19 2.87222165343688
};
\addlegendentry{Uniform Random Noise}
\end{axis}

\end{tikzpicture}
  \vspace{-1.5em}
  \centering
    \label{sfig:vaml_a}
 \end{subfigure}
 \hfill
 \begin{subfigure}[t]{0.23\textwidth}
  \centering
\begin{tikzpicture}[scale=0.4]

\definecolor{color0}{rgb}{0.297,0.445,0.9}
\definecolor{color1}{rgb}{0.172549019607843,0.627450980392157,0.172549019607843}
\definecolor{color2}{rgb}{0.83921568627451,0.152941176470588,0.156862745098039}
\definecolor{color3}{rgb}{1,0.55,0}

\begin{axis}[
legend cell align={left},
legend style={
  fill opacity=0.3,
  draw opacity=1,
  text opacity=1,
  at={(0.97,0.03)},
  anchor=south east,
  draw=white!80!black
},
xlabel style={font=\Large},
title style={font=\Large},
ylabel style={font=\large},
tick align=outside,
tick pos=left,
x grid style={white!69.0196078431373!black},
xlabel={Number of Interactions},
xmajorgrids,
xmin=-1, xmax=19.95,
xtick style={color=black},
xtick={-1,4,9,14,19},
xticklabels={0,5e+4,10e+4,15e+4,20e+4},
y grid style={white!69.0196078431373!black},
ylabel={Performance},
ymajorgrids,
ymin=-105.934908333333, ymax=4869.176075,
ytick style={color=black}
]
\path [draw=color0, fill=color0, opacity=0.2]
(axis cs:0,330.313666666667)
--(axis cs:0,308.519)
--(axis cs:1,486.113)
--(axis cs:2,567.194666666667)
--(axis cs:3,891.686833333334)
--(axis cs:4,1452.66033333333)
--(axis cs:5,2078.04266666667)
--(axis cs:6,2526.144)
--(axis cs:7,2812.56333333333)
--(axis cs:8,3204.74383333333)
--(axis cs:9,3358.9015)
--(axis cs:10,3459.67283333333)
--(axis cs:11,3571.882)
--(axis cs:12,3466.83966666667)
--(axis cs:13,3902.66833333333)
--(axis cs:14,3784.20366666667)
--(axis cs:15,4101.16483333333)
--(axis cs:16,3725.154)
--(axis cs:17,4050.60466666667)
--(axis cs:18,4128.38666666667)
--(axis cs:19,4192.92866666667)
--(axis cs:19,4563.3125)
--(axis cs:19,4563.3125)
--(axis cs:18,4415.99566666667)
--(axis cs:17,4368.012)
--(axis cs:16,4132.71183333333)
--(axis cs:15,4359.17133333333)
--(axis cs:14,4032.947)
--(axis cs:13,4186.754)
--(axis cs:12,3947.298)
--(axis cs:11,4007.29016666667)
--(axis cs:10,3896.56383333333)
--(axis cs:9,3807.42133333333)
--(axis cs:8,3697.33433333333)
--(axis cs:7,3299.13266666667)
--(axis cs:6,3088.88783333333)
--(axis cs:5,2689.27483333333)
--(axis cs:4,2055.912)
--(axis cs:3,1413.24283333333)
--(axis cs:2,744.135833333333)
--(axis cs:1,566.860666666667)
--(axis cs:0,330.313666666667)
--cycle;

\path [draw=color1, fill=color1, opacity=0.2]
(axis cs:0,185.275166666667)
--(axis cs:0,120.2065)
--(axis cs:1,287.229166666667)
--(axis cs:2,330.790166666667)
--(axis cs:3,487.070666666667)
--(axis cs:4,571.429)
--(axis cs:5,801.521166666667)
--(axis cs:6,993.362)
--(axis cs:7,1033.68666666667)
--(axis cs:8,1273.1565)
--(axis cs:9,1509.95516666667)
--(axis cs:10,1912.76316666667)
--(axis cs:11,1836.12366666667)
--(axis cs:12,2156.17383333333)
--(axis cs:13,2222.8195)
--(axis cs:14,2093.95566666667)
--(axis cs:15,2702.2125)
--(axis cs:16,2714.09666666667)
--(axis cs:17,2491.45083333333)
--(axis cs:18,2768.67733333333)
--(axis cs:19,2829.48316666667)
--(axis cs:19,3274.4185)
--(axis cs:19,3274.4185)
--(axis cs:18,3301.39433333333)
--(axis cs:17,2966.42466666667)
--(axis cs:16,3231.59866666667)
--(axis cs:15,3261.28366666667)
--(axis cs:14,2641.23)
--(axis cs:13,2982.62916666667)
--(axis cs:12,2858.07966666667)
--(axis cs:11,2515.751)
--(axis cs:10,2594.2855)
--(axis cs:9,2073.487)
--(axis cs:8,1958.94916666667)
--(axis cs:7,1725.57766666667)
--(axis cs:6,1542.38616666667)
--(axis cs:5,1202.8945)
--(axis cs:4,820.054333333333)
--(axis cs:3,715.321833333333)
--(axis cs:2,494.7215)
--(axis cs:1,342.9115)
--(axis cs:0,185.275166666667)
--cycle;

\path [draw=color2, fill=color2, opacity=0.2]
(axis cs:0,401.616666666667)
--(axis cs:0,362.566666666667)
--(axis cs:1,369.346666666667)
--(axis cs:2,522.977666666667)
--(axis cs:3,586.78)
--(axis cs:4,1038.70666666667)
--(axis cs:5,2069.26666666667)
--(axis cs:6,2649.84)
--(axis cs:7,3006.76333333333)
--(axis cs:8,3816.15333333333)
--(axis cs:9,3928.47733333333)
--(axis cs:10,4042.03333333333)
--(axis cs:11,3804.09333333333)
--(axis cs:12,3893.95333333333)
--(axis cs:13,4231.86)
--(axis cs:14,4322.77666666667)
--(axis cs:15,4484.96666666667)
--(axis cs:16,4314.79333333333)
--(axis cs:17,4478.46333333333)
--(axis cs:18,4584.64)
--(axis cs:19,4332.02333333333)
--(axis cs:19,4566.64333333333)
--(axis cs:19,4566.64333333333)
--(axis cs:18,4635.88666666667)
--(axis cs:17,4643.03466666667)
--(axis cs:16,4520.7)
--(axis cs:15,4592.56666666667)
--(axis cs:14,4498.62333333333)
--(axis cs:13,4320.16)
--(axis cs:12,4182.19666666667)
--(axis cs:11,4182.37333333333)
--(axis cs:10,4093.02666666667)
--(axis cs:9,4158.05)
--(axis cs:8,3958.44766666667)
--(axis cs:7,3249.63666666667)
--(axis cs:6,2813.09333333333)
--(axis cs:5,2409.81333333333)
--(axis cs:4,1319.64666666667)
--(axis cs:3,666.783333333333)
--(axis cs:2,644.743333333333)
--(axis cs:1,404.08)
--(axis cs:0,401.616666666667)
--cycle;

\path [draw=color3, fill=color3, opacity=0.2]
(axis cs:0,241.072166666667)
--(axis cs:0,176.050333333333)
--(axis cs:1,282.941666666667)
--(axis cs:2,379.841166666667)
--(axis cs:3,462.947333333333)
--(axis cs:4,710.946666666667)
--(axis cs:5,1112.81133333333)
--(axis cs:6,1552.8455)
--(axis cs:7,1701.83333333333)
--(axis cs:8,2188.193)
--(axis cs:9,2623.19666666667)
--(axis cs:10,2597.0515)
--(axis cs:11,2381.32683333333)
--(axis cs:12,3040.78966666667)
--(axis cs:13,2991.73666666667)
--(axis cs:14,2764.0345)
--(axis cs:15,3186.44366666667)
--(axis cs:16,3606.86433333333)
--(axis cs:17,3484.92783333333)
--(axis cs:18,3591.418)
--(axis cs:19,3366.87933333333)
--(axis cs:19,3656.6645)
--(axis cs:19,3656.6645)
--(axis cs:18,3837.78066666667)
--(axis cs:17,3813.113)
--(axis cs:16,3888.41916666667)
--(axis cs:15,3469.54083333333)
--(axis cs:14,3226.97966666667)
--(axis cs:13,3297.39116666667)
--(axis cs:12,3219.94566666667)
--(axis cs:11,2739.00833333333)
--(axis cs:10,2927.14016666667)
--(axis cs:9,2964.1985)
--(axis cs:8,2599.36433333333)
--(axis cs:7,2199.554)
--(axis cs:6,2085.67433333333)
--(axis cs:5,1564.90266666667)
--(axis cs:4,977.358833333333)
--(axis cs:3,537.050666666667)
--(axis cs:2,403.330333333333)
--(axis cs:1,358.379666666667)
--(axis cs:0,241.072166666667)
--cycle;

\addplot [very thick, color0]
table {%
0 319.289136666667
1 524.559003333333
2 651.149666666667
3 1128.72768666667
4 1747.73023
5 2387.74693666667
6 2802.5787
7 3057.55263333333
8 3456.76774
9 3596.18605666667
10 3687.45848333333
11 3801.29329333333
12 3714.95464
13 4051.92573333333
14 3903.80539
15 4237.13866666667
16 3924.60339666667
17 4212.91326666667
18 4271.95926333333
19 4386.84106333333
};
\addlegendentry{Probabilistic Ensemble}
\addplot [very thick, color1]
table {%
0 152.545256666667
1 315.138986666667
2 409.816366666667
3 593.064793333333
4 696.480306666667
5 990.236846666667
6 1266.18504666667
7 1371.52563
8 1613.94372333333
9 1788.66574666667
10 2251.58896666667
11 2159.58657666667
12 2510.45103333333
13 2594.08677333333
14 2355.01722333333
15 2986.36354
16 2986.52669666667
17 2729.29981
18 3036.53782666667
19 3048.08345
};
\addlegendentry{Single Deterministic}
\addplot [very thick, color2]
table {%
0 382.803366666667
1 387.30689
2 579.818793333333
3 624.84374
4 1186.6648
5 2238.76073333333
6 2732.19804666667
7 3128.96746333333
8 3885.23556666667
9 4041.05514666667
10 4068.48513333333
11 4005.56404333333
12 4043.22711666667
13 4277.36201666667
14 4412.45053
15 4538.59144
16 4418.70198666667
17 4557.72385333333
18 4610.04280666667
19 4446.50850333333
};
\addlegendentry{Robust Regularization}
\addplot [very thick, color3]
table {%
0 207.975973333333
1 320.457633333333
2 390.66864
3 498.737726666667
4 842.141313333333
5 1339.88405333333
6 1818.26933
7 1946.54680666667
8 2394.46314666667
9 2787.16840333333
10 2761.15919333333
11 2568.43921
12 3129.51337333333
13 3147.67252
14 3000.12524666667
15 3328.24074666667
16 3745.70707
17 3651.20899333333
18 3712.80809333333
19 3511.62203666667
};
\addlegendentry{Uniform Random Noise}
\end{axis}

\end{tikzpicture}
    \vspace{-1.5em}
  \centering
    \label{sfig:vaml_b}
 \end{subfigure}
 \hfill
 \begin{subfigure}[t]{0.22\textwidth}
  \centering
  \input{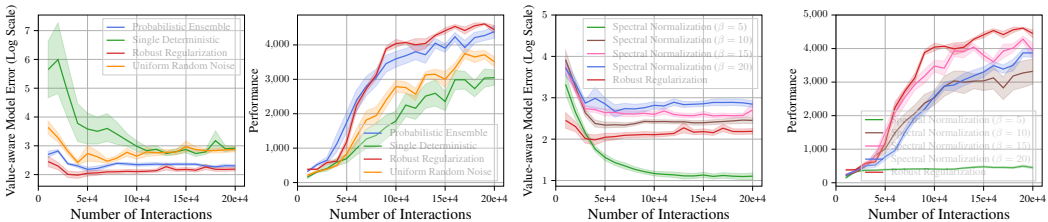} 
  \vspace{-1.5em}
  \centering
    \label{sfig:vaml_c}
 \end{subfigure}
 \hfill
 \begin{subfigure}[t]{0.25\textwidth}
  \centering
\begin{tikzpicture}[scale=0.4]

\definecolor{crimson2143839}{RGB}{214,38,39}
\definecolor{darkgray176}{RGB}{176,176,176}
\definecolor{forestgreen4316043}{RGB}{43,160,43}
\definecolor{hotpink}{RGB}{255,105,180}
\definecolor{lightgray204}{RGB}{204,204,204}
\definecolor{royalblue75113229}{RGB}{75,113,229}
\definecolor{sienna1398575}{RGB}{139,85,75}

\begin{axis}[
legend cell align={left},
legend style={
  fill opacity=0.2,
  draw opacity=1,
  text opacity=1,
  at={(0.97,0.03)},
  anchor=south east,
  draw=lightgray204
},
xlabel style={font=\Large},
title style={font=\Large},
ylabel style={font=\large},
tick align=outside,
tick pos=left,
x grid style={darkgray176},
xlabel={Number of Interactions},
xmajorgrids,
xmin=-1, xmax=19.95,
xtick style={color=black},
xtick={-1,4,9,14,19},
xticklabels={0,5e+4,10e+4,15e+4,20e+4},
y grid style={darkgray176},
ylabel={Performance},
ymajorgrids,
ymin=-102.555991666667, ymax=5000,
ytick style={color=black}
]
\path [draw=forestgreen4316043, fill=forestgreen4316043, opacity=0.2]
(axis cs:0,176.769333333333)
--(axis cs:0,123.290166666667)
--(axis cs:1,315.919333333333)
--(axis cs:2,348.6655)
--(axis cs:3,358.769)
--(axis cs:4,372.639166666667)
--(axis cs:5,382.777666666667)
--(axis cs:6,392.835833333333)
--(axis cs:7,399.491)
--(axis cs:8,403.269833333333)
--(axis cs:9,402.145666666667)
--(axis cs:10,388.912333333333)
--(axis cs:11,393.537666666667)
--(axis cs:12,418.396)
--(axis cs:13,437.575666666667)
--(axis cs:14,448.140833333333)
--(axis cs:15,435.5725)
--(axis cs:16,433.483333333333)
--(axis cs:17,423.612666666667)
--(axis cs:18,461.257333333333)
--(axis cs:19,417.181166666667)
--(axis cs:19,486.458)
--(axis cs:19,486.458)
--(axis cs:18,536.757833333333)
--(axis cs:17,474.674833333333)
--(axis cs:16,480.974833333333)
--(axis cs:15,474.410166666667)
--(axis cs:14,506.481833333333)
--(axis cs:13,486.325833333333)
--(axis cs:12,460.458666666667)
--(axis cs:11,419.9075)
--(axis cs:10,414.4165)
--(axis cs:9,416.220833333333)
--(axis cs:8,421.731333333333)
--(axis cs:7,424.170166666667)
--(axis cs:6,420.191333333333)
--(axis cs:5,411.751333333333)
--(axis cs:4,412.068166666667)
--(axis cs:3,390.309)
--(axis cs:2,393.885833333333)
--(axis cs:1,363.943666666667)
--(axis cs:0,176.769333333333)
--cycle;

\path [draw=sienna1398575, fill=sienna1398575, opacity=0.2]
(axis cs:0,234.432833333333)
--(axis cs:0,198.812666666667)
--(axis cs:1,301.178666666667)
--(axis cs:2,451.425333333333)
--(axis cs:3,698.9135)
--(axis cs:4,892.466833333333)
--(axis cs:5,1256.65916666667)
--(axis cs:6,1581.14966666667)
--(axis cs:7,1803.91566666667)
--(axis cs:8,2124.80166666667)
--(axis cs:9,2314.48866666667)
--(axis cs:10,2576.19166666667)
--(axis cs:11,2729.21216666667)
--(axis cs:12,2655.679)
--(axis cs:13,2804.50833333333)
--(axis cs:14,2748.78583333333)
--(axis cs:15,2792.5725)
--(axis cs:16,2476.89366666667)
--(axis cs:17,2761.22416666667)
--(axis cs:18,2842.36583333333)
--(axis cs:19,2937.40483333333)
--(axis cs:19,3689.68816666667)
--(axis cs:19,3689.68816666667)
--(axis cs:18,3630.75566666667)
--(axis cs:17,3555.687)
--(axis cs:16,3153.1475)
--(axis cs:15,3449.14933333333)
--(axis cs:14,3289.31433333333)
--(axis cs:13,3242.70566666667)
--(axis cs:12,3326.6415)
--(axis cs:11,3329.464)
--(axis cs:10,3174.26133333333)
--(axis cs:9,2754.3735)
--(axis cs:8,2635.53416666667)
--(axis cs:7,2312.3405)
--(axis cs:6,2058.64333333333)
--(axis cs:5,1757.41066666667)
--(axis cs:4,1128.27233333333)
--(axis cs:3,931.6865)
--(axis cs:2,539.085333333333)
--(axis cs:1,325.4535)
--(axis cs:0,234.432833333333)
--cycle;

\path [draw=hotpink, fill=hotpink, opacity=0.2]
(axis cs:0,268.7475)
--(axis cs:0,225.6125)
--(axis cs:1,333.845333333333)
--(axis cs:2,477.740833333333)
--(axis cs:3,692.351333333334)
--(axis cs:4,1327.28333333333)
--(axis cs:5,1835.27066666667)
--(axis cs:6,2310.25883333333)
--(axis cs:7,2752.3365)
--(axis cs:8,2964.9835)
--(axis cs:9,3305.51383333333)
--(axis cs:10,3194.86183333333)
--(axis cs:11,3760.83083333333)
--(axis cs:12,3726.1)
--(axis cs:13,3954.12316666667)
--(axis cs:14,3662.68083333333)
--(axis cs:15,3502.47433333333)
--(axis cs:16,3714.24183333333)
--(axis cs:17,3893.11116666667)
--(axis cs:18,4194.535)
--(axis cs:19,3680.493)
--(axis cs:19,4161.37283333333)
--(axis cs:19,4161.37283333333)
--(axis cs:18,4373.11083333333)
--(axis cs:17,4141.79566666667)
--(axis cs:16,4128.16016666667)
--(axis cs:15,3703.96666666667)
--(axis cs:14,3927.74283333333)
--(axis cs:13,4144.17333333333)
--(axis cs:12,4113.62283333333)
--(axis cs:11,4106.75416666667)
--(axis cs:10,3630.65283333333)
--(axis cs:9,3648.32516666667)
--(axis cs:8,3377.53766666667)
--(axis cs:7,3089.16466666667)
--(axis cs:6,2799.35633333333)
--(axis cs:5,2323.90133333333)
--(axis cs:4,1647.57166666667)
--(axis cs:3,889.64)
--(axis cs:2,580.177)
--(axis cs:1,356.657)
--(axis cs:0,268.7475)
--cycle;

\path [draw=royalblue75113229, fill=royalblue75113229, opacity=0.2]
(axis cs:0,251.1475)
--(axis cs:0,200.658)
--(axis cs:1,324.819666666667)
--(axis cs:2,421.115833333333)
--(axis cs:3,458.034666666667)
--(axis cs:4,671.5715)
--(axis cs:5,878.876666666667)
--(axis cs:6,1372.52933333333)
--(axis cs:7,1803.18083333333)
--(axis cs:8,2024.95816666667)
--(axis cs:9,2407.21633333333)
--(axis cs:10,2758.76216666667)
--(axis cs:11,2745.05783333333)
--(axis cs:12,2891.222)
--(axis cs:13,3012.09116666667)
--(axis cs:14,3149.942)
--(axis cs:15,3370.84283333333)
--(axis cs:16,3265.19133333333)
--(axis cs:17,3331.68983333333)
--(axis cs:18,3753.66466666667)
--(axis cs:19,3738.98633333333)
--(axis cs:19,4007.23433333333)
--(axis cs:19,4007.23433333333)
--(axis cs:18,3986.94316666667)
--(axis cs:17,3676.893)
--(axis cs:16,3506.70933333333)
--(axis cs:15,3620.95966666667)
--(axis cs:14,3454.38333333333)
--(axis cs:13,3359.29766666667)
--(axis cs:12,3257.4005)
--(axis cs:11,3065.26516666667)
--(axis cs:10,2992.478)
--(axis cs:9,2710.35)
--(axis cs:8,2368.42933333333)
--(axis cs:7,2089.7555)
--(axis cs:6,1632.04133333333)
--(axis cs:5,1048.91016666667)
--(axis cs:4,885.511333333333)
--(axis cs:3,595.859833333333)
--(axis cs:2,596.428166666667)
--(axis cs:1,403.835)
--(axis cs:0,251.1475)
--cycle;

\path [draw=crimson2143839, fill=crimson2143839, opacity=0.2]
(axis cs:0,401.433333333333)
--(axis cs:0,362.383333333333)
--(axis cs:1,370.483333333333)
--(axis cs:2,527.486666666667)
--(axis cs:3,585.193333333333)
--(axis cs:4,1053.15333333333)
--(axis cs:5,2069.12)
--(axis cs:6,2643.16)
--(axis cs:7,2988.15783333333)
--(axis cs:8,3811.06466666667)
--(axis cs:9,3929.28333333333)
--(axis cs:10,4042.03333333333)
--(axis cs:11,3814.43)
--(axis cs:12,3893.95333333333)
--(axis cs:13,4231.86)
--(axis cs:14,4322.77666666667)
--(axis cs:15,4484.86333333333)
--(axis cs:16,4313.36666666667)
--(axis cs:17,4472.97666666667)
--(axis cs:18,4584.15333333333)
--(axis cs:19,4324.94666666667)
--(axis cs:19,4557.43333333333)
--(axis cs:19,4557.43333333333)
--(axis cs:18,4635.4)
--(axis cs:17,4640.21333333333)
--(axis cs:16,4517.84666666667)
--(axis cs:15,4592.36)
--(axis cs:14,4493.791)
--(axis cs:13,4320.67666666667)
--(axis cs:12,4184.62916666667)
--(axis cs:11,4167.90666666667)
--(axis cs:10,4095.1)
--(axis cs:9,4161.69333333333)
--(axis cs:8,3956.92866666667)
--(axis cs:7,3249.7955)
--(axis cs:6,2813.09333333333)
--(axis cs:5,2409.82066666667)
--(axis cs:4,1319.64666666667)
--(axis cs:3,666.783333333333)
--(axis cs:2,633.446666666667)
--(axis cs:1,405.233333333333)
--(axis cs:0,401.433333333333)
--cycle;

\addplot [very thick, forestgreen4316043]
table {%
0 149.665846666667
1 339.61471
2 370.278433333333
3 374.270453333333
4 392.456226666667
5 396.95199
6 406.237676666667
7 411.903646666667
8 412.5313
9 409.126363333333
10 401.729983333333
11 406.599043333333
12 437.983803333333
13 460.892216666667
14 476.73593
15 455.74318
16 457.11532
17 448.616843333333
18 495.78679
19 451.505826666667
};
\addlegendentry{Spectral Normalization ($\beta=5$)}
\addplot [very thick, sienna1398575]
table {%
0 217.575276666667
1 313.227436666667
2 492.725746666667
3 812.28302
4 1009.09388333333
5 1493.85498
6 1828.43158666667
7 2066.25551666667
8 2373.88800666667
9 2530.40953666667
10 2870.50161
11 3036.80485666667
12 2997.36538666667
13 3028.88714333333
14 3035.73192
15 3126.48660333333
16 2828.84317
17 3182.72496666667
18 3272.53894333333
19 3322.00387666667
};
\addlegendentry{Spectral Normalization ($\beta=10$)}
\addplot [very thick, hotpink]
table {%
0 246.049606666667
1 345.15639
2 528.24637
3 792.462223333333
4 1489.37335333333
5 2083.90902
6 2550.47054666667
7 2922.84349666667
8 3180.34121
9 3478.94380333333
10 3416.92292
11 3931.92752666667
12 3912.89976
13 4045.77148333333
14 3796.07533666667
15 3605.57263
16 3924.73296
17 4016.05071666667
18 4281.33382666667
19 3932.75948333333
};
\addlegendentry{Spectral Normalization ($\beta=15$)}
\addplot [very thick, royalblue75113229]
table {%
0 225.201663333333
1 364.52536
2 506.68646
3 528.795723333333
4 775.49912
5 961.39358
6 1500.77848
7 1942.25597333333
8 2193.11813333333
9 2567.76776
10 2876.61252
11 2914.39263333333
12 3066.8817
13 3188.15902
14 3302.37533
15 3492.98378666667
16 3388.72852666667
17 3501.87037
18 3870.12989666667
19 3870.59570666667
};
\addlegendentry{Spectral Normalization ($\beta=20$)}
\addplot [very thick, crimson2143839]
table {%
0 382.168116666667
1 388.309206666667
2 579.615186666667
3 625.400793333333
4 1186.66403333333
5 2239.77152
6 2729.29272666667
7 3124.35828
8 3881.18942666667
9 4043.88450333333
10 4068.62372666667
11 3992.46487333333
12 4039.29506
13 4278.03178333333
14 4408.7318
15 4537.66727
16 4415.68692
17 4558.91764
18 4609.54126
19 4447.90295333333
};
\addlegendentry{Robust Regularization}
\end{axis}

\end{tikzpicture} 
  \vspace{-1.5em}
  \centering
    \label{sfig:vaml_d}
 \end{subfigure}
 \vspace{-0.6em}
 \caption{(Left Two) Robust Regularization with a comparison of uniform random noise on Walker. (Right Two) Spectral Normalization with different spectral radius $\beta$ on Walker.
 }
\label{fig:vaml}
\vspace{-1.5em}
\end{figure}
As we see from the plot, to reduce the value-aware model error so that it is about the same as robust regularization, we need to choose the spectral radius $\beta$ as small as 10. However, under this constraint, it achieves a significantly worse performance than robust regularization, indicating that this constraint is perhaps too strong. To get the best performance, spectral normalization has to use a larger spectral radius to trade value-aware model error for expressive power and thus has worse empirical performance. In addition, we observe from Figure \ref{fig:main_experimental_results} that the performance discrepancy between the two methods is even greater in more complicated environments such as Ant and Humanoid, which further shows the limitation of constraining the global Lipschitz constant.




\section{Related Work}\label{6_relatedwork}
\textbf{Dynamics ensemble in model-based reinforcement learning.} \quad
Dynamics model ensemble is first introduced in MBRL by \citet{kurutach2018model} to avoid the learned policy exploit the insufficient data regions. Then \citet{chua2018deep} proposed a probabilistic dynamics ensemble to capture the aleatoric uncertainty and epistemic uncertainty, and achieves significant performance improvement compared to the deterministic dynamics ensemble.
Probabilistic dynamics ensemble is now widely used in MBRL methods (e.g. \cite{{29}, {clavera2020modelaugmented}, {d2020learn}, {lai2021effective}, {froehlich2022onpolicy}, {li2022gradient}}).
However, despite its popularity, there is still no clear explanation why the use of probabilistic dynamics ensemble brings such a large improvement to the policy. In this paper, we fill this gap and propose a novel theoretical explanation on why probabilistic dynamics ensemble works well.

\textbf{Lipschitz continuity in reinforcement learning.} \quad 
In the realm of model-free RL, a few recent works also apply the technique of spectral normalization to regularize the Lipschitz condition of the RL agent's value function, which results in more stable optimization dynamics \citep{gogianu21a, bjorck2021}. \citet{ball2021} investigate the effects of added Gaussian noise of exploration on the smoothness of value functions. In addition, \citet{shen20, zhang2021} propose to use a similar robust regularizer as our robust regularization method, aiming at an adversarially robust policy. In our work, we focus on  MBRL, a fundamentally different setting. 
As explained in Section~\ref{s3s1_insight} and Appendix~\ref{app:mf}, if no restrictions are imposed on the value function class, then the value-aware model error could explode, whereas the model error vanishes in the model-free setting.  In MBRL, \citet{osband2014} connects Lipschitz constant of the value function to the regret bound of posterior sampling for RL. \citet{asadi18a} proposed to learn a generalized Lipschitz transition model with respect to the Wasserstein metric, resulting in a bounded multi-step model prediction error and Lipschitz optimal value function induced from the learned dynamics. In contrast, our paper focuses on the local Lipschitz condition of the value function instead of the underlying transition model, studying its relation with the suboptimality of MBRL algorithm.
\section{Conclusion and Discussion}\label{7_conclusion}
In this paper, we provide insight into why the probabilistic ensemble model can achieve great empirical performance. 
We demonstrate the importance of the local Lipschitz condition of value functions in MBRL algorithms and justify our hypothesis with both theoretical and empirical results. 
Based on our insight, we propose two training mechanisms that directly regularize the Lipschitz condition of the value function. 
Empirical studies demonstrate the effectiveness of the proposed mechanisms.
One limitation is that if we use a single model instead of an ensemble, then we cannot use many existing methods of uncertainty estimation based on model ensemble, and we need to redesign the uncertainty estimation mechanism based on a single model. We leave this as future work.
\section*{Acknowledgement}

This work is supported by National Science Foundation NSF-IIS-FAI program, DOD-ONR-Office of Naval Research, DOD Air Force Office of Scientific Research, DOD-DARPA-Defense Advanced Research Projects Agency Guaranteeing AI Robustness against Deception (GARD), Adobe, Capital One and JP Morgan faculty fellowships.

\bibliography{iclr2023_conference}

\begin{thebibliography}{34}
\providecommand{\natexlab}[1]{#1}
\providecommand{\url}[1]{\texttt{#1}}
\expandafter\ifx\csname urlstyle\endcsname\relax
  \providecommand{\doi}[1]{doi: #1}\else
  \providecommand{\doi}{doi: \begingroup \urlstyle{rm}\Url}\fi

\bibitem[Asadi et~al.(2018)Asadi, Misra, and Littman]{asadi18a}
Kavosh Asadi, Dipendra Misra, and Michael Littman.
\newblock {L}ipschitz continuity in model-based reinforcement learning.
\newblock In Jennifer Dy and Andreas Krause (eds.), \emph{Proceedings of the
  35th International Conference on Machine Learning}, volume~80 of
  \emph{Proceedings of Machine Learning Research}, pp.\  264--273. PMLR, 10--15
  Jul 2018.
\newblock URL \url{https://proceedings.mlr.press/v80/asadi18a.html}.

\bibitem[Asadi et~al.(2019)Asadi, Misra, Kim, and Littman]{asadi2019combating}
Kavosh Asadi, Dipendra Misra, Seungchan Kim, and Michel~L Littman.
\newblock Combating the compounding-error problem with a multi-step model.
\newblock \emph{arXiv preprint arXiv:1905.13320}, 2019.

\bibitem[Ball \& Roberts(2021)Ball and Roberts]{ball2021}
Philip~J. Ball and Stephen~J. Roberts.
\newblock Offcon\({}^{\mbox{3}}\): What is state of the art anyway?
\newblock \emph{CoRR}, abs/2101.11331, 2021.
\newblock URL \url{https://arxiv.org/abs/2101.11331}.

\bibitem[Bartlett et~al.(2005)Bartlett, Bousquet, and Mendelson]{bartlett2005}
Peter~L. Bartlett, Olivier Bousquet, and Shahar Mendelson.
\newblock Local rademacher complexities.
\newblock \emph{The Annals of Statistics}, 33\penalty0 (4):\penalty0
  1497--1537, 2005.
\newblock ISSN 00905364.
\newblock URL \url{http://www.jstor.org/stable/3448616}.

\bibitem[Bjorck et~al.(2021)Bjorck, Gomes, and Weinberger]{bjorck2021}
Nils Bjorck, Carla~P Gomes, and Kilian~Q Weinberger.
\newblock Towards deeper deep reinforcement learning with spectral
  normalization.
\newblock In M.~Ranzato, A.~Beygelzimer, Y.~Dauphin, P.S. Liang, and J.~Wortman
  Vaughan (eds.), \emph{Advances in Neural Information Processing Systems},
  volume~34, pp.\  8242--8255. Curran Associates, Inc., 2021.
\newblock URL
  \url{https://proceedings.neurips.cc/paper/2021/file/4588e674d3f0faf985047d4c3f13ed0d-Paper.pdf}.

\bibitem[Buckman et~al.(2018)Buckman, Hafner, Tucker, Brevdo, and
  Lee]{buckman2018}
Jacob Buckman, Danijar Hafner, George Tucker, Eugene Brevdo, and Honglak Lee.
\newblock Sample-efficient reinforcement learning with stochastic ensemble
  value expansion.
\newblock In S.~Bengio, H.~Wallach, H.~Larochelle, K.~Grauman, N.~Cesa-Bianchi,
  and R.~Garnett (eds.), \emph{Advances in Neural Information Processing
  Systems}, volume~31. Curran Associates, Inc., 2018.
\newblock URL
  \url{https://proceedings.neurips.cc/paper/2018/file/f02208a057804ee16ac72ff4d3cec53b-Paper.pdf}.

\bibitem[Chua et~al.(2018)Chua, Calandra, McAllister, and Levine]{chua2018deep}
Kurtland Chua, Roberto Calandra, Rowan McAllister, and Sergey Levine.
\newblock Deep reinforcement learning in a handful of trials using
  probabilistic dynamics models.
\newblock \emph{Advances in Neural Information Processing Systems}, 31, 2018.

\bibitem[Clavera et~al.(2020)Clavera, Fu, and
  Abbeel]{clavera2020modelaugmented}
Ignasi Clavera, Yao Fu, and Pieter Abbeel.
\newblock Model-augmented actor-critic: Backpropagating through paths.
\newblock In \emph{International Conference on Learning Representations}, 2020.

\bibitem[D'Oro \& Ja{\'s}kowski(2020)D'Oro and Ja{\'s}kowski]{d2020learn}
Pierluca D'Oro and Wojciech Ja{\'s}kowski.
\newblock How to learn a useful critic? model-based action-gradient-estimator
  policy optimization.
\newblock \emph{Advances in Neural Information Processing Systems},
  33:\penalty0 313--324, 2020.

\bibitem[Farahmand et~al.(2017{\natexlab{a}})Farahmand, Barreto, and
  Nikovski]{farahmand17itervaml}
Amir-Massoud Farahmand, Andre Barreto, and Daniel Nikovski.
\newblock {Value-Aware Loss Function for Model-based Reinforcement Learning}.
\newblock In Aarti Singh and Jerry Zhu (eds.), \emph{Proceedings of the 20th
  International Conference on Artificial Intelligence and Statistics},
  volume~54 of \emph{Proceedings of Machine Learning Research}, pp.\
  1486--1494. PMLR, 20--22 Apr 2017{\natexlab{a}}.
\newblock URL \url{https://proceedings.mlr.press/v54/farahmand17a.html}.

\bibitem[Farahmand et~al.(2017{\natexlab{b}})Farahmand, Barreto, and
  Nikovski]{farahmand2017vaml}
Amir-massoud Farahmand, Andre Barreto, and Daniel Nikovski.
\newblock Value-aware loss function for model-based reinforcement learning.
\newblock In \emph{Artificial Intelligence and Statistics}, pp.\  1486--1494.
  PMLR, 2017{\natexlab{b}}.

\bibitem[Froehlich et~al.(2022)Froehlich, Lefarov, Zeilinger, and
  Berkenkamp]{froehlich2022onpolicy}
Lukas Froehlich, Maksym Lefarov, Melanie Zeilinger, and Felix Berkenkamp.
\newblock On-policy model errors in reinforcement learning.
\newblock In \emph{International Conference on Learning Representations}, 2022.

\bibitem[Gogianu et~al.(2021)Gogianu, Berariu, Rosca, Clopath, Busoniu, and
  Pascanu]{gogianu21a}
Florin Gogianu, Tudor Berariu, Mihaela~C Rosca, Claudia Clopath, Lucian
  Busoniu, and Razvan Pascanu.
\newblock Spectral normalisation for deep reinforcement learning: An
  optimisation perspective.
\newblock In Marina Meila and Tong Zhang (eds.), \emph{Proceedings of the 38th
  International Conference on Machine Learning}, volume 139 of
  \emph{Proceedings of Machine Learning Research}, pp.\  3734--3744. PMLR,
  18--24 Jul 2021.
\newblock URL \url{https://proceedings.mlr.press/v139/gogianu21a.html}.

\bibitem[Goodfellow et~al.(2015)Goodfellow, Shlens, and Szegedy]{ian2015}
Ian Goodfellow, Jonathon Shlens, and Christian Szegedy.
\newblock Explaining and harnessing adversarial examples.
\newblock In \emph{International Conference on Learning Representations}, 2015.
\newblock URL \url{http://arxiv.org/abs/1412.6572}.

\bibitem[Grimm et~al.(2020)Grimm, Barreto, Singh, and Silver]{grimm2020value}
Christopher Grimm, Andr{\'e} Barreto, Satinder Singh, and David Silver.
\newblock The value equivalence principle for model-based reinforcement
  learning.
\newblock \emph{Advances in Neural Information Processing Systems},
  33:\penalty0 5541--5552, 2020.

\bibitem[Grimm et~al.(2021)Grimm, Barreto, Farquhar, Silver, and
  Singh]{grimm2021}
Christopher Grimm, Andre Barreto, Greg Farquhar, David Silver, and Satinder
  Singh.
\newblock Proper value equivalence.
\newblock In M.~Ranzato, A.~Beygelzimer, Y.~Dauphin, P.S. Liang, and J.~Wortman
  Vaughan (eds.), \emph{Advances in Neural Information Processing Systems},
  volume~34, pp.\  7773--7786. Curran Associates, Inc., 2021.
\newblock URL
  \url{https://proceedings.neurips.cc/paper/2021/file/400e5e6a7ce0c754f281525fae75a873-Paper.pdf}.

\bibitem[Haarnoja et~al.(2018)Haarnoja, Zhou, Abbeel, and
  Levine]{haarnoja2018soft}
Tuomas Haarnoja, Aurick Zhou, Pieter Abbeel, and Sergey Levine.
\newblock Soft actor-critic: Off-policy maximum entropy deep reinforcement
  learning with a stochastic actor.
\newblock In \emph{International conference on machine learning}, pp.\
  1861--1870. PMLR, 2018.

\bibitem[Huang et~al.(2021)Huang, Xie, Bharadhwaj, and
  Shkurti]{huang2021continual}
Yizhou Huang, Kevin Xie, Homanga Bharadhwaj, and Florian Shkurti.
\newblock Continual model-based reinforcement learning with hypernetworks.
\newblock In \emph{2021 IEEE International Conference on Robotics and
  Automation (ICRA)}, pp.\  799--805. IEEE, 2021.

\bibitem[Janner et~al.(2019)Janner, Fu, Zhang, and Levine]{janner2019trust}
Michael Janner, Justin Fu, Marvin Zhang, and Sergey Levine.
\newblock When to trust your model: Model-based policy optimization.
\newblock \emph{Advances in Neural Information Processing Systems},
  32:\penalty0 12519--12530, 2019.

\bibitem[Kamthe \& Deisenroth(2018)Kamthe and Deisenroth]{kamthe2018data}
Sanket Kamthe and Marc Deisenroth.
\newblock Data-efficient reinforcement learning with probabilistic model
  predictive control.
\newblock In \emph{International conference on artificial intelligence and
  statistics}, pp.\  1701--1710. PMLR, 2018.

\bibitem[Kocijan et~al.(2004)Kocijan, Murray-Smith, Rasmussen, and
  Girard]{kocijan2004gaussian}
Ju{\v{s}} Kocijan, Roderick Murray-Smith, Carl~Edward Rasmussen, and Agathe
  Girard.
\newblock Gaussian process model based predictive control.
\newblock In \emph{Proceedings of the 2004 American control conference},
  volume~3, pp.\  2214--2219. IEEE, 2004.

\bibitem[Kurutach et~al.(2018)Kurutach, Clavera, Duan, Tamar, and
  Abbeel]{kurutach2018model}
Thanard Kurutach, Ignasi Clavera, Yan Duan, Aviv Tamar, and Pieter Abbeel.
\newblock Model-ensemble trust-region policy optimization.
\newblock In \emph{International Conference on Learning Representations}, 2018.

\bibitem[Lai et~al.(2020)Lai, Shen, Zhang, and Yu]{29}
Hang Lai, Jian Shen, Weinan Zhang, and Yong Yu.
\newblock Bidirectional model-based policy optimization.
\newblock In \emph{International Conference on Machine Learning}, pp.\
  5618--5627. PMLR, 2020.

\bibitem[Lai et~al.(2021)Lai, Shen, Zhang, Huang, Zhang, Tang, Yu, and
  Li]{lai2021effective}
Hang Lai, Jian Shen, Weinan Zhang, Yimin Huang, Xing Zhang, Ruiming Tang, Yong
  Yu, and Zhenguo Li.
\newblock On effective scheduling of model-based reinforcement learning.
\newblock In \emph{Thirty-Fifth Conference on Neural Information Processing
  Systems}, 2021.

\bibitem[Li et~al.(2022)Li, Wang, Chen, Liu, Ma, and Liu]{li2022gradient}
Chongchong Li, Yue Wang, Wei Chen, Yuting Liu, Zhi-Ming Ma, and Tie-Yan Liu.
\newblock Gradient information matters in policy optimization by
  back-propagating through model.
\newblock In \emph{International Conference on Learning Representations}, 2022.

\bibitem[Lin(2022)]{xingyu2022}
Xingyu Lin.
\newblock mbpo\_pytorch.
\newblock \url{https://github.com/Xingyu-Lin/mbpo_pytorch}, 2022.

\bibitem[Liu et~al.(2022)Liu, Zhu, Zhuang, Zhang, Hao, Yu, and
  Wang]{liu2022plan}
Minghuan Liu, Zhengbang Zhu, Yuzheng Zhuang, Weinan Zhang, Jianye Hao, Yong Yu,
  and Jun Wang.
\newblock Plan your target and learn your skills: Transferable state-only
  imitation learning via decoupled policy optimization.
\newblock In \emph{Proceedings of the 39th International Conference on Machine
  Learning}, 2022.

\bibitem[Miyato et~al.(2018)Miyato, Kataoka, Koyama, and Yoshida]{miyato2018}
Takeru Miyato, Toshiki Kataoka, Masanori Koyama, and Yuichi Yoshida.
\newblock Spectral normalization for generative adversarial networks.
\newblock In \emph{International Conference on Learning Representations}, 2018.
\newblock URL \url{https://openreview.net/forum?id=B1QRgziT-}.

\bibitem[Munos(2005)]{munos2005error}
R{\'e}mi Munos.
\newblock Error bounds for approximate value iteration.
\newblock In \emph{Proceedings of the National Conference on Artificial
  Intelligence}, volume~20, pp.\  1006. Menlo Park, CA; Cambridge, MA; London;
  AAAI Press; MIT Press; 1999, 2005.

\bibitem[Nguyen-Tuong et~al.(2008)Nguyen-Tuong, Peters, and
  Seeger]{nguyen2008local}
Duy Nguyen-Tuong, Jan Peters, and Matthias Seeger.
\newblock Local gaussian process regression for real time online model
  learning.
\newblock \emph{Advances in neural information processing systems}, 21, 2008.

\bibitem[Osband \& Van~Roy(2014)Osband and Van~Roy]{osband2014}
Ian Osband and Benjamin Van~Roy.
\newblock Model-based reinforcement learning and the eluder dimension.
\newblock In Z.~Ghahramani, M.~Welling, C.~Cortes, N.~Lawrence, and K.Q.
  Weinberger (eds.), \emph{Advances in Neural Information Processing Systems},
  volume~27. Curran Associates, Inc., 2014.
\newblock URL
  \url{https://proceedings.neurips.cc/paper/2014/file/1141938ba2c2b13f5505d7c424ebae5f-Paper.pdf}.

\bibitem[Pathak et~al.(2017)Pathak, Agrawal, Efros, and
  Darrell]{pathak2017curiosity}
Deepak Pathak, Pulkit Agrawal, Alexei~A Efros, and Trevor Darrell.
\newblock Curiosity-driven exploration by self-supervised prediction.
\newblock In \emph{International conference on machine learning}, pp.\
  2778--2787. PMLR, 2017.

\bibitem[Shen et~al.(2020)Shen, Li, Jiang, Wang, and Zhao]{shen20}
Qianli Shen, Yan Li, Haoming Jiang, Zhaoran Wang, and Tuo Zhao.
\newblock Deep reinforcement learning with robust and smooth policy.
\newblock In Hal~Daumé III and Aarti Singh (eds.), \emph{Proceedings of the
  37th International Conference on Machine Learning}, volume 119 of
  \emph{Proceedings of Machine Learning Research}, pp.\  8707--8718. PMLR,
  13--18 Jul 2020.
\newblock URL \url{https://proceedings.mlr.press/v119/shen20b.html}.

\bibitem[Zhang et~al.(2020)Zhang, Chen, Xiao, Li, Liu, Boning, and
  Hsieh]{zhang2021}
Huan Zhang, Hongge Chen, Chaowei Xiao, Bo~Li, Mingyan Liu, Duane Boning, and
  Cho-Jui Hsieh.
\newblock Robust deep reinforcement learning against adversarial perturbations
  on state observations.
\newblock In H.~Larochelle, M.~Ranzato, R.~Hadsell, M.F. Balcan, and H.~Lin
  (eds.), \emph{Advances in Neural Information Processing Systems}, volume~33,
  pp.\  21024--21037. Curran Associates, Inc., 2020.
\newblock URL
  \url{https://proceedings.neurips.cc/paper/2020/file/f0eb6568ea114ba6e293f903c34d7488-Paper.pdf}.

\end{thebibliography}
\bibliographystyle{iclr2023_conference}

\clearpage
\newpage
\appendix
{\begin{center} \bf \LARGE
    Supplementary Material
    \end{center}
}
\section{Additional Definitions and notation}
We start with the definition of the Concentrability constant
\begin{definition}
\label{df:concentrability}
(\textbf{Concentrability Constant)} \citep{farahmand17itervaml}\newline
Given $\rho, \nu\in \Delta(\mathcal{S})$, an integer $k>0$, and an arbitrary sequence of policies $({\pi}_i)_{i=1}^{k}$, the distribution $\rho \mathcal{P}^{\pi_1} \mathcal{P}^{\pi_2}... \mathcal{P}^{\pi_k}$ denotes the future state distribution obtained when the state in the first step is distributed according to $\rho$ and the agent follows the sequence of policies $\pi_1,..., \pi_k$. Define:
$$c_{\rho, \nu}(k)=\sup_{\pi_1,...,\pi_k}\Big\|\frac{d\rho \mathcal{P}^{\pi_1}\mathcal{P}^{\pi_2}... \mathcal{P}^{\pi_k}}{d\nu}\Big\|_{2,\nu}$$
Here, $\|f\|_{2,\nu}^2=\int f(s)^2d\nu$. The derivative $\displaystyle \frac{d\rho \mathcal{P}^{\pi_1}\mathcal{P}^{\pi_2}... \mathcal{P}^{\pi_k}}{d\nu}$ is the Radon-Nikydom Derivative of two probability measures, which is well-defined up to a set of measure zero by $\nu$ if $\rho \mathcal{P}^{\pi_1}\mathcal{P}^{\pi_2}... \mathcal{P}^{\pi_k}$ is absolutely continuous with respect to $\nu$. In case it's not absolutely continuous, we set it to be $\infty$. Then, for a constant $0\leq\gamma<1$, define the discounted weighted average concentrability coefficient as
$$
C(\rho, \nu)=(1-\gamma)^2\sum_{k=1}^\infty\gamma^{k-1}c_{\rho, \nu}(k)
$$
\end{definition}

Throughout the proof, we use $\|f\|_{p, \mu}$ denote the $L_p(\mu)$-norm $1\leq p < \infty$ of a measurable function $f:\mathcal{S}\rightarrow \mathbb{R}$ such that $$
\|f\|_{p, \mu}^p=\int_{\mathcal{S}}|f(x)|^pd\mu(x)
$$
In addition, we define the empirical norm. Given a collection points $\{s_1, ..., s_n\}$ in $\mathcal{S}$, define the empirical norm $L_p(s_1,..., s_n)$ such that $$\|f\|_{L_p(s_1,..., s_n)}^p=\frac{1}{N}\sum_{i=1}^{N}|f(s_i)|^p$$

Finally, we define Rademacher complexity as in \citep{bartlett2005}. Same as \citet{farahmand17itervaml}, we will use a local variant of Rademacher complexity to derive the rate of estimation error.
\begin{definition} 
(Rademacher Complexity) Let $\sigma_1,..., \sigma_n$ be $n$ independent Rademacher random variables, i.e. $\mathbb{P}\{\sigma_i=-1\}=\mathbb{P}\{\sigma_i=1\}=\frac{1}{2}$. Given a collection of measurable functions $\mathcal{F}$ from $\mathcal{X}$ to $\mathbb{R}$ and a probability distribution $\mu$ over $\mathcal{X}$, we sample n points $x_1,..., x_n$ i.i.d. from $\mu$. Define $$R_n\mathcal{F}=\sup_{f\in \mathcal{F}}\frac{1}{N}\sum_{i=1}^{n}\sigma_i f(x_i)$$
Then we define the Rademacher complexity of $\mathcal{F}$ as $\mathbb{E}[R_n\mathcal{\mathcal{F}}]$
\end{definition}
Besides, same as \citep{bartlett2005}, we define the sub-root function as non-negative and non-decreasing function $\psi: [0;\infty) \rightarrow [0, \infty)$ such that $r\mapsto \frac{\psi(r)}{\sqrt{r}}$ is non-increasing for $r>0$

\section{Proof of the theorem}
\label{app:proof}

We begin by citing the following theorem.
\begin{theorem} \citep{bartlett2005}
\label{thm:local_rad}
Let $\mathcal{F}$ be a class of functions with values in range $[a,b]$ and assume that there are some functional $T:\mathcal{F}\rightarrow \mathbb{R}^+$ and some constant $B$ such that for every $f\in \mathcal{F}$, 
\begin{equation}
Var[f]\leq T(f)\leq B\mathbb{E}(f)
\end{equation}
Let $\psi$ be a sub-root function and let $r^*=r^*(\mathcal{F})$ be the fixed point of $\psi$. Assume that for any $r\leq r^*$, $\psi$ satisfies 
\begin{equation}
\label{eq:subroot}
\psi(r)>B\mathbb{E}[R_n\{f\in \mathcal{\mathbb{F}}:T(f)\leq r\}]
\end{equation}
Then, with $c_1=704$ and $c_2=26$, for any $K>1$ and every $x>0$, with probability at least $1-e^{-x}$, for any $f\in \mathcal{F}$, we have 
\begin{equation}
\mathbb{E}[f]\leq \frac{K}{K-1}\mathbb{E}_n[f]+\frac{c_1K}{B}r^*+\frac{x(11(b-a)+c_2BK)}{n}
\end{equation}
Also with a probability at least $1-e^{-x}$, for any $f\in \mathcal{F}$, we have 
\begin{equation}
\mathbb{E}_n[f]\leq \frac{K}{K-1}\mathbb{E}[f]+\frac{c_1K}{B}r^*+\frac{x(11(b-a)+c_2BK)}{n}
\end{equation}
We will then use $r^\star(\mathcal{F})$ to denote the fixed point of a sub-root function $\psi$ that satisfies \ref{eq:subroot}
\end{theorem}

Now we prove the following theorem which provides a finite sample bound on the value-aware model error.
\begin{theorem}
\label{t1_finite_sample}
Under the four assumptions \ref{a0_deterministic},\ref{a1_boundedness}, \ref{a2_realizability}, \ref{a4_complexity}, with the probability model learned based on Equation \ref{eq:model}, there exists a constant $\kappa(\alpha)$ which depends solely on $\alpha\in (0,1)$, such that with probability $1-\delta$,
\begin{align}
L(\hat{\mathcal{P}})&\leq \epsilon+\frac{\kappa(\alpha)D^2R^\frac{2\alpha}{1+\alpha}\ln (\frac{1}{\delta})}{N^{\frac{1}{1+\alpha}}},
\end{align}
where $N$ is the number of samples from the data-collection distribution $\rho$, $D$ is the size of the state-space defined in assumption \ref{a1_boundedness}, and $R$ is defined in the metric entropy condition of the model class in assumption \ref{a4_complexity}. 
\end{theorem}

\begin{proof}
Given a batch of state-action transition triples $\{(s_i, a_i, s'_i)\}_{i=1}^{N}$ with $(s_i, a_i)$ sampled i.i.d from data distribution $\rho$, we denote the empirical loss 
\begin{equation}
L_n(\mathcal{P})=\frac{1}{N}\sum_{i=1}^{N}\int_{\mathcal{S}}\mathcal{P}(\hat{s}_i'|s_i,a_i)\|\hat{s}_i'-s_i'\|d\hat{s_i}'
\end{equation}
We also denote the underlying loss over the data distribution $\rho$ as \begin{equation}
L(\mathcal{P})=\mathbb{E}_{(s, a)\sim \rho}\Big[\int_{\mathcal{S}}\mathcal{P}(\hat{s}'|s,a)\|\hat{s}^{'}-s'\|d\hat{s}'\Big]\end{equation}
In addition, let $\tilde{\mathcal{P}}\in \mathcal{M}$ be the best model in the transition kernel class $\mathcal{M}$ defined in \ref{a4_complexity}.

Now we would like to apply Theorem \ref{t1_finite_sample} to bound the difference between the empirical and the true underlying loss. First, let $\mathcal{F}=\{(s\times a,s')\mapsto l(s\times a, s'; \mathcal{P})-l(s\times a, s'; \tilde{\mathcal{P}}); \mathcal{P}\in \mathcal{M}\}$ be the class of functions in Theorem~\ref{thm:local_rad}, where $l(s\times a,s'; \mathcal{P})$ is the single datapoint version of the empirical loss $L_n(\mathcal{P})$. So $\mathbb{E}_{s\times a\sim \rho}[l(s\times a, s'; \mathcal{P})]=L(\mathcal{P})$.
Now by assumption \ref{a1_boundedness}, $0 \leq l(s\times a, a; \mathcal{P})\leq 2D$ for every $s\times a\in \mathcal{S}\times \mathcal{A}, s'\in \mathcal{S}$, and $\mathcal{P}\in \mathcal{M}$. Therefore, the value of $f$ is bounded between $-2D$ and $2D$ for every $f\in \mathcal{F}$,

As a consequence, for every $f\in \mathcal{F}$, $Var(f)\leq \mathbb{E}[f^2]\leq 4D^2$. So we can set
\begin{align*}
T(f)&=4D^2\mathbb{E}\Big[\int_{\mathcal{S}}\mathcal{P}(\hat{s}'|s,a)\|\hat{s}'-s'\|d\hat{s}'\Big] \\
B &= 4D^2
\end{align*}
Now, we can apply Theorem \ref{t1_finite_sample} to conclude that with probability $1-\delta$ (let $K=2$), 
\begin{equation}
L(\mathcal{P})-L(\tilde{\mathcal{P}})\leq 2(L_n(\mathcal{P})-L_n(\tilde{\mathcal{P}}))+\frac{2\times 704}{4D^2}r^*(\mathcal{F})+\frac{(11\times 4D+2\times 26 \times 4D^2)\ln(\frac{1}{\delta})}{N}
\end{equation}
Since $\hat{\mathcal{P}}$ is the minimizer of the empirical loss $L_n(\mathcal{P})$, 
\begin{equation}
\label{eq:excess_error}
L(\hat{\mathcal{P}})-L(\tilde{\mathcal{P}})\leq \frac{352}{D^2}r^*(\mathcal{F})+\frac{(44D+208D^2)\ln(\frac{1}{\delta})}{N}
\end{equation}
We can provide an upper bound of the local Rademacher complexity $r^*(\mathcal{F})$: there exists a finite constant $\tau>0$ such that for a given $0\leq \alpha\leq 1$, we have 
\begin{equation}
r^*(\mathcal{F})\leq \frac{c_1(\alpha)D^4 R^\frac{2\alpha}{1+\alpha}}{N^{\frac{1}{1+\alpha}}}+\frac{\tau D^4 \ln N}{N},
\end{equation}
where $\displaystyle c(\alpha)=\frac{\tau}{(1-\alpha)^{\frac{2}{1+\alpha}}}$. The proof follows the exact same steps of Proposition 10 in \citet{farahmand17itervaml}.
Now back to Equation \ref{eq:excess_error}, by the realizability assumption 
\ref{a2_realizability}, the best model $\mathcal{\tilde{P}}$ in the model class satisfies that $L(\mathcal{\tilde{P}})\leq \epsilon$. Therefore, with probability $1-\delta$,
\begin{align}
L(\hat{\mathcal{P}})&\leq \epsilon+\frac{352c_1(\alpha)D^2R^\frac{2\alpha}{1+\alpha}}{N^{\frac{1}{1+\alpha}}}+\frac{352\tau D^2\ln N}{N}+\frac{(44D+208D^2)\ln (\frac{1}{\delta})}{N}
\end{align}
Finally, there should exist a constant $\kappa(\alpha)$ sufficiently large such that with probability $1-\delta$,
\begin{align}
L(\hat{\mathcal{P}})&\leq \epsilon+\frac{\kappa(\alpha)D^2R^\frac{2\alpha}{1+\alpha}\ln (\frac{1}{\delta})}{N^{\frac{1}{1+\alpha}}}
\end{align}
\end{proof}

\begin{corollary}
Under the five assumptions \ref{a0_deterministic},\ref{a1_boundedness}, \ref{a2_realizability}, and \ref{a4_complexity} with the probability model learned based on Equation \ref{eq:model}, there exists a constant $\kappa(\alpha)$ which depends solely on $\alpha\in (0,1)$, such that with probability $\displaystyle 1-\exp(-\frac{\epsilon N^{\frac{1}{1+\alpha}}}{\kappa(\alpha)D^2R^{\frac{2\alpha}{1+\alpha}}})$,
\begin{align}
L(\hat{\mathcal{P}})&\leq 2\epsilon,
\end{align}
\end{corollary}
\label{corollary_mse}
\begin{proof}
This is a straightforward application of Theorem \ref{t1_finite_sample}, where we could just let $\displaystyle \epsilon=\frac{\kappa(\alpha)D^2R^\frac{2\alpha}{1+\alpha}\ln (\frac{1}{\delta})}{N^{\frac{1}{1+\alpha}}}$.
\end{proof}

Next, we consider the local Lipschitz condition of the value function and provide a finite sample bound of the value-aware model error.
\begin{theorem}
\label{t2_vaml}
Under the five assumptions \ref{a0_deterministic},\ref{a1_boundedness}, \ref{a2_realizability}, \ref{a3_lip}, and \ref{a4_complexity}, with the probability model learned based on Equation \ref{eq:model}, there exists a constant $\kappa(\alpha)$ which depends solely on $\alpha\in (0,1)$, such that for any $m>1$,
\begin{align}
\int \big|\mathcal{T}^*Q(s,a)-\widehat{\mathcal{T}}^*Q(s,a)\big|^2 d\rho(s,a)\leq \gamma^2\big[4\epsilon^2 L^2 \xi
+\frac{R_{\max}^2}{(1-\gamma)^2}(1-\xi)\big],
\end{align}
where $\displaystyle \xi=1-\exp(-\frac{\epsilon N^{\frac{1}{1+\alpha}}}{\kappa(\alpha)D^2R^{\frac{2\alpha}{1+\alpha}}})$
\end{theorem}
\begin{proof}
\begin{align*}
\|\mathcal{T}^*Q-\widehat{\mathcal{T}}^*Q\|^2_\rho &= \int \Big|r(s,a)+\gamma V(s')-r(s,a)-\gamma\int \hat{\mathcal{P}}(d\hat{s}'|s,a) V(\hat{s}')\Big|^2 d\rho(s,a) \\
&=\gamma^2 \int \Big|\-\int \hat{\mathcal{P}}(d\hat{s}'|s,a) \big(V(\hat{s}')-V(s')\big)\Big|^2 d\rho(s,a) \\
&\leq \gamma^2 \int \int \mathcal{P}(ds'|s,a)\big(V(\hat{s}')-V(s')\big)^2 d\rho(s,a) \\
&\leq \gamma^2 \int \int \ind\{\|s'-\hat{s}'\|\leq 2\epsilon\}\big(V(\hat{s}')-V(s')\big)^2 \hat{\mathcal{P}}(d\hat{s}'|s,a)d\rho(s,a)
\\
&+\gamma^2 \int \int \ind\{\|s'-\hat{s}'\|> 2\epsilon\}\big(V(\hat{s}')-V(s')\big)^2 \hat{\mathcal{P}}(d\hat{s}'|s,a) d\rho(s,a)\\
&\leq \gamma^2 {2}^2\epsilon^2 L^2 \probP\{\|s'-\hat{s}'\|\leq 2\epsilon\}
+ \gamma^2\frac{R_{\max}^2}{(1-\gamma)^2} \probP\{\|s'-\hat{s}'\|> 2\epsilon\}
\end{align*}
Now apply Corollary \ref{corollary_mse},  
\begin{align*}
\|\mathcal{T}^*Q-\widehat{\mathcal{T}}^*Q\|^2_\rho
&\leq \gamma^2 4\epsilon^2 L^2 \big(1-\exp(-\frac{\epsilon N^{\frac{1}{1+\alpha}}}{\kappa(\alpha)D^2R^{\frac{2\alpha}{1+\alpha}}})\big)\\
&+\gamma^2\frac{R_{\max}^2}{(1-\gamma)^2}\exp(-\frac{\epsilon N^{\frac{1}{1+\alpha}}}{\kappa(\alpha)D^2R^{\frac{2\alpha}{1+\alpha}}})
\end{align*}
\end{proof}

Finally, with the value-aware model error bounded, we could apply the error propagation results from \citep{munos2005error, farahmand17itervaml} and prove our main theorem, which relates the local Lipschitz constant to the sub-optimality of the approximate model-based value iteration algorithm.
\begin{theorem}
\label{thm:ap_main}
Suppose $\hat{Q}_0$ is initialized such that  $\displaystyle \hat{Q}_0(s,a)\leq \frac{R_{\max}}{1-\gamma}$ for $\forall\: (s,a)$. Under the assumptions of \ref{a0_deterministic}, \ref{a1_boundedness}, \ref{a2_realizability}, \ref{a3_lip}, and \ref{a4_complexity}, after $K$ iterations of the model-based approximate value iteration algorithm, there exists a constant $\kappa(\alpha)$ which depends solely on $\alpha\in (0,1)$ such that, 
\begin{align}
\mathbb{E}_{(s,a)\sim \mathcal{P}_0}\Big[\Big|Q^*(s,a)-\hat{Q}_K(s,a)\Big|\Big] \leq \frac{2\gamma}{(1-\gamma)^2} \Big[&C(\rho, \mathcal{P}_0)\Big(\max_{0\leq k\leq K}\delta_k+\gamma^2\big(4\epsilon^2 L^2 \xi \nonumber \\
&+\frac{(1-\xi)R_{\max}^2}{(1-\gamma)^2}\big)\Big)+2\gamma^K R_{\max}\Big]
\end{align}
where  $\delta_k^2=\mathcal{L}_{\textit{reg}}(\hat{Q}_k; \hat{Q}_{k-1}, \hat{\mathcal{P}}_k)$ is the regression error defined in Equation~(\ref{eq:reg_error}), 
$\xi=\displaystyle 1-\exp(-\frac{\epsilon N^{\frac{1}{1+\alpha}}}{\kappa(\alpha)D^2R^{\frac{2\alpha}{1+\alpha}}})$, and $C(\rho, \mathcal{P}_0)$ is the concentrability constant defined in Definition~\ref{df:concentrability}.
\end{theorem}
\begin{proof}
This follows directly from the Theorem \ref{t2_vaml} and also Theorem 4 from \citep{farahmand17itervaml}, where the value-aware model error $\displaystyle e_{\textit{model}}(N)\leq \gamma^2\Big(4\epsilon^2 L^2 \xi
+\frac{R_{\max}^2}{(1-\gamma)^2}(1-\xi)\Big)$

\end{proof}

\section{Implementation Details}
\label{app:implementation}
In this section, we are going to introduce the implementation details of our two proposed methods.
We implement our methods based on a PyTorch implementation of MBPO~\citep{xingyu2022}. 
The dynamics model architecture is MLP with four hidden layers of size 200. In Ant and Humanoid, the hidden size is 400 because these two environments are more complex than others. 
For the probabilistic dynamics model ensemble, we set the ensemble size to 7 which is the setting used in the original paper of MBPO~\citep{janner2019trust}.
The policy is optimized with Soft Actor-Critic (SAC)~\citep{haarnoja2018soft}.
The actor network architecture and the critic network architecture are MLP with two hidden layers of size 256. 

For robust regularization, we fix $\epsilon$ to 0.1 as discussed in Section~\ref{s4s2_robust_reg}, and we do a grid search of $\lambda$ over $[0.01,0.1,1.0]$, with the result presented in Table~\ref{tab:grid_search_rr}
For spectral normalization, we add the normalization on each layer of the critic network. In particular,  at every forward pass, we approximate the spectral radius of the weight matrix with one step of power iteration. The algorithm is sketched below with $\bm{u}$ and $\bm{v}$ being the right and left singular vector of the weight matrix $W$. 
\begin{align}
\bm{v} \leftarrow W\bm{u}^{(t-1)}; \;\; \alpha\leftarrow \|\bm{v}\|; \;\;\;\bm{v}^{(t)}\leftarrow \alpha^{-1}\bm{v}\\
\bm{u} \leftarrow W^T\bm{v}^{(t)}; \;\; \rho\leftarrow \|\bm{v}\|; \;\;\;\bm{u}^{(t)}\leftarrow \rho^{-1}\bm{u}
\end{align}
Then we perform a projection of the parameters: $\displaystyle W\vcentcolon=\max(1, \frac{\max(\alpha, \rho)}{\beta})^{-1}W$. So the spectral norm will be clipped to $\beta$ if it's bigger than $\lambda$, unchanged otherwise.
We do a grid search of $\beta$ over $[15,20,25,30,35]$, with results shown in Table~\ref{tab:grid_search_spec_norm}.
For a fair comparison, we set the rollout horizon during model rollouts to 1 in all environments.
\begin{table*}[!t]
\vspace{-1em}
\centering
\renewcommand{\arraystretch}{1.1}
\caption{Robust Regularization with different regularization weights $\lambda$}
\resizebox{0.65\textwidth}{!}{%
\setlength{\tabcolsep}{3pt}
\begin{tabular}{|c|c|c|c|c|c|  }
\toprule
Environment & \thead{Probabilistic \\ Model Ensemble} & \thead{Deterministic \\ Ensemble} & $\lambda=0.01$ &  $\lambda=0.01$ & $\lambda=1.0$ \\
\cline{2-3}
\hline
Walker & 3236 & 4385 & 4138 & \textbf{4515} & 4112     \\
\hline
HalfCheetah & 10519 & 10648 & 10871 & \textbf{10982} & 10763 \\
\hline
Ant   & 1684 & 4290 & 4190 & \textbf{4592} & 4011     \\
\hline
Humanoid & 468  & 4235 & 4210   & 4531  & \textbf{5123} \\
\hline
Hopper& 2432 & 3278 & \textbf{3213} & 3110 & 3010 \\
\hline
\bottomrule
\end{tabular}}
\label{tab:grid_search_rr}
\end{table*}

\begin{table*}[!t]
\centering
\renewcommand{\arraystretch}{1.1}
\caption{Spectral Normalization with different spectral radius $\beta$}
\resizebox{0.52\textwidth}{!}{%
\setlength{\tabcolsep}{3pt}
\begin{tabular}{|c|c|c|c|c|c|  }
\toprule
Environment & $\beta=15$ & $\beta=20$ & $\beta=25$ & $\beta=30$ & $\beta=35$ \\
\cline{2-3}
\hline
Walker & 3870 & \textbf{4447} & 4124 & 3982 & 3645    \\
\hline
HalfCheetah & 9931 & 10452 & \textbf{10763} & 10313 & 10423  \\
\hline
Ant  & 3870 & \textbf{4447} & 4124 & 3982 & 3645  \\
\hline
Humanoid & 427 & 662 & 704 & \textbf{712} & 684  \\
\hline
Hopper & 2876 & 3125 & \textbf{3364} &  3243  & 2972 \\
\hline
\bottomrule
\end{tabular}}
\vspace{-0.5em}
\label{tab:grid_search_spec_norm}
\end{table*}
\section{Additional Experiment Results}
\label{app:add_exp}
\subsection{Additional Experimental Results of Value-aware Model Error}
\label{a3:add_vaml}
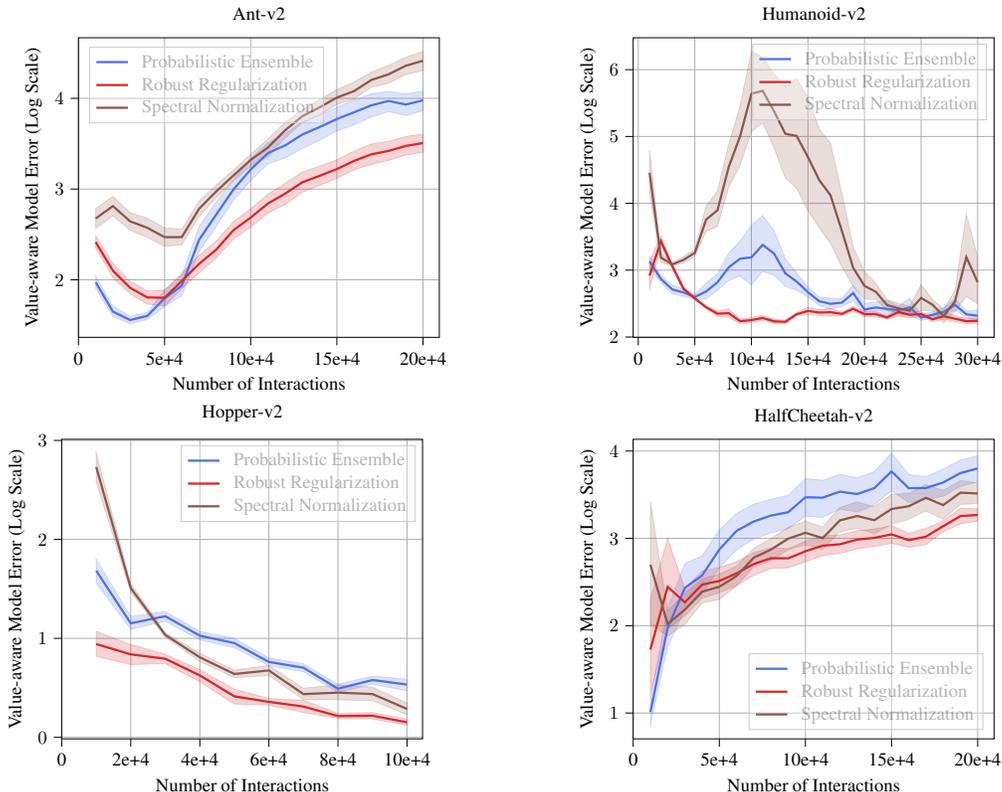
\begin{figure}[!htbp]
\vspace{0em}
\centering
\begin{subfigure}[t]{0.45\textwidth}
\centering
\begin{tikzpicture}[scale=0.7]

\definecolor{color0}{rgb}{0.297,0.445,0.9}
\definecolor{color1}{rgb}{0.83921568627451,0.152941176470588,0.156862745098039}
\definecolor{color2}{rgb}{0.549019607843137,0.337254901960784,0.294117647058824}

\begin{axis}[
legend cell align={left},
legend style={
  fill opacity=0.3,
  draw opacity=1,
  text opacity=1,
  at={(0.03,0.97)},
  anchor=north west,
  draw=white!80!black
},
tick align=outside,
tick pos=left,
title={Ant-v2},
x grid style={white!69.0196078431373!black},
xlabel={Number of Interactions},
xmajorgrids,
xmin=-1, xmax=19.95,
xtick style={color=black},
xtick={-1,4,9,14,19},
xticklabels={0,5e+4,10e+4,15e+4,20e+4},
y grid style={white!69.0196078431373!black},
ylabel={Value-aware Model Error (Log Scale)},
ymajorgrids,
ymin=1.36628168105073, ymax=4.66762234306533,
ytick style={color=black}
]
\path [draw=color0, fill=color0, opacity=0.2]
(axis cs:0,2.03838120004468)
--(axis cs:0,1.90623755515842)
--(axis cs:1,1.593064958193)
--(axis cs:2,1.51634262023321)
--(axis cs:3,1.5622538295288)
--(axis cs:4,1.70471338523099)
--(axis cs:5,1.82542698150625)
--(axis cs:6,2.30718073675421)
--(axis cs:7,2.58983496481708)
--(axis cs:8,2.87819752756695)
--(axis cs:9,3.09302505925097)
--(axis cs:10,3.27617692814279)
--(axis cs:11,3.34275019449838)
--(axis cs:12,3.44735748709997)
--(axis cs:13,3.53709697694799)
--(axis cs:14,3.64087889831236)
--(axis cs:15,3.70770971991553)
--(axis cs:16,3.79343259829778)
--(axis cs:17,3.85297875037391)
--(axis cs:18,3.81098721621295)
--(axis cs:19,3.86335017489763)
--(axis cs:19,4.08230510898971)
--(axis cs:19,4.08230510898971)
--(axis cs:18,4.04748713663627)
--(axis cs:17,4.07772796071982)
--(axis cs:16,4.05016062920835)
--(axis cs:15,3.98151079262885)
--(axis cs:14,3.90308332699848)
--(axis cs:13,3.81541666515765)
--(axis cs:12,3.7399913900018)
--(axis cs:11,3.6114672808773)
--(axis cs:10,3.50882332115242)
--(axis cs:9,3.34094857223294)
--(axis cs:8,3.10513404648726)
--(axis cs:7,2.86160233624732)
--(axis cs:6,2.58705349971248)
--(axis cs:5,2.0484900543839)
--(axis cs:4,1.88855960699433)
--(axis cs:3,1.64145128146386)
--(axis cs:2,1.59686388170086)
--(axis cs:1,1.71350990507237)
--(axis cs:0,2.03838120004468)
--cycle;

\path [draw=color1, fill=color1, opacity=0.2]
(axis cs:0,2.48717153319502)
--(axis cs:0,2.35079857467846)
--(axis cs:1,2.02249475464731)
--(axis cs:2,1.84363843800587)
--(axis cs:3,1.73321883616233)
--(axis cs:4,1.71062410652288)
--(axis cs:5,1.90234688018273)
--(axis cs:6,2.07491164361734)
--(axis cs:7,2.23718269901404)
--(axis cs:8,2.44790640543298)
--(axis cs:9,2.58056559339763)
--(axis cs:10,2.73479328092325)
--(axis cs:11,2.82851406198773)
--(axis cs:12,2.96212068863021)
--(axis cs:13,3.04072074510288)
--(axis cs:14,3.11786373725038)
--(axis cs:15,3.21262626587503)
--(axis cs:16,3.27138810488239)
--(axis cs:17,3.31570396915723)
--(axis cs:18,3.3738308637783)
--(axis cs:19,3.40563986105956)
--(axis cs:19,3.60672757247726)
--(axis cs:19,3.60672757247726)
--(axis cs:18,3.57854972226307)
--(axis cs:17,3.52552731493491)
--(axis cs:16,3.49781781713644)
--(axis cs:15,3.41932039304401)
--(axis cs:14,3.32296058571401)
--(axis cs:13,3.2526470985163)
--(axis cs:12,3.18732256701901)
--(axis cs:11,3.06341867945232)
--(axis cs:10,2.94411647832869)
--(axis cs:9,2.78454682549606)
--(axis cs:8,2.64443943100454)
--(axis cs:7,2.42781914262451)
--(axis cs:6,2.2702173853116)
--(axis cs:5,2.07283927873395)
--(axis cs:4,1.88621600531736)
--(axis cs:3,1.87887836029354)
--(axis cs:2,1.98448290965848)
--(axis cs:1,2.18183591414127)
--(axis cs:0,2.48717153319502)
--cycle;

\path [draw=color2, fill=color2, opacity=0.2]
(axis cs:0,2.78485222674881)
--(axis cs:0,2.56658704217063)
--(axis cs:1,2.70433707188055)
--(axis cs:2,2.54398325505894)
--(axis cs:3,2.46803696093086)
--(axis cs:4,2.36791564385928)
--(axis cs:5,2.38592862352157)
--(axis cs:6,2.69462757207421)
--(axis cs:7,2.90431455837744)
--(axis cs:8,3.09060585523355)
--(axis cs:9,3.25112769511654)
--(axis cs:10,3.3552868808112)
--(axis cs:11,3.56367286535018)
--(axis cs:12,3.73088850508557)
--(axis cs:13,3.82161411131225)
--(axis cs:14,3.92936783396254)
--(axis cs:15,3.99165608032286)
--(axis cs:16,4.12158751540122)
--(axis cs:17,4.16554993772375)
--(axis cs:18,4.26384346136856)
--(axis cs:19,4.30576574827057)
--(axis cs:19,4.51756140388285)
--(axis cs:19,4.51756140388285)
--(axis cs:18,4.45326707424418)
--(axis cs:17,4.36301662404693)
--(axis cs:16,4.28905580108106)
--(axis cs:15,4.17992753892285)
--(axis cs:14,4.0970914083544)
--(axis cs:13,3.98578581526549)
--(axis cs:12,3.87401234616502)
--(axis cs:11,3.73718798342895)
--(axis cs:10,3.55152290197211)
--(axis cs:9,3.38629557846747)
--(axis cs:8,3.20999429078348)
--(axis cs:7,3.04520419437888)
--(axis cs:6,2.86827771563474)
--(axis cs:5,2.5578664211966)
--(axis cs:4,2.57262336125202)
--(axis cs:3,2.68012410718755)
--(axis cs:2,2.73972604362803)
--(axis cs:1,2.92020314867759)
--(axis cs:0,2.78485222674881)
--cycle;

\addplot [very thick, color0]
table {%
0 1.97387114297969
1 1.64931782622824
2 1.55682567481625
3 1.60272049236109
4 1.79598124137684
5 1.93508107181998
6 2.43998161642334
7 2.72022544707291
8 2.99617369949893
9 3.21538918129977
10 3.39675291065129
11 3.48115879678733
12 3.60162379667224
13 3.68327389128719
14 3.77227816305445
15 3.84611838387074
16 3.9224047488101
17 3.9716011317335
18 3.93161404311333
19 3.97751401998562
};
\addlegendentry{Probabilistic Ensemble}
\addplot [very thick, color1]
table {%
0 2.41551837567952
1 2.10041876239553
2 1.91180830374207
3 1.8058847208471
4 1.80180044661685
5 1.98810964296792
6 2.17506655186131
7 2.33235041172927
8 2.54699453036108
9 2.68592814871033
10 2.84031822535535
11 2.94989525207884
12 3.07471406583812
13 3.14649976382836
14 3.22031649302251
15 3.31048934395584
16 3.38274224401974
17 3.42068949966608
18 3.4750724642386
19 3.50735385940108
};
\addlegendentry{Robust Regularization}
\addplot [very thick, color2]
table {%
0 2.67551307717566
1 2.81250230861908
2 2.64415682728786
3 2.5730221667482
4 2.47004350612691
5 2.47094685579435
6 2.78230518341597
7 2.97449888453948
8 3.15021090547661
9 3.32181092769063
10 3.45771960667696
11 3.64858842082256
12 3.80354192003575
13 3.90106852739289
14 4.00824944961248
15 4.07997191155602
16 4.20201796236547
17 4.26471869828296
18 4.35878703400785
19 4.41448084620189
};
\addlegendentry{Spectral Normalization}
\end{axis}

\end{tikzpicture}
 \end{subfigure}
 \hfill
 \begin{subfigure}[t]{0.45\textwidth}
\begin{tikzpicture}[scale=0.7]

\definecolor{color0}{rgb}{0.297,0.445,0.9}
\definecolor{color1}{rgb}{0.83921568627451,0.152941176470588,0.156862745098039}
\definecolor{color2}{rgb}{0.549019607843137,0.337254901960784,0.294117647058824}

\begin{axis}[
legend cell align={left},
legend style={fill opacity=0.3, draw opacity=1, text opacity=1, draw=white!80!black},
tick align=outside,
tick pos=left,
title={Humanoid-v2},
x grid style={white!69.0196078431373!black},
xlabel={Number of Interactions},
xmajorgrids,
xmin=-1.45, xmax=30.45,
xtick style={color=black},
xtick={-1,4,9,14,19,24,29},
xticklabels={0,5e+4,10e+4,15e+4,20e+4,25e+4,30e+4},
y grid style={white!69.0196078431373!black},
ylabel={Value-aware Model Error (Log Scale)},
ymajorgrids,
ymin=1.99507754985083, ymax=6.47721009881366,
ytick style={color=black}
]
\path [draw=color0, fill=color0, opacity=0.2]
(axis cs:0,3.21933167996323)
--(axis cs:0,3.04066308086508)
--(axis cs:1,2.81267286280633)
--(axis cs:2,2.64316604082825)
--(axis cs:3,2.60882469843139)
--(axis cs:4,2.56154386478864)
--(axis cs:5,2.56859579792229)
--(axis cs:6,2.69735733508232)
--(axis cs:7,2.84249846799915)
--(axis cs:8,2.91237328846558)
--(axis cs:9,2.77570789491714)
--(axis cs:10,2.97902207468269)
--(axis cs:11,2.92447020031277)
--(axis cs:12,2.75682346260126)
--(axis cs:13,2.6993226630075)
--(axis cs:14,2.60088290957678)
--(axis cs:15,2.47111508520298)
--(axis cs:16,2.41610561013549)
--(axis cs:17,2.45595218941969)
--(axis cs:18,2.53923895047458)
--(axis cs:19,2.32696182067988)
--(axis cs:20,2.36550998820768)
--(axis cs:21,2.29667921632921)
--(axis cs:22,2.33071744587533)
--(axis cs:23,2.31658863723778)
--(axis cs:24,2.24764345571054)
--(axis cs:25,2.28432512287377)
--(axis cs:26,2.31618433673497)
--(axis cs:27,2.40149153423666)
--(axis cs:28,2.29114612049974)
--(axis cs:29,2.24238352292767)
--(axis cs:29,2.40229685370848)
--(axis cs:29,2.40229685370848)
--(axis cs:28,2.38827303475824)
--(axis cs:27,2.56478922227287)
--(axis cs:26,2.44013507781908)
--(axis cs:25,2.3766284222841)
--(axis cs:24,2.33383553325251)
--(axis cs:23,2.58246742858468)
--(axis cs:22,2.46378917837069)
--(axis cs:21,2.54462666561202)
--(axis cs:20,2.5223291160935)
--(axis cs:19,2.50025201669928)
--(axis cs:18,2.7729054162798)
--(axis cs:17,2.57687484943452)
--(axis cs:16,2.5791601885449)
--(axis cs:15,2.60023986236974)
--(axis cs:14,2.74811665843364)
--(axis cs:13,2.95524914145748)
--(axis cs:12,3.17392212147808)
--(axis cs:11,3.58945304094803)
--(axis cs:10,3.81516385599353)
--(axis cs:9,3.67596020986973)
--(axis cs:8,3.43498291242962)
--(axis cs:7,3.25987424999397)
--(axis cs:6,2.94892150813523)
--(axis cs:5,2.80553451822262)
--(axis cs:4,2.64079181325843)
--(axis cs:3,2.71932903912772)
--(axis cs:2,2.77775962669012)
--(axis cs:1,2.92062120238125)
--(axis cs:0,3.21933167996323)
--cycle;

\path [draw=color1, fill=color1, opacity=0.2]
(axis cs:0,3.13927765797731)
--(axis cs:0,2.71444120980915)
--(axis cs:1,3.35878934449173)
--(axis cs:2,3.02707664541718)
--(axis cs:3,2.68425179820594)
--(axis cs:4,2.55838625195089)
--(axis cs:5,2.42791640964738)
--(axis cs:6,2.30725824717148)
--(axis cs:7,2.28982331963164)
--(axis cs:8,2.20002222860797)
--(axis cs:9,2.2187136478986)
--(axis cs:10,2.24460960458223)
--(axis cs:11,2.20146833351234)
--(axis cs:12,2.19881084753096)
--(axis cs:13,2.32029484696115)
--(axis cs:14,2.33352017678431)
--(axis cs:15,2.3232863837448)
--(axis cs:16,2.31176472201484)
--(axis cs:17,2.31878887424621)
--(axis cs:18,2.3883890970914)
--(axis cs:19,2.30336002119879)
--(axis cs:20,2.30120613637571)
--(axis cs:21,2.26861725531182)
--(axis cs:22,2.30468771959848)
--(axis cs:23,2.286704995961)
--(axis cs:24,2.30725409582186)
--(axis cs:25,2.22712459290107)
--(axis cs:26,2.26015246037364)
--(axis cs:27,2.23083228209417)
--(axis cs:28,2.20039478915658)
--(axis cs:29,2.20444267851505)
--(axis cs:29,2.27592319298349)
--(axis cs:29,2.27592319298349)
--(axis cs:28,2.27944530947598)
--(axis cs:27,2.31814170210102)
--(axis cs:26,2.36632580239539)
--(axis cs:25,2.30161928738914)
--(axis cs:24,2.37538707930482)
--(axis cs:23,2.37751068442341)
--(axis cs:22,2.4446742860504)
--(axis cs:21,2.3152481572978)
--(axis cs:20,2.37998224987709)
--(axis cs:19,2.38394211672584)
--(axis cs:18,2.456129824635)
--(axis cs:17,2.37287957714457)
--(axis cs:16,2.42674008682189)
--(axis cs:15,2.41012020514561)
--(axis cs:14,2.44760235674873)
--(axis cs:13,2.36106632616608)
--(axis cs:12,2.25326898234586)
--(axis cs:11,2.26520303449669)
--(axis cs:10,2.32286580826292)
--(axis cs:9,2.28987640923752)
--(axis cs:8,2.27159913065526)
--(axis cs:7,2.41820027541273)
--(axis cs:6,2.39030693105509)
--(axis cs:5,2.46148747767003)
--(axis cs:4,2.59733249023382)
--(axis cs:3,2.75448359053847)
--(axis cs:2,3.09819002352723)
--(axis cs:1,3.51343308402299)
--(axis cs:0,3.13927765797731)
--cycle;

\path [draw=color2, fill=color2, opacity=0.2]
(axis cs:0,4.7961515046359)
--(axis cs:0,4.15706777124714)
--(axis cs:1,3.11480882208192)
--(axis cs:2,3.06333639024756)
--(axis cs:3,3.10353650167801)
--(axis cs:4,3.20800940585031)
--(axis cs:5,3.58070733879931)
--(axis cs:6,3.65096039796606)
--(axis cs:7,4.2409617463187)
--(axis cs:8,4.56443098276454)
--(axis cs:9,5.07371868265338)
--(axis cs:10,5.20410433793362)
--(axis cs:11,4.83258281304799)
--(axis cs:12,4.37484425209767)
--(axis cs:13,4.2213573116156)
--(axis cs:14,3.88915677654452)
--(axis cs:15,3.63320785751676)
--(axis cs:16,3.39749868061294)
--(axis cs:17,3.04895762018438)
--(axis cs:18,2.75945056949479)
--(axis cs:19,2.56443639278039)
--(axis cs:20,2.56227496309151)
--(axis cs:21,2.39465212984344)
--(axis cs:22,2.35507147456177)
--(axis cs:23,2.32959595999248)
--(axis cs:24,2.3931007859586)
--(axis cs:25,2.39843859623369)
--(axis cs:26,2.2133195568475)
--(axis cs:27,2.3815378715542)
--(axis cs:28,2.60775648588337)
--(axis cs:29,2.49920181848077)
--(axis cs:29,3.19494476460148)
--(axis cs:29,3.19494476460148)
--(axis cs:28,3.84466793843934)
--(axis cs:27,2.69756965274165)
--(axis cs:26,2.41541490807653)
--(axis cs:25,2.56619281768215)
--(axis cs:24,2.78359315422482)
--(axis cs:23,2.45960949474871)
--(axis cs:22,2.50263627818376)
--(axis cs:21,2.56124999504339)
--(axis cs:20,2.80183676735081)
--(axis cs:19,2.98201093539288)
--(axis cs:18,3.34807669498956)
--(axis cs:17,4.15116151394873)
--(axis cs:16,4.90014215456041)
--(axis cs:15,5.11249726079268)
--(axis cs:14,5.58747860017694)
--(axis cs:13,5.86626296725903)
--(axis cs:12,5.74552268118024)
--(axis cs:11,5.93669753957249)
--(axis cs:10,6.18462537672623)
--(axis cs:9,6.27347680113353)
--(axis cs:8,5.45575587628316)
--(axis cs:7,4.89291385043115)
--(axis cs:6,4.18421315220594)
--(axis cs:5,3.9628496854489)
--(axis cs:4,3.3107268336201)
--(axis cs:3,3.20539723273076)
--(axis cs:2,3.10773296457592)
--(axis cs:1,3.26079898018806)
--(axis cs:0,4.7961515046359)
--cycle;

\addplot [very thick, color0]
table {%
0 3.12809646542603
1 2.86883667822653
2 2.71200412150126
3 2.66553330550798
4 2.59926845176974
5 2.68051826351572
6 2.81708732188825
7 3.04356875436609
8 3.17089662345071
9 3.19071652105227
10 3.37833870817943
11 3.24866355406833
12 2.95140944684119
13 2.82866457067056
14 2.67123379912031
15 2.5343867154829
16 2.49644718514929
17 2.51390382546295
18 2.65817397497839
19 2.41015728423578
20 2.44188428448239
21 2.4187637712715
22 2.40019095513268
23 2.44269473869775
24 2.28965568883491
25 2.33145761958767
26 2.37638579505957
27 2.48074282271293
28 2.33924305701946
29 2.31798715214987
};
\addlegendentry{Probabilistic Ensemble}
\addplot [very thick, color1]
table {%
0 2.91840678765431
1 3.43374744055041
2 3.06235266274043
3 2.72062596487657
4 2.57813173785291
5 2.44497938925809
6 2.34859036984299
7 2.35640330105543
8 2.23687130578786
9 2.25290011003145
10 2.28459514073598
11 2.23440602586307
12 2.22662785413869
13 2.34073285807958
14 2.3897602442279
15 2.36575234029468
16 2.37085590990826
17 2.3459968942322
18 2.42173409643843
19 2.34136138184643
20 2.34010640640774
21 2.29181872350073
22 2.37421218276739
23 2.33171063109901
24 2.34195837801315
25 2.26413555888282
26 2.31112770869072
27 2.2734009802049
28 2.2368332649141
29 2.24002299120245
};
\addlegendentry{Robust Regularization}
\addplot [very thick, color2]
table {%
0 4.45755156371786
1 3.18896414522552
2 3.086569460011
3 3.15403768077181
4 3.25922267392506
5 3.75729772410085
6 3.89530211030968
7 4.546551900638
8 5.00223988285934
9 5.63756892557021
10 5.68287802646575
11 5.37823297451491
12 5.04105920557461
13 5.00752003381236
14 4.69156828824364
15 4.34675577115833
16 4.1188575530441
17 3.58617135877769
18 3.03408927910289
19 2.76395233919584
20 2.67608944457311
21 2.47724962002195
22 2.42727590153602
23 2.39359913380675
24 2.58483231943606
25 2.48397273906953
26 2.31282633383124
27 2.53274482875088
28 3.19508673801079
29 2.82016863189097
};
\addlegendentry{Spectral Normalization}
\end{axis}

\end{tikzpicture}
 \end{subfigure}
 \vspace{0em}

 \begin{subfigure}[t]{0.45\textwidth}
\begin{tikzpicture}[scale=0.7]

\definecolor{color0}{rgb}{0.297,0.445,0.9}
\definecolor{color1}{rgb}{0.83921568627451,0.152941176470588,0.156862745098039}
\definecolor{color2}{rgb}{0.549019607843137,0.337254901960784,0.294117647058824}

\begin{axis}[
legend cell align={left},
legend style={fill opacity=0.3, draw opacity=1, text opacity=1, draw=white!80!black},
tick align=outside,
tick pos=left,
title={Hopper-v2},
x grid style={white!69.0196078431373!black},
xlabel={Number of Interactions},
xmajorgrids,
xmin=-1, xmax=9.45,
xtick style={color=black},
xtick={-1,1,3,5,7,9},
xticklabels={0,2e+4,4e+4,6e+4,8e+4,10e+4},
y grid style={white!69.0196078431373!black},
ylabel={Value-aware Model Error (Log Scale)},
ymajorgrids,
ymin=-0.0157985823348863, ymax=3.01151759916199,
ytick style={color=black}
]
\path [draw=color0, fill=color0, opacity=0.2]
(axis cs:0,1.80126977307777)
--(axis cs:0,1.55741280162024)
--(axis cs:1,1.09178926613298)
--(axis cs:2,1.17194421604592)
--(axis cs:3,0.983974361163842)
--(axis cs:4,0.906965930154522)
--(axis cs:5,0.722574474902808)
--(axis cs:6,0.662673139473782)
--(axis cs:7,0.45379940731482)
--(axis cs:8,0.536180354471509)
--(axis cs:9,0.469645007127665)
--(axis cs:9,0.587323080984556)
--(axis cs:9,0.587323080984556)
--(axis cs:8,0.620167376201653)
--(axis cs:7,0.535191717270529)
--(axis cs:6,0.744163453241122)
--(axis cs:5,0.802884695789206)
--(axis cs:4,1.00039721285784)
--(axis cs:3,1.069523514057)
--(axis cs:2,1.27142619407408)
--(axis cs:1,1.22226498338747)
--(axis cs:0,1.80126977307777)
--cycle;

\path [draw=color1, fill=color1, opacity=0.2]
(axis cs:0,1.07088169038472)
--(axis cs:0,0.817274066139946)
--(axis cs:1,0.735302364423854)
--(axis cs:2,0.744974421589906)
--(axis cs:3,0.560583244924619)
--(axis cs:4,0.336092792718301)
--(axis cs:5,0.32573581621571)
--(axis cs:6,0.250218901299865)
--(axis cs:7,0.189677351436226)
--(axis cs:8,0.184767638540691)
--(axis cs:9,0.121806698642245)
--(axis cs:9,0.182904053281285)
--(axis cs:9,0.182904053281285)
--(axis cs:8,0.25015017211514)
--(axis cs:7,0.241753150390556)
--(axis cs:6,0.374388738200404)
--(axis cs:5,0.392266249001259)
--(axis cs:4,0.487532053311383)
--(axis cs:3,0.68335828259705)
--(axis cs:2,0.839609103661529)
--(axis cs:1,0.93763489350973)
--(axis cs:0,1.07088169038472)
--cycle;

\path [draw=color2, fill=color2, opacity=0.2]
(axis cs:0,2.87391231818486)
--(axis cs:0,2.58034383637374)
--(axis cs:1,1.46795283503574)
--(axis cs:2,1.01111151871404)
--(axis cs:3,0.76636448418979)
--(axis cs:4,0.602797987820143)
--(axis cs:5,0.6239711991591)
--(axis cs:6,0.380816658208164)
--(axis cs:7,0.378440214498634)
--(axis cs:8,0.369953377215167)
--(axis cs:9,0.227854853538097)
--(axis cs:9,0.346663014157944)
--(axis cs:9,0.346663014157944)
--(axis cs:8,0.50587517942349)
--(axis cs:7,0.530099135667342)
--(axis cs:6,0.496225222322882)
--(axis cs:5,0.726674098746567)
--(axis cs:4,0.678266212572837)
--(axis cs:3,0.84922890097459)
--(axis cs:2,1.05754771107491)
--(axis cs:1,1.55073111718137)
--(axis cs:0,2.87391231818486)
--cycle;

\addplot [very thick, color0]
table {%
0 1.68329855660104
1 1.15175016072287
2 1.22372925706382
3 1.02682798014595
4 0.952575681409834
5 0.761820802069886
6 0.703559683233842
7 0.491091288651815
8 0.578906830792988
9 0.533165838406124
};
\addlegendentry{Probabilistic Ensemble}
\addplot [very thick, color1]
table {%
0 0.942424483494981
1 0.838351903121697
2 0.79354907099801
3 0.625760929201986
4 0.413760887620008
5 0.358275233038072
6 0.310528187933256
7 0.214739836510044
8 0.218065785899164
9 0.151672409490635
};
\addlegendentry{Robust Regularization}
\addplot [very thick, color2]
table {%
0 2.73118252935153
1 1.50842555138058
2 1.03496453883604
3 0.808365718840827
4 0.63979537864192
5 0.675531886556632
6 0.436728274457925
7 0.451324829773652
8 0.437125665175669
9 0.286743316002161
};
\addlegendentry{Spectral Normalization}
\end{axis}

\end{tikzpicture} 
 \end{subfigure}
 \hfill
 \begin{subfigure}[t]{0.45\textwidth}
\begin{tikzpicture}[scale=0.7]

\definecolor{color0}{rgb}{0.297,0.445,0.9}
\definecolor{color1}{rgb}{0.83921568627451,0.152941176470588,0.156862745098039}
\definecolor{color2}{rgb}{0.549019607843137,0.337254901960784,0.294117647058824}

\begin{axis}[
legend cell align={left},
legend style={
  fill opacity=0.3,
  draw opacity=1,
  text opacity=1,
  at={(0.97,0.03)},
  anchor=south east,
  draw=white!80!black
},
tick align=outside,
tick pos=left,
title={HalfCheetah-v2},
x grid style={white!69.0196078431373!black},
xlabel={Number of Interactions},
xmajorgrids,
xmin=-1, xmax=19.95,
xtick style={color=black},
xtick={-1,4,9,14,19},
xticklabels={0,5e+4,10e+4,15e+4,20e+4},
y grid style={white!69.0196078431373!black},
ylabel={Value-aware Model Error (Log Scale)},
ymajorgrids,
ymin=0.70371645025238, ymax=4.13471284707665,
ytick style={color=black}
]
\path [draw=color0, fill=color0, opacity=0.2]
(axis cs:0,1.17256491378547)
--(axis cs:0,0.85967083192621)
--(axis cs:1,1.75376372043168)
--(axis cs:2,2.18591026836314)
--(axis cs:3,2.35945380719646)
--(axis cs:4,2.66397922275124)
--(axis cs:5,2.86596130038531)
--(axis cs:6,2.9917322805985)
--(axis cs:7,3.06241631781946)
--(axis cs:8,3.11042995947213)
--(axis cs:9,3.25515172066261)
--(axis cs:10,3.23994335320745)
--(axis cs:11,3.33573179902066)
--(axis cs:12,3.30203006813357)
--(axis cs:13,3.38080000763258)
--(axis cs:14,3.54384241646796)
--(axis cs:15,3.40970557231623)
--(axis cs:16,3.44522293126204)
--(axis cs:17,3.48120876474359)
--(axis cs:18,3.59389435175131)
--(axis cs:19,3.64788895062078)
--(axis cs:19,3.94718353611874)
--(axis cs:19,3.94718353611874)
--(axis cs:18,3.89781380975991)
--(axis cs:17,3.79405416691039)
--(axis cs:16,3.70874084017404)
--(axis cs:15,3.7316816814344)
--(axis cs:14,3.97875846540282)
--(axis cs:13,3.76374054746095)
--(axis cs:12,3.69447510704204)
--(axis cs:11,3.73298702257757)
--(axis cs:10,3.66596780385382)
--(axis cs:9,3.68382448313705)
--(axis cs:8,3.49366084225702)
--(axis cs:7,3.44911752678846)
--(axis cs:6,3.39124239337738)
--(axis cs:5,3.28671848414345)
--(axis cs:4,3.09605992709637)
--(axis cs:3,2.80043212580402)
--(axis cs:2,2.71393707758974)
--(axis cs:1,2.20342564297428)
--(axis cs:0,1.17256491378547)
--cycle;

\path [draw=color1, fill=color1, opacity=0.2]
(axis cs:0,2.30431641493766)
--(axis cs:0,1.22492902272671)
--(axis cs:1,1.95714516261141)
--(axis cs:2,2.1133179355834)
--(axis cs:3,2.31238120629603)
--(axis cs:4,2.37791482794274)
--(axis cs:5,2.45611396166251)
--(axis cs:6,2.57723348228945)
--(axis cs:7,2.65423903128461)
--(axis cs:8,2.66247924539682)
--(axis cs:9,2.73138350135783)
--(axis cs:10,2.80102731290643)
--(axis cs:11,2.82728452415126)
--(axis cs:12,2.87897864938515)
--(axis cs:13,2.90789005000587)
--(axis cs:14,2.94312451025689)
--(axis cs:15,2.89812346080325)
--(axis cs:16,2.92858031765679)
--(axis cs:17,3.07789412348794)
--(axis cs:18,3.16622010144298)
--(axis cs:19,3.1964935989431)
--(axis cs:19,3.33833786615737)
--(axis cs:19,3.33833786615737)
--(axis cs:18,3.34787647177963)
--(axis cs:17,3.20161581634541)
--(axis cs:16,3.11432627652723)
--(axis cs:15,3.05669195016962)
--(axis cs:14,3.14738884851147)
--(axis cs:13,3.10550742766101)
--(axis cs:12,3.08933332483645)
--(axis cs:11,3.03855773199219)
--(axis cs:10,3.02850174270846)
--(axis cs:9,2.96918240563575)
--(axis cs:8,2.88186587360794)
--(axis cs:7,2.89165142135407)
--(axis cs:6,2.8367703882889)
--(axis cs:5,2.73309554600488)
--(axis cs:4,2.66692376999262)
--(axis cs:3,2.62505731502859)
--(axis cs:2,2.42636092430535)
--(axis cs:1,2.99918971092282)
--(axis cs:0,2.30431641493766)
--cycle;

\path [draw=color2, fill=color2, opacity=0.2]
(axis cs:0,3.42062437428383)
--(axis cs:0,2.01394315321788)
--(axis cs:1,1.88183578176269)
--(axis cs:2,2.01431715483207)
--(axis cs:3,2.25934098401421)
--(axis cs:4,2.29911495642041)
--(axis cs:5,2.45483377636959)
--(axis cs:6,2.67285203486322)
--(axis cs:7,2.74317077965167)
--(axis cs:8,2.86900823570949)
--(axis cs:9,2.92856646094148)
--(axis cs:10,2.87397571682447)
--(axis cs:11,3.07319286962176)
--(axis cs:12,3.11480250839612)
--(axis cs:13,3.05417031043098)
--(axis cs:14,3.16092956191273)
--(axis cs:15,3.21936944316242)
--(axis cs:16,3.31442011708015)
--(axis cs:17,3.21822827703981)
--(axis cs:18,3.38792745677037)
--(axis cs:19,3.39904188148967)
--(axis cs:19,3.63448580245963)
--(axis cs:19,3.63448580245963)
--(axis cs:18,3.66232230617087)
--(axis cs:17,3.54965536938094)
--(axis cs:16,3.61656517526219)
--(axis cs:15,3.51706156874543)
--(axis cs:14,3.50239897622058)
--(axis cs:13,3.36221376853606)
--(axis cs:12,3.41707411731908)
--(axis cs:11,3.35372986207153)
--(axis cs:10,3.1310561924921)
--(axis cs:9,3.19703815273124)
--(axis cs:8,3.1253911333569)
--(axis cs:7,2.99927245136449)
--(axis cs:6,2.88641678416278)
--(axis cs:5,2.68026945498858)
--(axis cs:4,2.59144281243542)
--(axis cs:3,2.53481404789975)
--(axis cs:2,2.37079055364597)
--(axis cs:1,2.16873813753131)
--(axis cs:0,3.42062437428383)
--cycle;

\addplot [very thick, color0]
table {%
0 1.01063027598244
1 1.97378811658345
2 2.43306922532118
3 2.57526560023742
4 2.87216497564212
5 3.08515698888699
6 3.19342153223363
7 3.26083096588818
8 3.29954972117086
9 3.47092323291291
10 3.46640588270214
11 3.53557143285438
12 3.50679932001999
13 3.5752549865375
14 3.76661113906637
15 3.57457365393192
16 3.57678469970613
17 3.64065276410572
18 3.7481737322014
19 3.80072828394505
};
\addlegendentry{Probabilistic Ensemble}
\addplot [very thick, color1]
table {%
0 1.72643415828232
1 2.44577073806393
2 2.26749739594223
3 2.47027375913475
4 2.51210098653638
5 2.59866909703892
6 2.70635470023758
7 2.77228054786247
8 2.7713925496962
9 2.85230305722214
10 2.91624355901587
11 2.93214437652774
12 2.98536202631365
13 3.00623944911636
14 3.04612871270632
15 2.97898253229062
16 3.01955979986288
17 3.1402334288263
18 3.25502297009328
19 3.26829910960201
};
\addlegendentry{Robust Regularization}
\addplot [very thick, color2]
table {%
0 2.69896778734775
1 2.01846231691619
2 2.18379093724918
3 2.38907468284845
4 2.44603029706405
5 2.56853025964061
6 2.77846901258054
7 2.87166669866127
8 2.99676553473737
9 3.06409824552084
10 3.00214939515718
11 3.205845120296
12 3.25622297380949
13 3.2078724706883
14 3.33645300971716
15 3.36846573812135
16 3.46445939813922
17 3.3811729731661
18 3.52371822781241
19 3.513856116141
};
\addlegendentry{Spectral Normalization}
\end{axis}

\end{tikzpicture} 
 \end{subfigure}
 \vspace{0em}
 \caption{
 Comparison between the proposed mechanisms and the probabilistic ensemble baseline in terms of Value-aware Model Error (Log Scale) 
 }
\label{fig:app_vaml}
\vspace{-0.5em}
\end{figure}
In Section \ref{s5s2:vaml} and \ref{s5s3:rr_sn}, we demonstrate the effectiveness of robust regularization in controlling the value-aware model error on Walker and also show the limitation of constraining the global Lipschitz constant by spectral normalization. Here, we provide the results of value-aware model error in the rest of the four environments. We used the same best hyperparameters reported in Table~\ref{tab:grid_search_rr} and Table~\ref{tab:grid_search_spec_norm}, which have the best performance under grid search. Once again, we observe that robust regularization effectively reduces the value-aware model error by constraining the local Lipschitz condition with computing adversarial perturbation. In addition, we also see that global Lipschitz constraints are too strong for spectral normalization. In two more complicated environments, Ant and Humanoid, it has to sacrifice value-aware model error for the expressive power of the value function. Therefore, spectral normalization does not achieve a good performance in these two environments. However, in two easier environments, Hopper and HalfCheetah, spectral normalization could still effectively reduce the value-aware model error and has good empirical performance.
\subsection{Proposed Mechanisms with Single Probabilistic Model}
\label{a3:robust_probabilistic}
\begin{figure}[!htbp]
\vspace{-0.5em}
\centering
 \begin{subfigure}[t]{0.45\textwidth}
  \centering
\begin{tikzpicture}[scale=0.7]

\definecolor{color0}{rgb}{0.549019607843137,0.337254901960784,0.294117647058824}
\definecolor{color1}{rgb}{0.83921568627451,0.152941176470588,0.156862745098039}
\definecolor{color2}{rgb}{0.95,0.5,0.3}
\definecolor{color3}{rgb}{0.25,0.41,0.99}
\definecolor{color4}{rgb}{0.172549019607843,0.627450980392157,0.172549019607843}

\begin{axis}[
legend cell align={left},
legend style={
  fill opacity=0.3,
  draw opacity=1,
  text opacity=1,
  at={(0.03,0.97)},
  anchor=north west,
  draw=white!80!black
},
tick align=outside,
tick pos=left,
x grid style={white!69.0196078431373!black},
xlabel={Number of Interactions},
xmajorgrids,
xmin=-1, xmax=19.95,
xtick style={color=black},
xtick={-1,4,9,14,19},
xticklabels={0,5e+4,10e+4,15e+4,20e+4},
y grid style={white!69.0196078431373!black},
ylabel={Performance},
ymajorgrids,
ymin=351.601733333333, ymax=5721.3456,
ytick style={color=black}
]
\path [draw=color0, fill=color0, opacity=0.2]
(axis cs:0,720.5585)
--(axis cs:0,697.260833333333)
--(axis cs:1,852.406)
--(axis cs:2,855.339833333333)
--(axis cs:3,1045.217)
--(axis cs:4,1234.09)
--(axis cs:5,1504.6965)
--(axis cs:6,1803.24883333333)
--(axis cs:7,2122.5195)
--(axis cs:8,2356.77816666667)
--(axis cs:9,2557.882)
--(axis cs:10,2844.083)
--(axis cs:11,3220.017)
--(axis cs:12,3322.88083333333)
--(axis cs:13,3455.19133333333)
--(axis cs:14,3479.41616666667)
--(axis cs:15,3654.134)
--(axis cs:16,3713.23716666667)
--(axis cs:17,3958.63083333333)
--(axis cs:18,4094.92683333333)
--(axis cs:19,4007.66616666667)
--(axis cs:19,4554.19433333333)
--(axis cs:19,4554.19433333333)
--(axis cs:18,4594.81683333333)
--(axis cs:17,4481.73033333333)
--(axis cs:16,4234.70483333333)
--(axis cs:15,4092.74816666667)
--(axis cs:14,3909.90033333333)
--(axis cs:13,3866.92683333333)
--(axis cs:12,3730.73883333333)
--(axis cs:11,3559.51916666667)
--(axis cs:10,3171.361)
--(axis cs:9,2901.44316666667)
--(axis cs:8,2700.0875)
--(axis cs:7,2332.0855)
--(axis cs:6,1961.36016666667)
--(axis cs:5,1674.56933333333)
--(axis cs:4,1379.24366666667)
--(axis cs:3,1152.53883333333)
--(axis cs:2,872.994333333333)
--(axis cs:1,866.737833333333)
--(axis cs:0,720.5585)
--cycle;

\path [draw=color1, fill=color1, opacity=0.2]
(axis cs:0,779.5835)
--(axis cs:0,769.859333333333)
--(axis cs:1,723.173333333334)
--(axis cs:2,843.5305)
--(axis cs:3,878.323833333333)
--(axis cs:4,933.2)
--(axis cs:5,1216.55266666667)
--(axis cs:6,1523.04516666667)
--(axis cs:7,1868.56333333333)
--(axis cs:8,2468.72)
--(axis cs:9,2702.33633333333)
--(axis cs:10,3282.48)
--(axis cs:11,3513.45466666667)
--(axis cs:12,3709.36783333333)
--(axis cs:13,4154.1715)
--(axis cs:14,4440.41)
--(axis cs:15,4653.05166666667)
--(axis cs:16,4676.22)
--(axis cs:17,5006.58333333333)
--(axis cs:18,5143.932)
--(axis cs:19,5344.54)
--(axis cs:19,5477.26633333333)
--(axis cs:19,5477.26633333333)
--(axis cs:18,5309.2535)
--(axis cs:17,5257.92)
--(axis cs:16,4953.17333333333)
--(axis cs:15,5006.3325)
--(axis cs:14,4813.9615)
--(axis cs:13,4471.40566666667)
--(axis cs:12,4163.247)
--(axis cs:11,3875.0155)
--(axis cs:10,3576.85866666667)
--(axis cs:9,3106.8)
--(axis cs:8,2569.0295)
--(axis cs:7,2137.75383333333)
--(axis cs:6,1789.23316666667)
--(axis cs:5,1466.30666666667)
--(axis cs:4,1169.94183333333)
--(axis cs:3,988.286833333333)
--(axis cs:2,860.553333333333)
--(axis cs:1,750.573333333333)
--(axis cs:0,779.5835)
--cycle;

\path [draw=color2, fill=color2, opacity=0.2]
(axis cs:0,746.539333333333)
--(axis cs:0,703.043)
--(axis cs:1,746.297666666666)
--(axis cs:2,908.325166666667)
--(axis cs:3,1002.82666666667)
--(axis cs:4,1162.65283333333)
--(axis cs:5,1554.04616666667)
--(axis cs:6,1575.01566666667)
--(axis cs:7,1987.0145)
--(axis cs:8,2434.76333333333)
--(axis cs:9,2820.42)
--(axis cs:10,3114.00833333333)
--(axis cs:11,3219.54166666667)
--(axis cs:12,3561.02083333333)
--(axis cs:13,3666.329)
--(axis cs:14,3743.0825)
--(axis cs:15,4076.418)
--(axis cs:16,3976.55183333333)
--(axis cs:17,4148.067)
--(axis cs:18,4300.27933333333)
--(axis cs:19,4417.17916666667)
--(axis cs:19,4773.81816666667)
--(axis cs:19,4773.81816666667)
--(axis cs:18,4886.34666666667)
--(axis cs:17,4639.3085)
--(axis cs:16,4455.25333333333)
--(axis cs:15,4631.79)
--(axis cs:14,4268.06516666667)
--(axis cs:13,4329.939)
--(axis cs:12,3946.90916666667)
--(axis cs:11,3777.55166666667)
--(axis cs:10,3682.28483333333)
--(axis cs:9,3356.15333333333)
--(axis cs:8,2871.92016666667)
--(axis cs:7,2495.29)
--(axis cs:6,1910.96783333333)
--(axis cs:5,1787.6705)
--(axis cs:4,1302.67166666667)
--(axis cs:3,1094.59833333333)
--(axis cs:2,976.220833333333)
--(axis cs:1,802.802333333333)
--(axis cs:0,746.539333333333)
--cycle;

\path [draw=color3, fill=color3, opacity=0.2]
(axis cs:0,657.544666666667)
--(axis cs:0,624.419333333333)
--(axis cs:1,797.6425)
--(axis cs:2,824.929333333333)
--(axis cs:3,934.812833333333)
--(axis cs:4,1184.15916666667)
--(axis cs:5,1569.51583333333)
--(axis cs:6,1944.00983333333)
--(axis cs:7,2222.51733333333)
--(axis cs:8,2805.927)
--(axis cs:9,2991.4825)
--(axis cs:10,3035.13)
--(axis cs:11,3641.31516666667)
--(axis cs:12,3651.83583333333)
--(axis cs:13,3679.6715)
--(axis cs:14,3837.87066666667)
--(axis cs:15,4256.79183333333)
--(axis cs:16,4211.67466666667)
--(axis cs:17,4421.60216666667)
--(axis cs:18,4674.54683333333)
--(axis cs:19,4665.48766666667)
--(axis cs:19,5120.42866666667)
--(axis cs:19,5120.42866666667)
--(axis cs:18,5204.35666666667)
--(axis cs:17,4955.84383333333)
--(axis cs:16,4739.78633333333)
--(axis cs:15,4691.77183333333)
--(axis cs:14,4347.01516666667)
--(axis cs:13,4136.66333333333)
--(axis cs:12,4060.8465)
--(axis cs:11,3988.228)
--(axis cs:10,3416.547)
--(axis cs:9,3479.0475)
--(axis cs:8,3218.534)
--(axis cs:7,2427.1505)
--(axis cs:6,2161.68166666667)
--(axis cs:5,1772.73066666667)
--(axis cs:4,1372.012)
--(axis cs:3,1006.08066666667)
--(axis cs:2,845.597166666667)
--(axis cs:1,823.631666666667)
--(axis cs:0,657.544666666667)
--cycle;

\path [draw=color4, fill=color4, opacity=0.2]
(axis cs:0,622.158666666667)
--(axis cs:0,595.681)
--(axis cs:1,731.609)
--(axis cs:2,675.834333333333)
--(axis cs:3,859.083333333333)
--(axis cs:4,663.4615)
--(axis cs:5,745.411)
--(axis cs:6,698.8155)
--(axis cs:7,784.143333333333)
--(axis cs:8,901.004)
--(axis cs:9,1102.8895)
--(axis cs:10,942.542)
--(axis cs:11,1043.21016666667)
--(axis cs:12,1203.89316666667)
--(axis cs:13,1423.51666666667)
--(axis cs:14,1426.5415)
--(axis cs:15,1567.63666666667)
--(axis cs:16,1561.23366666667)
--(axis cs:17,1668.62333333333)
--(axis cs:18,1821.41866666667)
--(axis cs:19,1909.2695)
--(axis cs:19,2185.553)
--(axis cs:19,2185.553)
--(axis cs:18,2266.568)
--(axis cs:17,2047.06666666667)
--(axis cs:16,2019.10283333333)
--(axis cs:15,1801.62433333333)
--(axis cs:14,1557.91466666667)
--(axis cs:13,1483.10266666667)
--(axis cs:12,1385.09433333333)
--(axis cs:11,1266.91083333333)
--(axis cs:10,1147.84466666667)
--(axis cs:9,1270.4905)
--(axis cs:8,998.156)
--(axis cs:7,886.309166666667)
--(axis cs:6,798.1975)
--(axis cs:5,818.171666666667)
--(axis cs:4,747.324)
--(axis cs:3,903.9565)
--(axis cs:2,735.946666666667)
--(axis cs:1,790.170666666667)
--(axis cs:0,622.158666666667)
--cycle;

\addplot [very thick, color0]
table {%
0 708.83179
1 859.576976666667
2 864.30829
3 1095.33326333333
4 1307.35946
5 1590.75043333333
6 1879.29275
7 2225.5867
8 2524.26302333333
9 2731.96645
10 3012.89232333333
11 3384.01491
12 3527.06149666667
13 3653.07816
14 3695.22315666667
15 3864.76092666667
16 3974.70018
17 4213.17128666667
18 4335.13481333333
19 4281.96414
};
\addlegendentry{Probabilistic Ensemble}
\addplot [very thick, color1]
table {%
0 774.527136666667
1 736.445733333333
2 852.236303333333
3 932.09789
4 1045.0669
5 1339.90174666667
6 1657.69921333333
7 1998.60910333333
8 2521.89809666667
9 2904.98042
10 3428.33179333333
11 3695.39565666667
12 3934.22533333333
13 4318.63698
14 4626.28426333333
15 4842.79526666667
16 4818.14630666667
17 5135.0706
18 5227.59845333333
19 5409.74289333333
};
\addlegendentry{RR (Probabilistic)}
\addplot [very thick, color2]
table {%
0 724.387293333333
1 774.178866666666
2 941.740526666666
3 1049.01886
4 1232.99361333333
5 1672.48011666667
6 1742.57169
7 2241.18381666667
8 2657.13046
9 3098.99550666667
10 3406.06659333333
11 3485.35640333333
12 3742.83055333333
13 4001.65156
14 3997.46729666667
15 4352.00017
16 4225.57134666667
17 4390.25054333333
18 4585.58409333333
19 4597.43649333333
};
\addlegendentry{RR (Deterministic)}
\addplot [very thick, color3]
table {%
0 641.677866666667
1 811.54136
2 835.063003333333
3 969.339673333333
4 1277.37287
5 1669.67583333333
6 2048.77485333333
7 2325.35166
8 3001.09462
9 3230.27418
10 3218.51289
11 3808.50608
12 3852.14285
13 3913.02562333333
14 4110.86612
15 4474.04845666667
16 4492.1824
17 4705.81733333333
18 4943.58583666667
19 4896.83578
};
\addlegendentry{SN (Probabilistic)}
\addplot [very thick, color4]
table {%
0 609.15485
1 760.10717
2 705.05116
3 881.55407
4 704.595193333333
5 783.063346666667
6 750.019786666667
7 836.16486
8 947.938873333334
9 1183.27515
10 1041.57480666667
11 1158.80982
12 1295.53258
13 1454.54363666667
14 1492.61379333333
15 1678.71472333333
16 1785.8084
17 1850.84719
18 2051.06960333333
19 2046.26205
};
\addlegendentry{SN (Deterministic)}
\end{axis}

\end{tikzpicture}
  \centering
  \vspace{-1.6em}
    \caption{Performance}
    \label{sfig:robust_prob_perf}
 \end{subfigure}
 \hfill
 \begin{subfigure}[t]{0.45\textwidth}
  \centering
\input{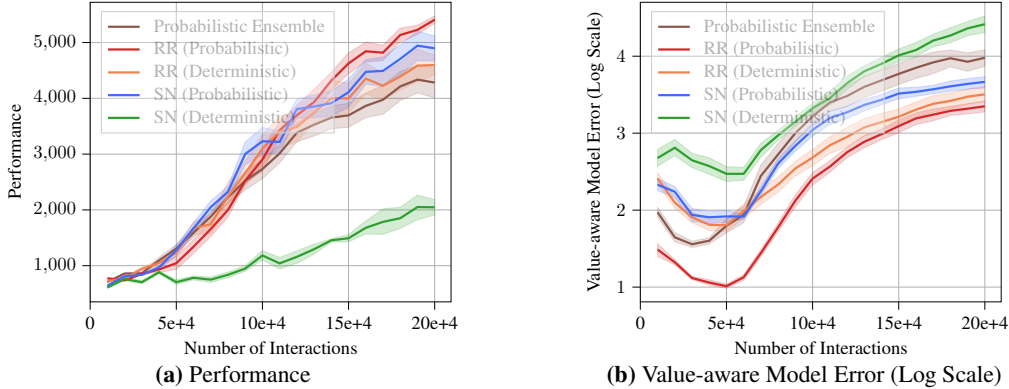}
  \centering
  \vspace{-0.5em}
    \caption{Value-aware Model Error (Log Scale) }
    \label{sfig:robust_prob_vaml}
 \end{subfigure}
 \vspace{-0.8em}
  \caption{Spectral Normalization and Robust Regularization with a single probabilistic model on Ant. RR is short for Robust Regularization, and SN is short for Spectral Normalization
  }
\label{fig:robust_probabilsitic}
\vspace{-0.5em}
\end{figure}
In the experiment section, we combine our proposed mechanisms with a single deterministic model and compare it against MBPO using an ensemble of probabilistic models. The purpose is to verify that regularization of the local Lipschitz constant is critical in MBRL algorithms and propose a computationally efficient MBRL algorithm without a model ensemble. In practice, complementary to our proposed Lipschitz regularization mechanisms, we can also use a single probabilistic model to further regularize the local Lipschitz condition of the value function. 
In addition, training a probabilistic environment model would be better suited for environments with stochastic transitions. 

Here in Figure~\ref{fig:robust_probabilsitic}, we combine our two proposed mechanisms with both a probabilistic and deterministic model on Ant, comparing them with the probabilistic ensemble baseline. 
From Figure~\ref{sfig:robust_prob_vaml}, we see that although spectral normalization with a single deterministic model has a large value-aware model error, it is significantly reduced when combining it with a probabilistic dynamics model. Therefore, we find that spectral normalization with a probabilistic model achieves much better performance and even outperforms MBPO with an ensemble of probabilistic models. For robust regularization, using a probabilistic model also helps improve the algorithm's value-aware model error and performance. This observation suggests that the two Lipschitz regularization approaches, explicit regularization by spectral normalization or robust regularization and implicit regularization by probabilistic models, are complementary. In practice, we can combine the two approaches to get the best performance of the MBRL algorithm.

\subsection{Robust Regularization with FGSM vs. PGD}
\label{a3:pgd_vs_fgsm}
\begin{figure}[!htbp]
\vspace{-1.0em}
  \centering
\begin{tikzpicture}[scale=0.7]

\definecolor{crimson2143839}{RGB}{214,38,39}
\definecolor{darkgray176}{RGB}{176,176,176}
\definecolor{forestgreen4416044}{RGB}{44,160,44}
\definecolor{lightgray204}{RGB}{204,204,204}

\begin{axis}[
legend cell align={left},
legend style={
  fill opacity=0.2,
  draw opacity=1,
  text opacity=1,
  at={(0.97,0.03)},
  anchor=south east,
  draw=lightgray204
},
tick align=outside,
tick pos=left,
x grid style={darkgray176},
xlabel={Number of Interactions},
xmajorgrids,
xmin=-1, xmax=19.95,
xtick style={color=black},
xtick={-1,4,9,14,19},
xticklabels={0,5e+4,10e+4,15e+4,20e+4},
y grid style={darkgray176},
ylabel={Performance},
ymajorgrids,
ymin=-17.817, ymax=5000,
ytick style={color=black}
]
\path [draw=forestgreen4416044, fill=forestgreen4416044, opacity=0.2]
(axis cs:0,223.920666666667)
--(axis cs:0,203.863333333333)
--(axis cs:1,356.26)
--(axis cs:2,444.656666666667)
--(axis cs:3,524.598333333333)
--(axis cs:4,951.538333333333)
--(axis cs:5,1875.92333333333)
--(axis cs:6,2783.04)
--(axis cs:7,2981.18)
--(axis cs:8,2840.86483333333)
--(axis cs:9,3282.757)
--(axis cs:10,3256.19666666667)
--(axis cs:11,3378.26666666667)
--(axis cs:12,3611.9855)
--(axis cs:13,3939.5805)
--(axis cs:14,4179.1)
--(axis cs:15,3836.66316666667)
--(axis cs:16,3969.49)
--(axis cs:17,4099.9)
--(axis cs:18,3956.5745)
--(axis cs:19,4291.23166666667)
--(axis cs:19,4417.6005)
--(axis cs:19,4417.6005)
--(axis cs:18,4249.65633333333)
--(axis cs:17,4227.60666666667)
--(axis cs:16,4200.95)
--(axis cs:15,4020.53283333333)
--(axis cs:14,4219.54066666667)
--(axis cs:13,4083.183)
--(axis cs:12,3771.26783333333)
--(axis cs:11,3628.15783333333)
--(axis cs:10,3554.22283333333)
--(axis cs:9,3443.36333333333)
--(axis cs:8,3052.01666666667)
--(axis cs:7,3163.5385)
--(axis cs:6,3029.08333333333)
--(axis cs:5,2331.39)
--(axis cs:4,1259.54633333333)
--(axis cs:3,675.301666666667)
--(axis cs:2,494.6)
--(axis cs:1,374.514)
--(axis cs:0,223.920666666667)
--cycle;

\path [draw=crimson2143839, fill=crimson2143839, opacity=0.2]
(axis cs:0,401.259166666667)
--(axis cs:0,360.916666666667)
--(axis cs:1,369.345833333333)
--(axis cs:2,525.26)
--(axis cs:3,583.606666666667)
--(axis cs:4,1056.07333333333)
--(axis cs:5,2068.81933333333)
--(axis cs:6,2649.84)
--(axis cs:7,3000.25116666667)
--(axis cs:8,3817.497)
--(axis cs:9,3929.28333333333)
--(axis cs:10,4042.00566666667)
--(axis cs:11,3828.89666666667)
--(axis cs:12,3896.27)
--(axis cs:13,4230.82666666667)
--(axis cs:14,4322.77666666667)
--(axis cs:15,4485.16816666667)
--(axis cs:16,4306.40666666667)
--(axis cs:17,4478.46333333333)
--(axis cs:18,4583.66666666667)
--(axis cs:19,4324.94666666667)
--(axis cs:19,4564.51)
--(axis cs:19,4564.51)
--(axis cs:18,4634.42666666667)
--(axis cs:17,4637.47)
--(axis cs:16,4522.12666666667)
--(axis cs:15,4592.67)
--(axis cs:14,4497.56666666667)
--(axis cs:13,4320.16)
--(axis cs:12,4182.19666666667)
--(axis cs:11,4182.37333333333)
--(axis cs:10,4093.60766666667)
--(axis cs:9,4155.06333333333)
--(axis cs:8,3950.33333333333)
--(axis cs:7,3252.81333333333)
--(axis cs:6,2813.19516666667)
--(axis cs:5,2409.37333333333)
--(axis cs:4,1322.71266666667)
--(axis cs:3,664.593333333333)
--(axis cs:2,642.516666666667)
--(axis cs:1,405.233333333333)
--(axis cs:0,401.259166666667)
--cycle;

\addplot [very thick, forestgreen4416044]
table {%
0 213.963226666667
1 365.67004
2 470.0784
3 597.647413333333
4 1096.89942
5 2093.77898333333
6 2905.07641333333
7 3069.00657
8 2947.97782333333
9 3364.55596666667
10 3409.31920333333
11 3500.90030666667
12 3691.49001666667
13 4015.7406
14 4200.60360666667
15 3931.11402666667
16 4091.73603666667
17 4165.92043333333
18 4113.79957666667
19 4356.25968666667
};
\addlegendentry{PGD-20}
\addplot [very thick, crimson2143839]
table {%
0 381.364016666667
1 387.773096666667
2 580.70792
3 623.208006666667
4 1193.98430666667
5 2233.55416
6 2731.34439
7 3129.97217
8 3882.61796
9 4042.64086
10 4067.92198666667
11 4007.11041333333
12 4044.55608
13 4277.13172333333
14 4409.17768333333
15 4540.08992333333
16 4416.81493333333
17 4558.56401666667
18 4609.26492666667
19 4444.05740333333
};
\addlegendentry{FGSM}
\end{axis}

\end{tikzpicture}
 \vspace{-1.em}
  \caption{Robust Regularization with 20 steps of Project Gradient Descent (PGD-20) against Fast Gradient Sign Method (FGSM). 
  }
\label{fig:fgsm_pgd}
\end{figure}
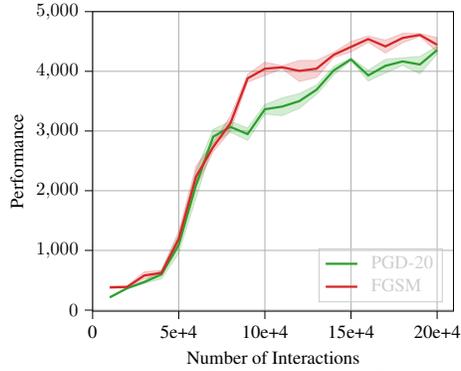
In Figure ~\ref{fig:fgsm_pgd}, we compare the performance of robust regularization with 20 steps of Project Gradient Descent (PGD-20) against the Fast Gradient Sign Method (FGSM) on Walker. In particular, although PGD-20 is much more computationally expensive, we do not observe the improvement with this more powerful constrained optimization solver.

\subsection{Robust Regularization with different regularization weights}
\label{a3:robust_coeff}

\begin{figure}[!htbp]
  \centering
\begin{tikzpicture}[scale=0.7]

\definecolor{darkgray176}{RGB}{176,176,176}
\definecolor{steelblue76114176}{RGB}{76,114,176}

\begin{axis}[
tick align=outside,
tick pos=left,
x grid style={darkgray176},
xlabel={Robust Regularization Coefficient \(\displaystyle \lambda\)},
xmajorgrids,
xmin=-0.025, xmax=0.525,
xtick style={color=black},
y grid style={darkgray176},
ylabel={Performance},
ymajorgrids,
ymin=2700, ymax=4593.68434029471,
ytick style={color=black}
]
\path [draw=steelblue76114176, very thick]
(axis cs:0,2953.80082112798)
--(axis cs:0,3519.07417887202);

\path [draw=steelblue76114176, very thick]
(axis cs:0.01,3657.28405521577)
--(axis cs:0.01,4205.09094478423);

\path [draw=steelblue76114176, very thick]
(axis cs:0.03,3811.31657776038)
--(axis cs:0.03,4189.45842223962);

\path [draw=steelblue76114176, very thick]
(axis cs:0.05,3762.05535109418)
--(axis cs:0.05,4515.59464890582);

\path [draw=steelblue76114176, very thick]
(axis cs:0.1,3992.94188669853)
--(axis cs:0.1,4378.60811330147);

\path [draw=steelblue76114176, very thick]
(axis cs:0.2,3791.76141161258)
--(axis cs:0.2,4251.06358838742);

\path [draw=steelblue76114176, very thick]
(axis cs:0.3,3688.143162233)
--(axis cs:0.3,4227.856837767);

\path [draw=steelblue76114176, very thick]
(axis cs:0.4,3630.46686433378)
--(axis cs:0.4,4184.43313566622);

\path [draw=steelblue76114176, very thick]
(axis cs:0.5,3573.8634694389)
--(axis cs:0.5,4216.5115305611);

\addplot [very thick, steelblue76114176, mark=-, mark size=4, mark options={solid}, only marks]
table {%
0 2953.80082112798
0.01 3657.28405521577
0.03 3811.31657776038
0.05 3762.05535109418
0.1 3992.94188669853
0.2 3791.76141161258
0.3 3688.143162233
0.4 3630.46686433378
0.5 3573.8634694389
};
\addplot [very thick, steelblue76114176, mark=-, mark size=4, mark options={solid}, only marks]
table {%
0 3519.07417887202
0.01 4205.09094478423
0.03 4189.45842223962
0.05 4515.59464890582
0.1 4378.60811330147
0.2 4251.06358838742
0.3 4227.856837767
0.4 4184.43313566622
0.5 4216.5115305611
};
\addplot [very thick, steelblue76114176, mark=*, mark size=4, mark options={solid}]
table {%
0 3236.4375
0.01 3931.1875
0.03 4000.3875
0.05 4138.825
0.1 4185.775
0.2 4021.4125
0.3 3958
0.4 3907.45
0.5 3895.1875
};
\end{axis}

\end{tikzpicture}
 \vspace{-1.em}
  \caption{Robust Regularization with different regularization weights $\lambda$'s on Walker. Experiments are all with 8 random seeds.
  }
\label{fig:robust_reg_hyper}
\end{figure}
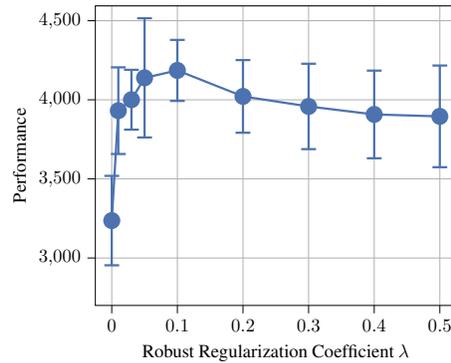
Figure~\ref{fig:robust_reg_hyper} further visualizes how the regularization weight $\lambda$ of robust regularization influences the algorithm performance. Similar to the findings of the experiments on spectral normalization, we see that the algorithm's performance first increases and drops as the regularization weight gets larger. This verifies our theoretical insights from Theorem~\ref{thm:ap_main} that with a small $\lambda$, with algorithm gets less regularization and thus has a big value-aware model error. But meanwhile, it also has a small regression error since the regularization has little effect on the expressive power of the value function. When $\lambda$ goes up, the regularization will have a stronger negative effect on the expressive power of the value function, but the value-aware model error will also get smaller. We observe that the algorithm performs the best with $\lambda=0.1$, achieving the balance between value-aware model error and the value function's expressive power.

\newpage
\section{Additional Details on The Inverted-Pendulum Experiment}
\label{a3:test_exp}
In Section~\ref{s3_theory}, we provide an experiment of model-based value iteration on the Inverted Pendulum to further verify the validity of our theorem. Below we provide the pseudocode for it.

\begin{algorithm}[!htbp]
\caption{Model-based Approximate Value Iteration on Inverted Pendulum}
\label{alg:train}
\begin{algorithmic}
   \STATE $K:$ Total number of iterations for the algorithm
   \STATE $H:$ Number of gradient steps to solve the inner optimization
   \STATE $\mathcal{M}, \mathcal{G}:$ Space of transition probability kernels and reward functions
   \STATE $\mathcal{F}$: Space of value function 
   \STATE $\mu\in \Delta(\mathcal{S}\times \mathcal{A})$: a data-collecting state-action distribution
   \STATE Sample i.i.d from $\mu$ to generate the training dataset $\mathcal{D}=\{(s_i, a_i, s'_i, r_i)\}_{i=1}^{N}$
   \STATE $\displaystyle \hat{\mathcal{P}} \leftarrow \argmin_{\hat{\mathcal{P}}\in \mathcal{M}}\sum_{i=1}^{N}\|s'_i-\int \hat{\mathcal{P}}(ds'|s,a)s'\|^2$
   \STATE $ \displaystyle \hat{r} \leftarrow \argmin_{\hat{r}\in \mathcal{G}} \sum_{i=1}^{N}(r_i-\hat{r}(s_i, a_i))^2$
   \STATE Initialize the value function $\hat{Q}_0$.
   \REPEAT
       \FOR{$k=0$ {\bfseries to} $K-1$}
          \STATE Sample i.i.d $N$ state-action pairs from $\rho: \{(s_i,a_i)\}_{i=1}^{N}$
          \STATE Compute $\hat{s}'_i\sim \mathcal{P}(\cdot|s_i, a_i)$, $\hat{r}_i=\hat{r}(s_i, a_i)$
          \FOR{$t=0$ {\bfseries to} $H-1$}
          \STATE Update policy using gradient ascent with 
          $\displaystyle \frac{1}{N} \sum_{i=1}^{N}\nabla_\theta \hat{Q}_\phi(s_i, \pi_{\theta}(s_i))$
          \ENDFOR
          
          \FOR{$t=0$ {\bfseries to} $H-1$}
          \STATE Update value function using gradient descent with 
          \STATE $\displaystyle \frac{1}{N}\sum_{i=1}^{N}\nabla_\phi \Big(\hat{r}_i+\gamma \hat{Q}_\phi(\hat{s}'_{i}, \pi_\theta(\hat{s}'_i)-\hat{Q}_\phi(s_i, a_i)\Big)^2$
          \ENDFOR
       \ENDFOR
   \UNTIL{end of training}
   \STATE {\bfseries Output:} $\hat{Q}_\phi$
\end{algorithmic}
\end{algorithm}
\newpage
\section{Lipschitz Regularization of Model-free RL algorithms}
\label{app:mf}
\begin{figure}[!htbp]
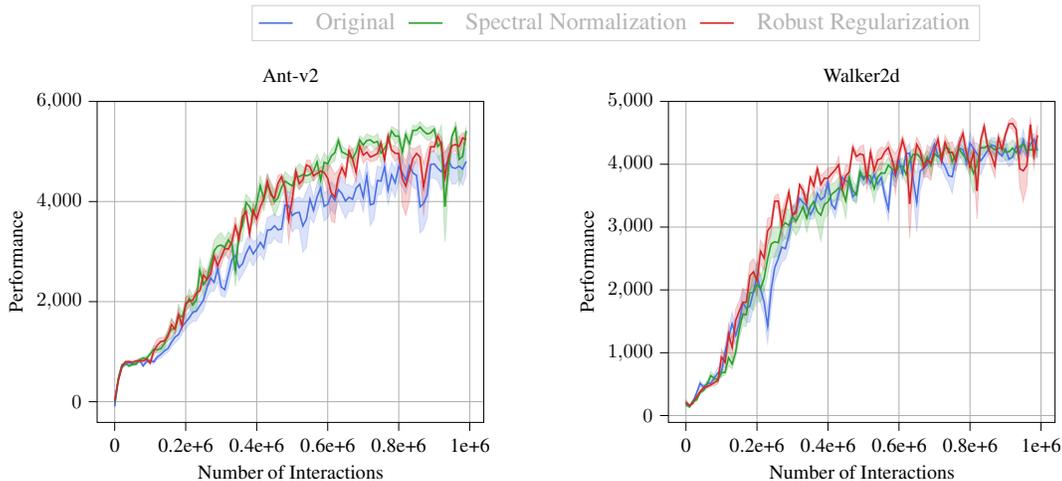

\vspace{-0.5em}
\centering
 \begin{subfigure}[t]{0.45\textwidth}
  \centering
  \input{ant_mf.tex}
  \centering
  \vspace{-1.6em}
    \label{sfig:ant_mf}
 \end{subfigure}
 \hfill
 \begin{subfigure}[t]{0.45\textwidth}
  \centering
\input{walker2d_mf.tex}
  \centering
    \label{sfig:walker_mf}
 \end{subfigure}
 \vspace{-0.8em}
  \caption{
  Combination of robust regularization with model-free SAC 
  }
\label{fig:robust_mf}
\vspace{-0.5em}
\end{figure}

In this paper, we provide both theoretical and empirical insights into why Lipschitz regularization is crucial in model-based RL algorithms through the lens of value-aware model error. On the contrary, the model error vanishes in the model-free setting, so our theoretical insights no longer hold. However, fitting a value function with a smaller Lipschitz constant may still be beneficial for policy optimization and the value prediction of out-of-distribution state-action pairs.

So does the improvement shown in Section~\ref{s5_experiment} indeed come from the controlled value-aware model error, which is unique in the model-based setting? We conduct an experiment on the model-free setting, adding our proposed spectral normallization and robust regularization mechanisms to the model-free SAC respectively with the same hyperparameter settings used in the paper. As shown in Figure~\ref{fig:robust_mf}, combining the Lipschitz regularization mechanisms slightly improves the performance of model-free SAC algorithm. Still, the improvement is far more limited compared with the improvement of the model-based algorithm shown in Figure~\ref{fig:main_experimental_results}. The results suggest that although some additional aspects of value learning could be affected by Lipschitz regularization, our insights into the value-aware model error for model-based scenarios should still hold. The exploration of how Lipschitz regularization impacts other aspects of value learning in the model-free setting is beyond the scope of our work, but it would be an interesting direction for future work.

\end{document}